\documentclass{article}

\usepackage[preprint]{neurips_2026}

\usepackage[utf8]{inputenc} 
\usepackage[T1]{fontenc}    
\usepackage{xcolor}         
\definecolor{darkgreen}{rgb}{0.1, 0.5, 0.1}
\definecolor{darkred}{rgb}{0.85, 0.1, 0.1}
\usepackage[colorlinks=true, citecolor=darkgreen, linkcolor=darkred, urlcolor=blue]{hyperref}       
\usepackage{url}            
\usepackage{booktabs}       
\usepackage{amsfonts}       
\usepackage{nicefrac}       
\usepackage{microtype}      
\usepackage{graphicx}       
\usepackage{amsmath}        
\usepackage{amsthm}         

\newtheorem{definition}{Definition}
\newtheorem{theorem}{Theorem}
\newtheorem{lemma}{Lemma}

\newtheorem{assumption}{Assumption}
\newtheorem{remark}{Remark}

\title{A Location-Invariant Estimator of Extremal Quantile Treatment Effects for Heavy-Tailed Distributions}

\author{%
  Xin Yu$^{1,2,*}$ \quad Shuwei Huang$^{1,2,*}$ \quad Jicheng Liu$^{1,\dagger}$ \\
  Jielin Tang$^{2}$ \quad Bolin Wang$^{2}$ \quad Yunxiao Zhang$^{2}$ \quad Tian Zhao$^{2}$ \\
  $^{1}$Huazhong University of Science and Technology \quad $^{2}$Tencent \\
  $^{*}$Equal contribution \quad $^{\dagger}$Corresponding author \\
  Contact: Jicheng Liu \texttt{\{jcliu@hust.edu.cn\}} \\
  Code: \url{https://github.com/Faze-Hsw/location_invariant_extremal_QTE_estimator} \\
}

\makeatletter
\renewcommand{\@maketitle}{%
  \vbox{%
    \hsize\textwidth
    \linewidth\hsize
    \vskip 0.1in
    \@toptitlebar
    \centering
    {\LARGE\bf \@title\par}
    \@bottomtitlebar
    \if@anonymous
      \begin{tabular}[t]{c}\rule{\z@}{24\p@}
        Anonymous Author(s) \\
        Affiliation \\
        Address \\
        \texttt{email} \\
      \end{tabular}%
    \else
      \def\And{%
        \end{tabular}\hfil\linebreak[0]\hfil%
        \begin{tabular}[t]{c}\rule{\z@}{24\p@}\ignorespaces%
      }
      \def\AND{%
        \end{tabular}\hfil\linebreak[4]\hfil%
        \begin{tabular}[t]{c}\rule{\z@}{24\p@}\ignorespaces%
      }
      \begin{tabular}[t]{c}\rule{\z@}{24\p@}\@author\end{tabular}%
    \fi
    \vskip 0.3in \@minus 0.1in
  }
}
\makeatother

\begin{document}

\maketitle

\begin{abstract}
  Quantile treatment effects (QTEs) measure the effect of a treatment on
  the distribution of an outcome, and their estimation at extreme
  quantile levels is of central interest in applications where the
  target quantiles lie far beyond the range of the data.  For
  heavy-tailed potential outcomes, existing extremal QTE estimators rely
  on extrapolation combined with a causal extreme value index (EVI)
  estimator, but the resulting estimator is not invariant under a common
  location shift of the potential outcome distributions, even though the
  population QTE is.  We address this issue in two steps.  First, we
  adapt the location-invariant Fraga estimator of the EVI to the causal
  setting using inverse propensity score weighting.  Second, we replace
  the original extrapolation formula with a difference-based scheme,
  under which the location parameter cancels when quantile differences
  are taken.  The resulting QTE estimator is therefore location
  invariant.  We establish the consistency and
  asymptotic normality of the proposed extremal QTE estimators, and
  provide a consistent variance estimator, leading to asymptotically
  valid inference.  A simulation study confirms the location invariance, the
  stability with respect to the threshold, and the coverage of the
  proposed methods.
\end{abstract}

\section{Introduction}


Quantile treatment effects (QTEs) describe the causal effect of a
treatment or policy intervention on the distribution of an outcome,
rather than on its mean, and they arise naturally in many real-world
applications.  For example, Abadie
et~al.~\cite{abadie2002quantiles} found that subsidized job training
raises the earnings of trainees only in the upper half of the earnings
distribution, Bitler et~al.~\cite{bitler2006mean} showed
that welfare reform experiments produce distributional effects on
family income that are missed by the average treatment effect, and
Autor et~al.~\cite{autor2016contribution} demonstrated
that minimum wage increases lift wages throughout the lower tail of the
wage distribution.  These examples show that an
intervention may affect different parts of the outcome distribution, in
particular its tails, very differently, so that focusing on averages
alone can be misleading.

To make this precise, we adopt the potential outcome framework.  Let
$D \in \{0, 1\}$ denote a binary treatment, and let $Y(j)$ be the
potential outcome of an individual under treatment $j \in \{0, 1\}$; the
observed outcome is $Y = D\,Y(1) + (1 - D)\,Y(0)$.  Denote by $F_j$ the
distribution function of $Y(j)$ and by
$q_j(\tau) = F_j^{\leftarrow}(\tau) = \inf\{y \in \mathbb{R} : F_j(y) \geq \tau\}$
the corresponding quantile function.  For $\tau \in (0, 1)$, the
\emph{quantile treatment effect} is defined as the difference
\begin{equation*}
  \Delta(\tau) = q_1(\tau) - q_0(\tau),
\end{equation*}
which measures the effect of the treatment on the $\tau$-quantile of the
outcome distribution.  As $\tau$ varies over $(0, 1)$, the curve
$\Delta(\tau)$ traces the effect of the treatment across the entire
outcome distribution, exposing heterogeneity that the average treatment
effect $\mathbb{E}\bigl[Y(1) - Y(0)\bigr]$, which collapses the
distribution into a single number, necessarily conceals.  Note that,
like the average treatment effect, the QTE is invariant under a common
location shift of the potential outcomes: if both $Y(1)$ and $Y(0)$ are
shifted by the same constant $c$, then the difference of the two
quantiles remains unchanged.

In practice, however, $\Delta(\tau)$ cannot be identified from
observational data without further assumptions, since for any unit only
one of the two potential outcomes $Y(0)$ and $Y(1)$ is observed, and
the treatment assignment $D$ is typically correlated with the
covariates $X \in \mathbb{R}^{p}$ that drive the outcomes.  Following
Rosenbaum and Rubin~\cite{rosenbaum1983central}, we invoke the
\emph{unconfoundedness} condition
$(Y(1), Y(0)) \perp\!\!\!\perp D \mid X$, which states that, conditional
on the covariates, the potential outcomes are independent of the
treatment.  Under this condition, the propensity score
$\pi(x) = \mathbb{P}(D = 1 \mid X = x)$ can be used to balance the
covariate distributions across treatment groups, and combining
unconfoundedness with the \emph{common support} requirement
$0 < \pi(x) < 1$ on the support of $X$ yields identification of
$\Delta(\tau)$ within the binary treatment setting.

Given this identification, an extensive literature has developed
estimation and inference methods for QTE.  For a fixed quantile level
$\tau \in (0, 1)$, Firpo~\cite{firpo2007efficient} proposed an
efficient semiparametric estimator of $\Delta(\tau)$ based on inverse
propensity score weighting (see Equation~\eqref{eq:qte_firpo} in
Section~\ref{sec:qte}).  However, it is confined to a fixed $\tau$:
deep in the tail, only very few observations exceed the level, so the
weighted empirical quantile becomes unreliable.  To move beyond fixed
levels, Zhang~\cite{zhang2018extremal} extended the framework to a
\emph{dynamic sequence} of quantile levels $\tau_n \to 0$ as
$n \to +\infty$.  Depending on the decay rate of $\tau_n$, two regimes
arise: for an \emph{intermediate level}, where $n\tau_n \to +\infty$,
the estimator remains asymptotically normal; for an \emph{extreme
level}, where $n\tau_n \to a \geq 0$, the expected number of
observations beyond the level is finite, so the empirical quantile is
based on only finitely many tail observations and asymptotic normality
breaks down.  Moreover, the QTE estimators of
Firpo~\cite{firpo2007efficient} and Zhang~\cite{zhang2018extremal} are
both invariant under a common location shift of the potential outcomes.

For many applications, however, the most relevant quantiles are
precisely those that go beyond the range of the data.  In climate
attribution studies, for instance, scientists investigate the causal
effect of anthropogenic influences on climate extremes such as heavy
precipitation; the quantiles of interest for such events typically go
far beyond the range of historical recordings and therefore require
extreme value extrapolation (e.g.,
Easterling et~al.~\cite{easterling2016detection};
van Oldenborgh et~al.~\cite{vanoldenborgh2017attribution}).
Formally, this corresponds to the \emph{extreme level} regime, where
$n\tau_n \to a \geq 0$, for which the fixed- and intermediate-level
estimators above are no longer applicable.

\begin{figure}[t]
  \centering
  \includegraphics[width=\linewidth]{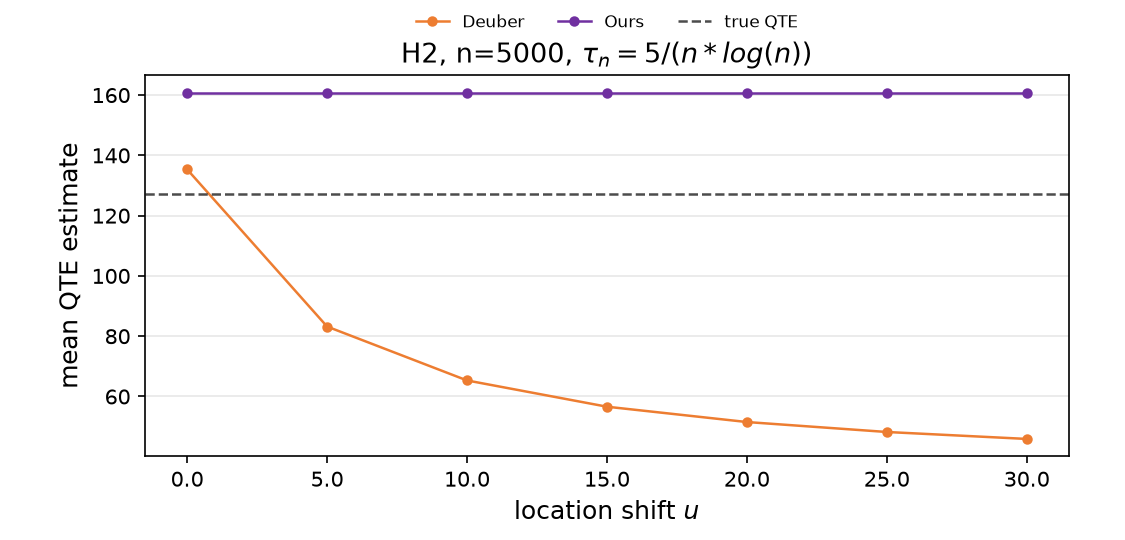}
  \caption{Mean of the quantile treatment effect estimates, averaged
  over $1000$ Monte Carlo replications, as a function of the location
  shift $u$ for model $H_2$ at $n = 5000$, at the extreme quantile
  level $\tau_n = 5/(n\log n)$, i.e.\ for the estimation of
  $\Delta(1-\tau_n)$.  The data generating process of model $H_2$ is
  described in Section~\ref{sec:simulations}.}
  \label{fig:intro_location_invariance}
\end{figure}

To estimate the QTE at such extreme levels,
Deuber et~al.~\cite{deuber2024estimation} focused on heavy-tailed
distributions, whose tail probabilities decay polynomially and are thus
heavier than Gaussian, and proposed an extrapolation-based estimator.
By the theory of regular variation of de Haan and
Ferreira~\cite{de2006extreme}, the extreme quantile of each
potential outcome can be extrapolated from an intermediate quantile as
\begin{equation*}
  q_j(1-\tau_n) \approx q_j(1-\alpha_n)
    \left( \frac{\alpha_n}{\tau_n} \right)^{\!\gamma_j},
    \qquad j \in \{0, 1\},
\end{equation*}
where $\gamma_j > 0$ is the extreme value index (EVI) of the potential
outcome $Y(j)$, $\alpha_n$ is an intermediate level with
$n\alpha_n \to +\infty$, and $\tau_n$ is an extreme level with
$n\tau_n \to a \geq 0$.  Their estimator plugs the inverse propensity
score weighted intermediate quantile estimator
of~Firpo~\cite{firpo2007efficient} and a newly proposed causal Hill
estimator of $\gamma_j$, based on the classical Hill
estimator~\cite{hill1975simple}, into this approximation, and they
established its asymptotic normality and provided a consistent variance
estimator.

Location invariance is a property that a QTE estimator should ideally
inherit from the QTE itself.  As noted above, $\Delta(\tau) =
q_1(\tau) - q_0(\tau)$ is unchanged when both potential outcomes are
shifted by the same constant; an estimator that instead depends on the
location of the data would therefore attribute to the treatment an
effect that is a mere artefact of the location.  In applications, where
the reference point or the measurement scale is often arbitrary, such
location dependence is undesirable.
However, the estimator of Deuber
et~al.~\cite{deuber2024estimation} lacks this invariance, and this for
two reasons that act jointly: first, its causal Hill estimator of the
extreme value index is not invariant to location shifts; second, even
with a location-invariant extreme value index estimator, the
multiplicative extrapolation formula itself remains location-dependent.
The present paper resolves both sources of location dependence.

To this end, two ingredients are required.  First, since the
extrapolation step hinges on the extreme value index, we need an
estimator of $\gamma_j$ that is both capable of causal identification
and invariant under location shifts.  Several classical extreme value
index estimators are location-invariant, including the Pickands
estimator~\cite{pickands1975statistical}, the moment
estimator of Dekkers et~al.~\cite{dekkers1989moment}, and
location-invariant Hill-type
constructions~\cite{fraga2001location,ling2012location,li2010asymptotic}.  In this paper
we build on the classical location-invariant estimator
of~Fraga~Alves~\cite{fraga2001location} and adapt it through inverse
propensity score weighting, so that the resulting estimator eliminates
confounding effects while inheriting the location invariance of its
classical counterpart; we refer to this estimator as the
\emph{causal Fraga} estimator.  We establish its consistency and
asymptotic normality under conditions that are weaker than those
required by the causal Hill estimator of Deuber
et~al.~\cite{deuber2024estimation}, whose validity relies on their
Assumptions~4, B1 and B2.

Second, even with a
location-invariant extreme value index estimator, the QTE estimator
obtained from the multiplicative extrapolation of Deuber
et~al.~\cite{deuber2024estimation} remains location-dependent, so the
extrapolation formula itself has to be modified.  From the theory of
regular variation of de Haan and Ferreira~\cite{de2006extreme},
however, an alternative form of extrapolation can be derived:
\begin{equation*}
  q_j(1-\tau_n)
  \approx q_j(1-\alpha_n)
    + \bigl[q_j(1-\beta_n) - q_j(1-\alpha_n)\bigr]
      \cdot
      \frac{(\alpha_n/\tau_n)^{\gamma_j} - 1}
           {(\alpha_n/\beta_n)^{\gamma_j} - 1},
\end{equation*}
where $\beta_n \to 0$ with $n\beta_n \to +\infty$, so that $\beta_n$ is
an intermediate level.  We call this the \emph{difference
extrapolation}, as the location parameter cancels in the differences of
the quantiles.  Combined with the location-invariant causal Fraga
estimator, this ensures that the resulting QTE estimator is
location-invariant.  Figure~\ref{fig:intro_location_invariance}
illustrates this point: while the QTE estimates of
Deuber et~al.~\cite{deuber2024estimation} drift away as the location
shift $u$ grows, the estimates based on our difference extrapolation
combined with the causal Fraga estimator remain stable across all
values of $u$.  Similarly, under weaker conditions than those of Deuber
et~al.~\cite{deuber2024estimation}, we establish the asymptotic
normality of the QTE estimators based on both the multiplicative and
the difference extrapolation schemes combined with the causal Fraga
estimator, whereas the estimator of Deuber et~al. requires additional
conditions (their Assumptions~4, B1, B2 and B3).  Moreover, in contrast
with the estimator of Deuber et~al., which has a non-zero
asymptotic bias, our QTE estimator is first-order unbiased.

The topic of causality for extreme events is receiving increasing
interest.  On the structural side,
Gissibl and~Kl{\"u}ppelberg~\cite{gissibl2018maxlinear} and
Gissibl et~al.~\cite{gissibl2018tail} introduced recursive max-linear
structural causal models that generate extreme observations, while
Mhalla et~al.~\cite{mhalla2020causal} and
Gnecco et~al.~\cite{gnecco2021extremal} developed causal structure
learning and conditional independence tests tailored to extremes.  In
climate science, a large literature on extreme event attribution
quantifies the causal effect of anthropogenic forcing on weather and
climate extremes, typically through model-based counterfactual
analyses~\cite{hannart2016causal,easterling2016detection,vanoldenborgh2017attribution,naveau2018revising,naveau2020statistical}.
On the more applied side, Bhuyan et~al.~\cite{bhuyan2021analysing}
quantified the causal effect of the London cycle superhighways on
extreme traffic congestion, albeit without theoretical guarantees.
Within this growing literature, our work adds a theoretically
justified estimator and inference procedure for the extremal QTE
$\Delta(1-\tau_n)$ in the presence of confounding, and in particular
supplements it with the important property that the estimator is
invariant under location shifts of the potential outcomes.

Our contributions are twofold.  First, we construct a
location-invariant \emph{causal Fraga estimator} of the extreme value
index via inverse propensity score weighting, and establish its
consistency and asymptotic normality under weaker conditions than those
required by the causal Hill estimator of Deuber
et~al.~\cite{deuber2024estimation}.  Second, building on this estimator,
we introduce a \emph{difference extrapolation} scheme in which the
location parameter cancels in the differences of the quantiles, thereby
obtaining a location-invariant estimator of the extremal QTE; we then
establish its asymptotic normality, provide a consistent variance
estimator, and construct asymptotically valid confidence intervals.  We
validate both contributions through a finite-sample simulation study,
confirming the location invariance, the stability with respect to the
threshold, and the coverage of the confidence intervals.

The remainder of the paper is organized as follows.  Section~\ref{sec:preliminaries}
recalls the necessary background on extreme value theory and
introduces the potential outcome framework together with the
identification conditions for the QTE.  Section~\ref{sec:causal_eviestimator}
constructs the causal Fraga estimator of the extreme value index and
establishes its consistency and asymptotic normality.  Section~\ref{sec:extremal_quantile}
develops the difference extrapolation estimator of the extremal QTE,
derives its asymptotic distribution, and provides a consistent
variance estimator for constructing asymptotically honest confidence
intervals.  Section~\ref{sec:simulations} reports the finite-sample
performance of the proposed methods through an extensive simulation
study.  In the appendix, Section~\ref{sec:proofs} collects all
technical proofs, Section~\ref{sec:mult_extrapolation} details the
corresponding multiplicative extrapolation estimator, and
Section~\ref{sec:additional_sim_results} presents additional
simulation results.

Throughout, and consistent with the literature on extreme value
theory, we focus on the extremal QTE in the upper tail, that is, on
$\Delta(1-\tau_n)$ with $\tau_n \to 0$.  The treatment effect on the
lower tail can be handled analogously by applying the same arguments to
the left tails of the potential outcome distributions.

\section{Preliminaries}
\label{sec:preliminaries}


\subsection{Extreme Value Theory}

We study extreme quantiles, namely quantiles at level $1-\tau_n$ with $\tau_n \to 0$ and $n\tau_n \to a \geq 0$ as $n \to +\infty$.
When $a = 0$, empirical estimates of extreme quantiles become severely biased and
classical asymptotic theory no longer applies.
Extreme value theory addresses this gap by providing methods for quantile
extrapolation that yield more accurate estimates of extreme quantiles; see
de Haan and Ferreira~\cite{de2006extreme}.
The theory rests on a mild distributional assumption that ensures the tail of
the distribution admits an accurate parametric approximation. Formally, let
$Y$ be a random variable with distribution $F$ and quantile function
$q(\cdot) = F^{\leftarrow}(\cdot)$, where
$f^{\leftarrow}(x) := \inf\{y \in \mathbb{R} : f(y) \geq x\}$ is the
left-continuous inverse of a nondecreasing function $f$.

\begin{definition}[Maximum domain of attraction and extreme value index (see De Haan and Ferreira~\cite{de2006extreme})]
\label{def:max_domain}
For $\gamma \in \mathbb{R}$, if there exist sequences $a_n > 0$ and
$b_n \in \mathbb{R}$, $n \in \mathbb{N}$, such that
\begin{equation*}
  \lim_{n \to +\infty} F^n(a_n x + b_n) = G_\gamma(x), \quad \forall x \in \mathbb{R},
\end{equation*}
where
\begin{equation*}
  G_\gamma(x) =
  \begin{cases}
    \exp\!\left(-(1 + \gamma x)^{-1/\gamma}\right), & 1 + \gamma x > 0,\ \gamma \neq 0, \\[4pt]
    \exp(-\exp(-x)),                               & x \in \mathbb{R},\ \gamma = 0,
  \end{cases}
\end{equation*}
then the distribution $F$ is said to belong to the maximum domain of
attraction of the generalized extreme value distribution $G_\gamma$, denoted
$F \in \mathcal{D}(G_\gamma)$, and the parameter $\gamma$ is called the
\emph{extreme value index}.
\end{definition}

This condition is not restrictive in practice: most commonly encountered
distributions---such as the Gaussian, Student-$t$, and beta families---fall
into the maximum domain of attraction of some $G_\gamma$. A systematic
treatment of the three regimes ($\gamma < 0$, $\gamma = 0$, $\gamma > 0$)
can be found in Resnick~\cite{resnick2007heavy} and Embrechts et al.~\cite{embrechts1997modelling}.

Let $Y_1, Y_2, \ldots, Y_n$ be independent and identically distributed
random variables drawn from a distribution $F \in \mathcal{D}(G_\gamma)$
and let $Y_{(1)} \leq Y_{(2)} \leq \cdots \leq Y_{(n)}$ denote the
corresponding order statistics. The integer $k = k_n \in \{1, \ldots, n-1\}$
is an intermediate sequence satisfying $k \to \infty$ and $k/n \to 0$ as
$n \to \infty$. We now recall four classical
estimators of the extreme value index $\gamma$.

When $\gamma > 0$, Hill~\cite{hill1975simple} proposed the celebrated Hill
estimator,
\begin{equation*}
  \hat{\gamma}^{H}(k) = \frac{1}{k} \sum_{i=0}^{k-1}
    \log \frac{Y_{(n-i)}}{Y_{(n-k)}}.
\end{equation*}
For the more general case $\gamma \in \mathbb{R}$, Dekkers
et~al.~\cite{dekkers1989moment} built on the Hill estimator and
proposed the classical moment estimator,
\begin{equation*}
  \hat{\gamma}^{M}(k) = M_n^{(1)}(k) + 1 - \frac{1}{2}
    \left(1 - \frac{\bigl(M_n^{(1)}(k)\bigr)^2}{M_n^{(2)}(k)}\right)^{-1},
\end{equation*}
where
\begin{equation*}
  M_n^{(j)}(k) = \frac{1}{k} \sum_{i=0}^{k-1}
    \left(\log \frac{Y_{(n-i)}}{Y_{(n-k)}}\right)^{\!j}, \quad j \in \{0, 1, 2\},
\end{equation*}
with $M_n^{(1)}(k) = \hat{\gamma}^{H}(k)$.
Pickands~\cite{pickands1975statistical} provided another estimator that
applies to any $\gamma \in \mathbb{R}$,
\begin{equation*}
  \hat{\gamma}^{P}(k) = \frac{1}{\log 2} \,
    \log \frac{Y_{(n-[k/4])}
                - Y_{(n-[k/2])}}
               {Y_{(n-[k/2])}
                - Y_{(n-[k])}},
\end{equation*}
where $[x]$ denotes the largest integer not exceeding $x$.
Inspired by Pickands's estimator, Fraga~Alves~\cite{fraga2001location}
proposed a location-invariant generalization of the Hill estimator via a
random shift:
\begin{equation*}
  \hat{\gamma}^{F}(k_0, k) = \frac{1}{k_0} \sum_{i=0}^{k_0 - 1}
    \log \frac{Y_{(n-i)} - Y_{(n-k)}}{Y_{(n-k_0)} - Y_{(n-k)}},
\end{equation*}
where $n \to +\infty$, $k_0 \to +\infty$, $k \to +\infty$,
$k_0 / k \to 0$ and $k / n \to 0$.

Since $\gamma$ is a shape parameter of the distribution $F$, it is
intrinsically unchanged by any location shift: for any constant $c \in
\mathbb{R}$, the extreme value index of $Y$ and that of $Y + c$ are
identical. However, the Hill and the moment estimators are not
location-invariant, since they are based on logarithms of ratios of the
observations themselves. By contrast, the Fraga estimator achieves
location invariance by subtracting the reference order statistic
$Y_{(n-k)}$, and the Pickands estimator is likewise location-invariant,
as it is built on differences of order statistics.
The extreme value index $\gamma$ controls the tail heaviness of the
distribution $F$: larger values of $\gamma$ correspond to heavier tails. In
this work, we focus on the heavy-tailed regime $\gamma > 0$. To make the
theoretical setting precise, we now give a formal definition of a
heavy-tailed distribution.

\begin{definition}[First-order regular variation (Deuber et al.~\cite{deuber2024estimation})]
\label{def:heavy_tailed}
For $\gamma > 0$, the distribution $F$ is called \emph{heavy-tailed} if
its tail $1 - F$ admits the first-order regular variation representation
\begin{equation}
  1 - F(x) = L(x) \, x^{-1/\gamma}, \quad \forall x > 0,
  \label{eq:heavy_tail}
\end{equation}
where $L(x)$ is a slowly varying function at infinity, that is,
\begin{equation*}
  \lim_{x \to +\infty} \frac{L(t x)}{L(x)} = 1, \quad \forall t > 0.
\end{equation*}
\end{definition}

We now present several equivalent characterizations of $F \in \mathcal{D}(G_\gamma)$
in the heavy-tailed regime $\gamma > 0$, following de Haan and
Ferreira~\cite{de2006extreme}.  According to de Haan and
Ferreira~\cite{de2006extreme}, $F \in \mathcal{D}(G_\gamma)$ holds if and only
if one of the following conditions is satisfied.
\begin{enumerate}
\def\theenumi{\roman{enumi}}
\def\labelenumi{(\theenumi)}
  \item For $\gamma > 0$,
  \begin{equation}
    \lim_{t \to +\infty} \frac{1 - F(t x)}{1 - F(t)} = x^{-1/\gamma},
    \quad \forall x > 0.
    \label{eq:equiv_first_order}
  \end{equation}

  \item For $\gamma > 0$,
  \begin{equation}
    \lim_{t \to +\infty} \frac{U(t x)}{U(t)} = x^{\gamma},
    \quad \forall x > 0,
    \label{eq:equiv_quantile}
  \end{equation}
  where
  \begin{equation*}
    U(t) =
    \begin{cases}
      0,                                     & 0 \leq t < 1, \\
      \left(\dfrac{1}{1 - F}\right)^{\!\leftarrow}\!(t), & t \geq 1,
    \end{cases}
  \end{equation*}
  with $\left(\dfrac{1}{1 - F}\right)^{\!\leftarrow}\!(t)
  = \inf\!\left\{y \,\middle|\, \dfrac{1}{1 - F(y)} \geq t\right\}$.

  \item For $\gamma > 0$, there exists a function $a(t) > 0$ such that
  \begin{equation}
    \lim_{t \to +\infty} \frac{U(t x) - U(t)}{a(t)}
      = \frac{x^{\gamma} - 1}{\gamma}, \quad \forall x > 0.
    \label{eq:equiv_second_order}
  \end{equation}
\end{enumerate}

Since Equation~\eqref{eq:heavy_tail} and Equation~\eqref{eq:equiv_first_order}
are equivalent, whenever $F \in \mathcal{D}(G_\gamma)$ and $\gamma > 0$,
the distribution $F$ is heavy-tailed and its tail $1 - F$ admits the
first-order regular variation property. By Equation~\eqref{eq:equiv_quantile},
this property can be equivalently expressed through the tail quantile
function $U(t)$, which provides the theoretical foundation for the
consistency of the extreme value index estimators introduced above. To
establish the asymptotic normality of these estimators, however, an
additional rate condition on $U(t)$ is required; this is captured by a
second-order regular variation property, which we state next.

\begin{definition}[Second-order regular variation (De Haan and Ferreira~\cite{de2006extreme})]
\label{def:second_order_regular_variation}
For $\gamma > 0$ and $\rho \leq 0$, the function $U(t)$ is said to
possess \emph{second-order regular variation} with parameters
$\langle \gamma, \rho \rangle$ if there exists a function $A(t)$ of
constant sign tending to zero as $t \to +\infty$ such that, for all
$x > 0$,
\begin{equation*}
  \lim_{t \to +\infty}
    \frac{U(t x) / U(t) - x^{\gamma}}{A(t)} =
    \begin{cases}
      x^{\gamma} \, \dfrac{x^{\rho} - 1}{\rho}, & \rho < 0, \\
      x^{\gamma} \log x,                         & \rho = 0.
    \end{cases}
\end{equation*}
Here $\gamma$ is the \emph{first-order} parameter,
$\rho$ is the \emph{second-order} parameter, and $A$ is the
\emph{second-order auxiliary function}.
\end{definition}

Definition~\ref{def:second_order_regular_variation} imposes a stronger condition than
the first-order regular variation in Definition~\ref{def:heavy_tailed}.
For most commonly encountered distributions, the auxiliary function $A(t)$
and the second-order parameter $\rho$ can be computed explicitly; see
Alves et al.~\cite{alves2007note} for a detailed treatment.

\subsection{Quantile Treatment Effect}
\label{sec:qte}

This subsection reviews classical methods for estimating
QTEs, which provide the analytical and methodological
foundation for the extremal QTEs studied in the
main text. We first introduce the QTE within the potential
outcome framework of Rosenbaum and Rubin~\cite{rosenbaum1983central}, making
explicit the identifying assumptions required for causal interpretation.
We then describe the classical fixed-quantile estimator
of~Firpo~\cite{firpo2007efficient}, and finally recall the additional
assumptions needed to extend the framework to dynamic sequences of
quantiles as developed by~Zhang~\cite{zhang2018extremal}. The detailed estimation procedure for
dynamic QTE is deferred to
Section~\ref{sec:extremal_quantile}. The assumptions
presented in this subsection will be used throughout the paper.

Within the potential outcome framework, let $D \in \{0, 1\}$ denote the
treatment variable and let $Y \in \mathbb{R}$ denote the outcome
variable. For $j \in \{0, 1\}$, the potential outcome of $Y$ when
$D = j$ is written as $Y(j)$.

\begin{definition}[QTE]
\label{def:qte}
Let $F_0(y)$ and $F_1(y)$ denote the distribution functions of the
potential outcomes $Y(0)$ and $Y(1)$, respectively. For
$\tau \in (0, 1)$, the \emph{$\tau$-QTE} is
defined as
\begin{equation*}
  \Delta(\tau) = q_1(\tau) - q_0(\tau),
\end{equation*}
where, for $j \in \{0, 1\}$, the \emph{quantile function}
$q_j(\tau)$ is given by
\begin{equation}
  q_j(\tau) = F_j^{\leftarrow}(\tau)
            = \inf\{y \in \mathbb{R} : F_j(y) \geq \tau\}.
  \label{eq:qte_quantile}
\end{equation}
For $j \in \{0, 1\}$, by the quantile function
definition, Equation~\eqref{eq:qte_quantile} and the tail quantile function
$U_j(\cdot)$, we have
\begin{equation*}
  q_j(1 - \tau) = \inf\{y \in \mathbb{R} : F_j(y) \geq 1 - \tau\}
               = U_j(1 / \tau).
\end{equation*}
\end{definition}

Causal inference from observational data faces two key challenges. First,
the problem of missing counterfactual outcomes: each sample unit is
observed under only one treatment status, and the observed outcome is
related to the potential outcomes through
$Y = D Y(1) + (1 - D) Y(0)$. Second, causal effects are easily
confounded: one or more covariates simultaneously affect both the
treatment and the outcome, blurring or distorting their relationship.
Concretely, confounding arises when the covariate distributions differ
across treatment groups, since the observed difference in outcomes then
reflects both the treatment effect and the compositional imbalance.
To address these issues and identify causal effects from observational
data, additional assumptions are required.

\begin{assumption}
\label{asm:causal_identification}
Let $X$ be the covariate vector and $\mathrm{supp}(X)$ its compact support.
\begin{enumerate}
\def\theenumi{\roman{enumi}}
\def\labelenumi{(\theenumi)}
  \item \emph{Unconfoundedness:}
        $\bigl(Y(1), Y(0)\bigr) \perp D \mid X$, i.e., conditional on $X$,
        the potential outcomes are independent of the treatment.
  \item \emph{Common support:}
        there exists a constant $c$ with $0 < c < 1$ such that, for every
        $x \in \mathrm{supp}(X)$, $c < \pi(x) < 1 - c$, where
        $\pi(x) = \mathbb{P}(D = 1 \mid X = x)$ is the propensity score.
\end{enumerate}
\end{assumption}

These are two standard assumptions in the causal inference
literature; see Rosenbaum and Rubin~\cite{rosenbaum1983central},
Zhang~\cite{zhang2018extremal} and Deuber et al.~\cite{deuber2024estimation}.
The unconfoundedness assumption means that, conditional on all observed
confounders, the treatment assignment is independent of the potential
outcomes; the common support assumption guarantees that, for any value
of the covariates, no sample unit is restricted to be assigned to only
the treated or only the control group.

For a fixed quantile level $\tau \in (0, 1)$,
Firpo~\cite{firpo2007efficient} proposed the following estimator under
Assumption~\ref{asm:causal_identification}. By inverse propensity score weighting to
adjust for confounding and minimizing the empirical quantile loss, the
$\tau$-quantile estimator of the potential outcome $Y(j)$ is given by
\begin{equation}
  \widehat{q}_j(\tau)
    = \mathop{\mathrm{argmin}}_{q \in \mathbb{R}}
      \sum_{i=1}^{n}
        \left(\frac{D_i}{\widehat{\pi}(X_i)}\right)^{\!j}
        \left(\frac{1 - D_i}{1 - \widehat{\pi}(X_i)}\right)^{\!1 - j}
        \,(Y_i - q)\,\bigl(\tau - \mathbf{1}_{\{Y_i \leq q\}}\bigr),
  \label{eq:qte_firpo}
\end{equation}
and the corresponding $\tau$-QTE estimator is
\begin{equation*}
  \widehat{\Delta}(\tau) = \widehat{q}_1(\tau) - \widehat{q}_0(\tau).
\end{equation*}
Here, $(D_i, Y_i, X_i)_{i=1}^{n}$ are independent observations drawn
from the population $(D, Y, X)$, and $\widehat{\pi}(X_i)$ is an
estimator of the propensity score
$\pi(X_i) = \mathbb{P}(D = 1 \mid X = X_i)$. In this work, we follow
the nonparametric sieve approach used in Firpo~\cite{firpo2007efficient},
Zhang~\cite{zhang2018extremal} and Deuber et al.~\cite{deuber2024estimation}
to estimate $\pi(X_i)$; the required assumptions are stated next
(see also Hirano et~al.~\cite{hirano2003efficient}).

\begin{assumption}
\label{asm:sieve_assumptions}
Let $X$ be the covariate vector and $\mathrm{supp}(X)$ its compact support.
\begin{enumerate}
\def\theenumi{\roman{enumi}}
\def\labelenumi{(\theenumi)}
  \item The covariate $X$ is an $r$-dimensional continuous random variable
        with density $f_X(x)$, and there exists $0 < d < 1$ such that
        $d < f_X(x) < 1/d$ for all $x \in \mathrm{supp}(X)$.
  \item The propensity score function $\pi(X)$ is $s$-times
        continuously differentiable with $s \geq 4r$, and derivatives of all orders are bounded.
  \item For $j \in \{0, 1\}$, the conditional expectation
        $\mathbb{E}\!\bigl(\tau_n - \mathbf{1}_{\{Y(j) > q_j(1 - \tau_n)\}}
          \mid x\bigr)$ is $t$-times continuously differentiable in $x$
        with $t \in \mathbb{Z}^+$, and all partial derivatives are
        uniformly bounded by $M_n$ for $x \in \mathrm{supp}(X)$.
  \item Let $h_n > 0$ be a sieve complexity parameter and
        let $\pi_{h_n}(x) = L(H_{h_n}(x)^T \pi_n)$ be the sieve
        approximation of $\pi(x)$, where $L$ is the sigmoid function
        $L(u) = 1/(1 + e^{-u})$ and
        $H_{h_n} = (H_{h_n,1}, \dots, H_{h_n,h_n})^T$ is the vector of
        sieve basis functions on $\mathbb{R}^r$.  Define
        $\zeta(h_n) = \sup_{x \in \mathrm{supp}(X)} \|H_{h_n}(x)\|$
        and assume the following hold:
        \begin{equation*}
          \frac{\zeta(h_n)^{2} \, h_n}{\sqrt{n}} \to 0, \quad
          \frac{\tau_n \, \zeta(h_n)^{10} \, h_n}{n} \to 0, \quad
          n \tau_n \zeta(h_n)^{6} \, h_n^{-s/r} \to 0, \quad
          \frac{n M_n}{\tau_n \, h_n^{t/r}} \to 0.
        \end{equation*}
\end{enumerate}
\end{assumption}

For more discussion and details on
Assumption~\ref{asm:sieve_assumptions}, we refer to~Zhang~\cite{zhang2018extremal}.

Zhang~\cite{zhang2018extremal}, building on the classical
Firpo~\cite{firpo2007efficient} estimator, extended the fixed-quantile
framework to a dynamic sequence of quantiles $\tau_n \to 0$ as
$n \to +\infty$, a procedure also adopted
by~Deuber et al.~\cite{deuber2024estimation} and this work. Zhang showed
that, for an intermediate level, i.e.\ $n\tau_n \to +\infty$, the
difference $\widehat{q}_1(1-\tau_n) - \widehat{q}_0(1-\tau_n)$ of the
estimated quantiles is asymptotically normal. For an extreme level, i.e.\
$n\tau_n \to a \geq 0$, however, asymptotic normality no longer holds.
Subsequently, Deuber et al.~\cite{deuber2024estimation} obtained an
extremal QTE estimator that is asymptotically
normal at extreme levels, by making use of an asymptotic tail
approximation. In theory, the QTE is
location-invariant: applying the same location shift to the potential
outcome distributions does not change the QTE.
Deuber's extremal QTE estimator, however, does
not possess this property. We now state the additional assumptions on the potential
outcome distributions imposed by~Zhang~\cite{zhang2018extremal} in this extension.

\begin{assumption}
\label{asm:potential_outcome_distributions}
For $j \in \{0, 1\}$,
\begin{enumerate}
\def\theenumi{\roman{enumi}}
\def\labelenumi{(\theenumi)}
  \item the potential outcome $Y(j)$ is a continuous random variable with
        density $f_j(y)$, and $f_j(y)$ is monotonically decreasing on its
        right tail;
  \item the cumulative distribution function of $Y(j)$ is $F_j(y)$, and
        $F_j(y) \in \mathcal{D}(G_\gamma)$;
  \item conditional on the covariate $X$, the potential outcome
        $Y(j) \mid X$ is a continuous random variable with conditional
        density $f_{j \mid X}(y)$.
\end{enumerate}
\end{assumption}

\section{Causal Extreme Value Index Estimator}
\label{sec:causal_eviestimator}

Building on the extreme value theory and the QTE identification
conditions reviewed in Section~\ref{sec:preliminaries}, this section
carries out the first step of our construction.  To endow the extremal
QTE estimator with location invariance, we first need to construct an
EVI estimator that is itself location-invariant and, at the same time,
capable of causal identification.  To this end, we build on the
classical Fraga estimator of
Fraga~Alves~\cite{fraga2001location} and adapt it through inverse
propensity score weighting, so that it eliminates confounding effects.
This section thus presents the main methodological contribution of the
paper: a \emph{causal Fraga estimator} of the extreme value index
$\gamma$.  We first construct the proposed estimator, then establish its
weak consistency and asymptotic normality.

\subsection{Definition of the Causal Fraga Estimator}
\label{sec:causal_fraga_form}

We now introduce the proposed causal Fraga estimator. Let
$\alpha_n \to 0$ with $k = n \alpha_n \to +\infty$, $\beta_n \to 0$
with $k_0 = n \beta_n \to +\infty$ and $\beta_n / \alpha_n \to 0$.
For $j \in \{0, 1\}$, the causal Fraga estimators of
$\gamma$ are
\begin{equation}
  \widehat{\gamma}_{1}^{F}(\beta_n, \alpha_n)
    = \frac{1}{n \beta_n} \sum_{i=1}^{n}
      \frac{D_i}{\widehat{\pi}(X_i)}
      \mathbf{1}_{\{Y_i > \widehat{q}_1(1 - \beta_n)\}}
      \log \frac{Y_i - \widehat{q}_1(1 - \alpha_n)}
                {\widehat{q}_1(1 - \beta_n) - \widehat{q}_1(1 - \alpha_n)},
  \label{eq:causal_fraga_treated}
\end{equation}
and
\begin{equation}
  \widehat{\gamma}_{0}^{F}(\beta_n, \alpha_n)
    = \frac{1}{n \beta_n} \sum_{i=1}^{n}
      \frac{1 - D_i}{1 - \widehat{\pi}(X_i)}
      \mathbf{1}_{\{Y_i > \widehat{q}_0(1 - \beta_n)\}}
      \log \frac{Y_i - \widehat{q}_0(1 - \alpha_n)}
                {\widehat{q}_0(1 - \beta_n) - \widehat{q}_0(1 - \alpha_n)}.
  \label{eq:causal_fraga_control}
\end{equation}
Here, the quantile estimators $\widehat{q}_j(1 - \beta_n)$ and
$\widehat{q}_j(1 - \alpha_n)$ are given
by Equation~\eqref{eq:qte_firpo}, and the propensity score
$\widehat{\pi}(X_i)$ appearing
in Equation~\eqref{eq:causal_fraga_treated}
and Equation~\eqref{eq:causal_fraga_control} is estimated by a nonparametric
sieve method (see~\cite{zhang2018extremal,deuber2024estimation} for details). Throughout the paper we refer to $\alpha_n$ (or $k$) as the
\emph{primary threshold parameter} and to $\beta_n$ (or $k_0$) as the
\emph{auxiliary threshold parameter}.

\begin{remark}
\label{rem:fraga_location_invariance}
The causal Fraga estimator defined in
Equations~\eqref{eq:causal_fraga_treated}
and~\eqref{eq:causal_fraga_control} is exactly location-invariant:
for any constant $c \in \mathbb{R}$, replacing each $Y_i$ by $Y_i + c$
leaves $\widehat{\gamma}_j^{F}(\beta_n, \alpha_n)$ unchanged. Indeed,
the quantile shift cancels in both the indicator
$\mathbf{1}_{\{Y_i > \widehat{q}_j(1-\beta_n)\}}$ and the log-ratio
$\log \frac{Y_i - \widehat{q}_j(1-\alpha_n)}
           {\widehat{q}_j(1-\beta_n) - \widehat{q}_j(1-\alpha_n)}$.
This property, inherited from the random-shift construction of
Fraga~\cite{fraga2001location} and Ling et al.~\cite{ling2012location},
distinguishes the causal Fraga estimator from the causal Hill estimator proposed in Deuber et al.~\cite{deuber2024estimation},
which is sensitive to the location of the threshold.
\end{remark}

\subsection{Asymptotic Properties of the Causal Fraga Estimator}
\label{sec:causal_fraga_asymptotic}

We first show that the proposed causal Fraga estimator is consistent under
the assumptions introduced in Section~\ref{sec:qte}. It is worth noting
that Deuber et~al.~\cite{deuber2024estimation} proved the weak
consistency of their causal Hill estimator using the same assumptions as
Theorem~3.1 of~Zhang~\cite{zhang2018extremal}, whereas the present work
establishes the weak consistency of the proposed causal Fraga estimator
under more relaxed conditions, thereby further broadening the
applicability of causal extreme value index estimation.

\begin{theorem}
\label{thm:causal_fraga_consistency}
Assume Assumptions~\ref{asm:causal_identification}--\ref{asm:potential_outcome_distributions} hold.
As $n \to +\infty$, suppose $\alpha_n \to 0$ with
$k = n \alpha_n \to +\infty$, and $\beta_n \to 0$ with
$k_0 = n \beta_n \to +\infty$ and $\beta_n / \alpha_n \to 0$. Then, for
each $j \in \{0, 1\}$, the extreme value index $\gamma_j > 0$ satisfies
\begin{equation*}
  \widehat{\gamma}_j^{F}(\beta_n, \alpha_n) \xrightarrow{p} \gamma_j,
\end{equation*}
where $\widehat{\gamma}_j^{F}(\beta_n, \alpha_n)$ is defined in
Equation~\eqref{eq:causal_fraga_treated}
and Equation~\eqref{eq:causal_fraga_control}.
\end{theorem}

To establish the asymptotic normality of the causal Fraga estimator,
we further introduce Assumption~\ref{asm:second_order}, which is a
second-order regular variation assumption.  It is worth
noting that this condition is routinely satisfied by common
heavy-tailed distributions such as the Pareto and Cauchy families, which
ensures the applicability of our theoretical results to typical
heavy-tail scenarios including extreme climate analysis and risk
management.

\begin{assumption}
\label{asm:second_order}
For $j \in \{0, 1\}$, let
$U_j(x) = \bigl(\tfrac{1}{1 - F_j}\bigr)^{\leftarrow}(x)$
denote the tail quantile function of the potential outcome $Y(j)$.
We assume that $U_j(x)$ possesses second-order regular variation (see Definition~\ref{def:second_order_regular_variation})
with parameters $(\gamma_j, \rho_j)$, where the extreme value index
satisfies $\gamma_j > 0$, the second-order parameter satisfies
$\rho_j \le 0$, and the second-order auxiliary function is $A_j$.
\end{assumption}

With these assumptions in place, we now present the key decomposition
of the causal Fraga estimator, which underpins the asymptotic normality
established below.

\begin{lemma}
\label{thm:causal_fraga_equation}
Let Assumptions~\ref{asm:causal_identification}--\ref{asm:second_order} hold.
As $n \to +\infty$, suppose $\alpha_n \to 0$ with $k = n \alpha_n \to +\infty$,
$\beta_n \to 0$ with $k_0 = n \beta_n \to +\infty$, and $\beta_n / \alpha_n \to 0$.
Then, for each $j \in \{0, 1\}$,
\begin{equation}\label{eq:thm3_decomp}
  \widehat{\gamma}_{j}^{F}(\beta_n, \alpha_n)
  = \gamma_j
  + \gamma_j a_{n,j} \frac{P_{n,j}}{\sqrt{k_0}}
  + b_j \Bigl(\frac{k_0}{k}\Bigr)^{\!\gamma_j} (1 + o_p(1))
  + c_j A_j\Bigl(\frac{n}{k}\Bigr)
         \Bigl(\frac{k_0}{k}\Bigr)^{\!-\rho_j} (1 + o_p(1))
  + A_{n,j},
\end{equation}
where $P_{n,j} \xrightarrow{d} \mathcal{N}(0, 1)$,
$b_j = \gamma_j / (1 + \gamma_j)$, and
$c_j = 1 / (1 - \rho_j)$; the normalized second moment
$a_{n,j}^{2} = n\,\mathbb{E}(R_{n,j}^{2}) / k_0$, where
\begin{equation*}
  R_{n,j} = \left(\frac{D}{\pi(X)}\right)^{\!j}
    \left(\frac{1 - D}{1 - \pi(X)}\right)^{\!1-j}
    \mathbf{1}_{\{Z(j) > \beta_n^{-1}\}}
    \log(Z(j)\,\beta_n) - \beta_n
\end{equation*}
with $Z(j) = 1 / (1 - F_j(Y(j)))$; moreover,
$a_{n,j}^{2}$ lies in the interval
$[\,\frac{2}{1-c} - \beta_n,\; \frac{2}{c} - \beta_n\,]$, so that
$a_{n,j} = O(1)$ is guaranteed by the boundedness of the
propensity score (Assumption~\ref{asm:causal_identification}); and
$A_{n,j} = o_p(1)$ collects the
remainder terms of the five-component decomposition of
$\widehat{\gamma}_j^{F}$, whose explicit form is given in the proof of
this lemma (Section~\ref{sec:causal_fraga_equation_proof}).
\end{lemma}

To obtain a well-defined asymptotic distribution, we additionally
require that the sequence $\{a_{n,j}\}$ admits a limit, as formalised
in Assumption~\ref{asm:an_convergence}.

\begin{assumption}
\label{asm:an_convergence}
For each $j \in \{0, 1\}$, there exists a constant $a_j > 0$ such that
$a_{n,j} \to a_j$ as $n \to +\infty$.
\end{assumption}

This assumption is in the same spirit as, but considerably less
demanding than, Assumption~B.3 in
Deuber et al.~\cite{deuber2024estimation} and Assumption~B.2 in
Zhang~\cite{zhang2018extremal}; our $a_{n,j}^{2}$ is an unconditional
second moment that is automatically bounded under
Assumption~\ref{asm:causal_identification}, so that
Assumption~\ref{asm:an_convergence} only imposes the existence of the
limit $a_j$.

Fraga~\cite{fraga2001location} and Ling et al.~\cite{ling2012location} discuss the
optimal choice of the auxiliary threshold parameter, which minimizes
the mean squared error, for the
location-invariant extreme value index estimator, but restrict
attention to the special case where the second-order auxiliary function
satisfies $A(t) = O(t^{\rho})$ as $t \to +\infty$, yielding different
optimal choices depending on the subcase.
In practice, however, it is difficult to determine which subcase
applies for a given data set, and moreover such an optimal choice is
still subject to an asymptotic bias.  In contrast, rather than pursuing the optimal choice of the auxiliary
threshold parameter $k_0$ (or~$\beta_n$), this paper restricts the
convergence rate of $k_0$, which suffices to establish the asymptotic
normality of $\widehat{\gamma}_j^{F}$ and is first-order unbiased.

Following Zhang~\cite{zhang2018extremal} and
Deuber et al.~\cite{deuber2024estimation}, in
analysing the distribution of the extreme value index estimator, a
common assumption is
$\sqrt{k}\,A_j(n/k) \to \lambda_j \in \mathbb{R}$
(notice that this condition implies
$\sqrt{k_0}\,A_j(n/k) \to 0$).  In addition, to make the asymptotic
bias negligible, we require $(k_0 / k)^{\gamma_j} = o(k_0^{-1/2})$,
and the residual term $A_{n,j}$ in
Equation~\eqref{eq:thm3_decomp} further satisfies $o_p(k_0^{-1/2})$.
The asymptotic distribution then follows, as summarised in the
following theorem.

\begin{theorem}
\label{thm:causal_fraga_clt}
Let Assumptions~\ref{asm:causal_identification}--\ref{asm:an_convergence} hold.
As $n \to +\infty$, suppose $\alpha_n \to 0$ with $k = n \alpha_n \to +\infty$,
$\beta_n \to 0$ with $k_0 = n \beta_n \to +\infty$, and
$\beta_n / \alpha_n \to 0$.
If for each $j \in \{0, 1\}$, $\sqrt{k}\,A_j(n/k) \to \lambda_j \in \mathbb{R}$,
$(k_0 / k)^{\gamma_j} = o(k_0^{-1/2})$, and
$A_{n,j} = o_p(k_0^{-1/2})$, then,
\begin{equation*}
  \sqrt{k_0}\,\bigl(\widehat{\gamma}_{j}^{F}(\beta_n, \alpha_n) - \gamma_j\bigr)
  \xrightarrow{d} \mathcal{N}\!\bigl(0,\; \gamma_j^{2} a_j^{2}\bigr).
\end{equation*}
\end{theorem}

\begin{remark}
\label{rem:causal_fraga_rate_conditions}
In practice, all three additional assumptions in
Theorem~\ref{thm:causal_fraga_clt} are achievable by choosing $k$ and
$k_0$ appropriately.  The condition
$\sqrt{k}\,A_j(n/k) \to \lambda_j$ is satisfied by taking $k$
smaller (see Deuber et al.~\cite{deuber2024estimation}).  For the two
remaining rate conditions, suppose $k_0 = O(k^{m})$ for some
$m \in (0, 1)$ (the upper bound follows from $k_0 / k \to 0$).  Then
\begin{align*}
  (k_0 / k)^{\gamma_j} = o(k_0^{-1/2})
  &\iff m < \frac{2 \gamma_j}{2 \gamma_j + 1},
\end{align*}
and the $o_p(1)$ rate of $A_{n,j}$ (from
Theorem~\ref{thm:causal_fraga_consistency}) makes
$A_{n,j} = o_p(k_0^{-1/2})$ once $k_0$ is taken smaller.  Thus
controlling the growth rates of both $k$ and $k_0$ (i.e., taking $k$
smaller and choosing $m$ in the above interval) yields all three
conditions.
\end{remark}

For the asymptotic analysis of the extremal QTE estimator in
Section~\ref{sec:asymptotic_properties}, we need the joint asymptotic
distribution of the causal Fraga estimators for the treated and control
arms, which we establish in the following theorem.

\begin{theorem}
\label{thm:causal_fraga_joint_clt}
Let Assumptions~\ref{asm:causal_identification}--\ref{asm:an_convergence} hold.
As $n \to +\infty$, suppose $\alpha_n \to 0$ with $k = n \alpha_n \to +\infty$,
$\beta_n \to 0$ with $k_0 = n \beta_n \to +\infty$, and
$\beta_n / \alpha_n \to 0$.  Assume additionally that
$\sqrt{k}\,A_j(n/k) \to \lambda_j \in \mathbb{R}$,
$(k_0 / k)^{\gamma_j} = o(k_0^{-1/2})$, and
$A_{n,j} = o_p(k_0^{-1/2})$ for $j \in \{0, 1\}$.
Then the joint asymptotic distribution
of the causal Fraga estimators for $j = 1$ and $j = 0$ is
\begin{equation}
  \sqrt{k_0}\,
  \begin{pmatrix}
    \widehat{\gamma}_1^{F}(\beta_n, \alpha_n) - \gamma_1 \\[4pt]
    \widehat{\gamma}_0^{F}(\beta_n, \alpha_n) - \gamma_0
  \end{pmatrix}
  \xrightarrow{d}
  \mathcal{N}\!\left(
    \mathbf{0},\;
    \begin{pmatrix}
      \gamma_1^{2} a_1^{2} & 0 \\[4pt]
      0 & \gamma_0^{2} a_0^{2}
    \end{pmatrix}
  \right),
  \label{eq:joint_clt}
\end{equation}
where $a_1^{2}$ and $a_0^{2}$ are defined as in
Lemma~\ref{thm:causal_fraga_equation}.
\end{theorem}

\section{Extremal Quantile Treatment Effect Estimator}
\label{sec:extremal_quantile}

Building on the location-invariant causal Fraga estimator of
Section~\ref{sec:causal_eviestimator}, this section carries out the
second step of our construction.  The main goal of this
section is to study the estimation and statistical inference of the
extreme $(1-\tau_n)$-quantile treatment effect under heavy-tailed
distributions; the $\tau_n$ case follows by
analogy. Depending on the convergence rate of $\tau_n \to 0$ as
$n \to +\infty$, this section distinguishes two regimes: the
\emph{extreme level} defined by $n\tau_n \to a \geq 0$, and the
\emph{intermediate level} defined by $n\tau_n \to +\infty$. To endow
the extremal QTE estimator with location invariance, we modify the
multiplicative extrapolation used in~Deuber
et al.~\cite{deuber2024estimation} into a difference extrapolation
scheme, and construct a location-invariant extremal QTE estimator that
builds upon the causal Fraga estimator of
Section~\ref{sec:causal_eviestimator}.  The estimators constructed
from the two extrapolation schemes possess similar asymptotic
properties; details are given in the rest of this section and in
Section~\ref{sec:mult_extrapolation}. We then establish the asymptotic
normality of the proposed estimator.

\subsection{Definition of the Extremal QTE Estimator}
\label{sec:extremal_qte}

For the extreme-level regime where $n\tau_n \to a \geq 0$, the
empirical quantile estimator given by
Equation~\eqref{eq:qte_firpo} in Section~\ref{sec:qte} suffers from
severe bias. The underlying cause is that the expected number of
observations exceeding the target quantile level is $a < \infty$,
leaving the empirical quantile with insufficient effective sample
size and incurring a bias that cannot be neglected.

A remedy for this problem is to use extrapolation to obtain an accurate
extreme-quantile estimator (see Deuber et al.~\cite{deuber2024estimation}).
Specifically, for sufficiently large $n$, applying
Equation~\eqref{eq:equiv_quantile}
to the tail quantile function yields
\begin{equation}
  q(1-\tau_n) \approx q(1-\alpha_n) \left( \frac{\alpha_n}{\tau_n} \right)^{\gamma}.
  \label{eq:qte_extrapolation}
\end{equation}
Here $\alpha_n$ is the intermediate level
($n\alpha_n \to +\infty$) introduced in
Section~\ref{sec:causal_fraga_form}, while $\tau_n$ corresponds to
the extreme level ($n\tau_n \to a \geq 0$).

Substituting the intermediate-level empirical quantile estimator
(Equation~\eqref{eq:qte_firpo}) for
$\widehat{q}_j(1-\alpha_n)$ and the causal Fraga estimator
(Equations~\eqref{eq:causal_fraga_treated}
and~\eqref{eq:causal_fraga_control}) for
$\widehat{\gamma}_j^{F}(\beta_n, \alpha_n)$ into the extrapolation
relation~\eqref{eq:qte_extrapolation} yields the following
\emph{multiplicative extrapolation estimator} of the extreme quantile
under the potential outcome with treatment $j \in \{0, 1\}$:
\begin{equation}
  \widehat{Q}_j^{\mathrm{mult}}(1-\tau_n)
  = \widehat{q}_j(1-\alpha_n)
    \left( \frac{\alpha_n}{\tau_n} \right)^{\widehat{\gamma}_j^{F}(\beta_n, \alpha_n)}.
  \label{eq:extremal_qte_estimator}
\end{equation}

The conventional \emph{multiplicative extrapolation QTE estimator} is
the difference of the per-arm multiplicative estimators
\eqref{eq:extremal_qte_estimator}:
\begin{equation}
  \widehat{\Delta}^{\mathrm{mult}}(1-\tau_n)
  = \widehat{Q}_1^{\mathrm{mult}}(1-\tau_n) - \widehat{Q}_0^{\mathrm{mult}}(1-\tau_n).
  \label{eq:qte_multiplicative}
\end{equation}
Unlike the population QTE $\Delta(1-\tau_n) = q_1(1-\tau_n) - q_0(1-\tau_n)$,
which is invariant under a common location shift, the multiplicative
extrapolation QTE estimator $\widehat{\Delta}^{\mathrm{mult}}$ is \emph{not}
location-invariant.

To obtain a location-invariant estimator, we adapt the random-shift
construction underlying the causal Fraga estimator to the extrapolation step.
Since $q_j(1-\tau) = U_j(1/\tau)$ with $U_j$ regularly varying of index
$\gamma_j$,
\begin{equation*}
  \frac{q_j(1-\tau_n) - q_j(1-\alpha_n)}
       {q_j(1-\beta_n) - q_j(1-\alpha_n)}
  \approx \frac{(\alpha_n/\tau_n)^{\gamma_j} - 1}
                {(\alpha_n/\beta_n)^{\gamma_j} - 1},
\end{equation*}
where the location parameter cancels in the differences.  Substituting
the estimated quantities yields the \emph{difference extrapolation
estimator}
\begin{equation}
  \widehat{Q}_j(1-\tau_n)
  = \widehat{q}_j(1-\alpha_n)
    + \bigl[\widehat{q}_j(1-\beta_n) - \widehat{q}_j(1-\alpha_n)\bigr]
      \cdot
      \frac{(\alpha_n/\tau_n)^{\widehat{\gamma}_j^{F}(\beta_n,\alpha_n)} - 1}
           {(\alpha_n/\beta_n)^{\widehat{\gamma}_j^{F}(\beta_n,\alpha_n)} - 1}.
  \label{eq:extremal_qte_estimator_add}
\end{equation}
This difference-extrapolation construction parallels Dekkers and de~Haan~\cite{dekkers1989estimation},
but employs the causal extreme value index $\widehat{\gamma}_j^{F}$ at
the two arbitrary threshold levels $\alpha_n$ and $\beta_n$, rather
than a univariate moment estimator at fixed ratios.  This allows the
extrapolation step to inherit the location invariance of the underlying
causal Fraga estimator.
Then we propose the following \emph{location-invariant extremal quantile
treatment effect estimator}
\begin{equation}
  \widehat{\Delta}(1-\tau_n)
  = \widehat{Q}_1(1-\tau_n)
    - \widehat{Q}_0(1-\tau_n).
  \label{eq:extremal_qte}
\end{equation}

\subsection{Asymptotic Properties of the Extremal QTE Estimator}
\label{sec:asymptotic_properties}

We first give the following lemma which shows that the asymptotic
behavior of the difference extrapolation estimator
$\widehat{Q}_j(1-\tau_n)$ defined in
Equation~\eqref{eq:extremal_qte_estimator_add} only depends on the
asymptotic distribution of the EVI estimator.  The properties and the
proofs for the QTE estimator based on the multiplicative extrapolation
are deferred to Section~\ref{sec:mult_extrapolation}.

\begin{lemma}
\label{lem:extreme_qte_linearization}
Let Assumptions~\ref{asm:causal_identification}--\ref{asm:second_order}
hold. As $n \to +\infty$, suppose $\alpha_n \to 0$ with
$k = n\alpha_n \to +\infty$, $\beta_n \to 0$ with
$k_0 = n\beta_n \to +\infty$, and $\beta_n / \alpha_n \to 0$.
If for each $j \in \{0, 1\}$,
$\sqrt{k}\,A_j(n/k) \to \lambda_j \in \mathbb{R}$,
$(k_0 / k)^{\gamma_j} = o(k_0^{-1/2})$,
$A_{n,j} = o_p(k_0^{-1/2})$ and $\log(\beta_n / \tau_n) = o(k_0^{1/2})$, then
\begin{equation*}
  \frac{\sqrt{k_0}}{\log(\beta_n / \tau_n)}
  \left( \frac{\widehat{Q}_j(1-\tau_n)}{q_j(1-\tau_n)} - 1 \right)
  = \sqrt{k_0}\,
    \bigl( \widehat{\gamma}_j^{F}(\beta_n, \alpha_n) - \gamma_j \bigr)
    + o_p(1).
\end{equation*}
In particular, $\widehat{Q}_j(1-\tau_n) / q_j(1-\tau_n) \xrightarrow{p} 1$.
\end{lemma}

Lemma~\ref{lem:extreme_qte_linearization} permits the extreme-level
regime $n\tau_n \to 0$, but the condition
$\log(\beta_n / \tau_n) = o(k_0^{1/2})$ imposes a rate restriction:
$\tau_n$ cannot decay arbitrarily fast relative to the auxiliary
threshold $\beta_n$. Equivalently, the effective number of extreme
observations cannot be too small; otherwise, the extrapolation from
the auxiliary level to the extreme level becomes unstable. This
reveals an inherent applicability boundary of the extrapolation
approach, a finding consistent with
Deuber et al.~\cite{deuber2024estimation}.

To establish the asymptotic normality of $\widehat{\Delta}(1-\tau_n)$,
we need to analyse the joint distribution of
$\widehat{Q}_0(1-\tau_n)$ and $\widehat{Q}_1(1-\tau_n)$. This analysis
faces a challenge: the normalising factors of
$\widehat{Q}_0(1-\tau_n)$ and $\widehat{Q}_1(1-\tau_n)$ may differ in
their convergence rates, leading to oscillations in their ratio.
Inspired by Chernozhukov and Fern\'andez-Val~\cite{chernozhukov2011inference}, Zhang~\cite{zhang2018extremal}, and Deuber et al.~\cite{deuber2024estimation}, we impose a constraint on
the convergence of the ratio of normalising factors. Building on
Lemma~\ref{lem:extreme_qte_linearization}, we propose the following
normalising factor for $\widehat{\Delta}(1-\tau_n)$
\begin{equation}
  \widehat{\phi}_n
  := \frac{\sqrt{k_0}}
         {\log(\beta_n/\tau_n) \cdot
          \max\bigl\{\widehat{Q}_1(1-\tau_n),\, \widehat{Q}_0(1-\tau_n)\bigr\}}.
  \label{eq:delta_normalizing_factor}
\end{equation}

Furthermore, akin to the analysis in
Deuber et al.~\cite{deuber2024estimation}, we require that the potential
outcome distributions are either comparable in their tails or one
distribution has a heavier tail than the other. This is a fairly standard assumption satisfied by many models. We formulate this condition as follows.

\begin{assumption}
\label{asm:tail_comparability}
The quantile functions $q_0$ and $q_1$ of the potential outcomes
$Y(0)$ and $Y(1)$ satisfy
\begin{equation*}
  \frac{q_1(1-\tau_n)}{q_0(1-\tau_n)} \to \kappa \in [0, +\infty].
\end{equation*}
\end{assumption} 

We are now ready to establish the asymptotic normality of the extremal
QTE estimator.

\begin{theorem}
\label{thm:extremal_qte_asymptotic}
Let Assumptions~\ref{asm:causal_identification}--\ref{asm:tail_comparability}
hold. As $n \to +\infty$, suppose $\alpha_n \to 0$ with
$k = n\alpha_n \to +\infty$, $\beta_n \to 0$ with
$k_0 = n\beta_n \to +\infty$, and $\beta_n / \alpha_n \to 0$.
If for each $j \in \{0, 1\}$,
$\sqrt{k}\,A_j(n/k) \to \lambda_j \in \mathbb{R}$,
$(k_0 / k)^{\gamma_j} = o(k_0^{-1/2})$,
$A_{n,j} = o_p(k_0^{-1/2})$ and $\log(\beta_n / \tau_n) = o(k_0^{1/2})$, then
\begin{equation*}
  \widehat{\phi}_n\bigl(\widehat{\Delta}(1-\tau_n) - \Delta(1-\tau_n)\bigr)
  \xrightarrow{d} \mathcal{N}(0, \sigma^2),
\end{equation*}
where
\begin{equation*}
  \sigma^2
  = \min\{1, \kappa\}^2 \gamma_1^2 a_1^2
    + \min\{1, 1/\kappa\}^2 \gamma_0^2 a_0^2.
\end{equation*}
\end{theorem}

To conduct statistical inference based on
Theorem~\ref{thm:extremal_qte_asymptotic}, the asymptotic variance
$\sigma^2$ must be consistently estimated. The estimators
$\widehat{\gamma}_j^{F}$ for $\gamma_j$ are provided by
Equations~\eqref{eq:causal_fraga_treated}
and~\eqref{eq:causal_fraga_control}, and an estimator $\widehat{\kappa}$
for $\kappa$ follows naturally from the definition in
Assumption~\ref{asm:tail_comparability}, namely
\begin{equation}
  \widehat{\kappa}
  := \frac{\widehat{Q}_1(1-\tau_n)}{\widehat{Q}_0(1-\tau_n)}.
  \label{eq:kappa_estimator}
\end{equation}
In the following we propose an estimator $\widehat{a}_j$ for $a_j$.
Define the estimator of $a_j^2$ as
\begin{equation}
  \widehat{a}_{n,j}^2
  := \frac{1}{k_0} \sum_{i=1}^{n} \widehat{R}_{n,j,i}^{\,2},
  \qquad j \in \{0, 1\},
  \label{eq:an_estimator}
\end{equation}
where
\begin{equation*}
  \widehat{R}_{n,j,i}
  := \left(\frac{D_i}{\widehat{\pi}(X_i)}\right)^{\!j}
     \left(\frac{1 - D_i}{1 - \widehat{\pi}(X_i)}\right)^{\!1-j}
     \mathbf{1}_{\{Y_i > \widehat{q}_j(1-\beta_n)\}}
     \cdot \frac{1}{\widehat{\gamma}_j^{F}}\,
     \log \frac{Y_i - \widehat{q}_j(1-\alpha_n)}
               {\widehat{q}_j(1-\beta_n) - \widehat{q}_j(1-\alpha_n)}
     - \beta_n.
\end{equation*}

\begin{equation}
  \widehat{\sigma}^{2}
  := \min\{1, \widehat{\kappa}\}^{2}\,
     \bigl(\widehat{\gamma}_1^{F}\bigr)^{2}\,
     \widehat{a}_{n,1}^{2}
     + \min\{1, 1/\widehat{\kappa}\}^{2}\,
       \bigl(\widehat{\gamma}_0^{F}\bigr)^{2}\,
       \widehat{a}_{n,0}^{2}.
  \label{eq:sigma2_estimator}
\end{equation}

We next establish the consistency of the variance estimator.

\begin{theorem}
\label{thm:sigma2_consistency}
Let Assumptions~\ref{asm:causal_identification}--\ref{asm:tail_comparability} hold.
As $n \to +\infty$, suppose $\alpha_n \to 0$ with
$k = n\alpha_n \to +\infty$, $\beta_n \to 0$ with
$k_0 = n\beta_n \to +\infty$, and $\beta_n / \alpha_n \to 0$.
If for each $j \in \{0, 1\}$,
$\sqrt{k}\,A_j(n/k) \to \lambda_j \in \mathbb{R}$,
$(k_0 / k)^{\gamma_j} = o(k_0^{-1/2})$,
$A_{n,j} = o_p(k_0^{-1/2})$, and
$\log(\beta_n / \tau_n) = o(k_0^{1/2})$, then 
\begin{equation*}
  \widehat{\sigma}^{2} \xrightarrow{p} \sigma^{2},
\end{equation*}
where $\widehat{\sigma}^{2}$ is defined by
Equation~\eqref{eq:sigma2_estimator}.
\end{theorem}

Combining the asymptotic normality of
Theorem~\ref{thm:extremal_qte_asymptotic} with the consistency of the
variance estimator established in
Theorem~\ref{thm:sigma2_consistency} yields
\begin{equation*}
  \frac{\widehat{\phi}_n\bigl(\widehat{\Delta}(1-\tau_n)
                             - \Delta(1-\tau_n)\bigr)}
       {\widehat{\sigma}}
  \xrightarrow{d} \mathcal{N}(0, 1).
\end{equation*}
Hence an asymptotic $(1-\alpha)$-level confidence interval for the
extremal quantile treatment effect $\Delta(1-\tau_n)$ is
\begin{equation}
  \bigl[\,\,
    \widehat{\Delta}(1-\tau_n)
    - z_{1-\alpha/2}\,\frac{\widehat{\sigma}}{\widehat{\phi}_n},
    \;
    \widehat{\Delta}(1-\tau_n)
    + z_{1-\alpha/2}\,\frac{\widehat{\sigma}}{\widehat{\phi}_n}
  \,\bigr],
  \label{eq:confidence_interval}
\end{equation}
where $z_{1-\alpha/2}$ denotes the $(1-\alpha/2)$-quantile of the
standard normal distribution.

Moreover, the linearization in
Lemma~\ref{lem:extreme_qte_linearization} only involves the EVI
estimator through the deviation
$\sqrt{k_\gamma}\,(\widehat{\gamma}_j - \gamma_j)$, where $k_\gamma$ is
the number of tail observations entering the estimator (equal to $k_0$
for the causal Fraga estimator).  Hence the difference extrapolation
construction \eqref{eq:extremal_qte_estimator_add} and the linearization
remain valid for any causal extreme value index estimator (e.g.\ the
causal Hill estimator of Deuber et al.~\cite{deuber2024estimation}) with
$\widehat{\gamma}_j \xrightarrow{p} \gamma_j$ and
$\sqrt{k_\gamma}\,(\widehat{\gamma}_j - \gamma_j) = O_p(1)$, upon
replacing $k_0$ by $k_\gamma$ in the normalising factor, provided the
second-order term is negligible, $\sqrt{k_\gamma}\,A_j(n/k) \to 0$, and
$\log(\beta_n/\tau_n) = o(k_\gamma^{1/2})$.  In the simulation study of
Section~\ref{sec:simulations}, we also combine the proposed difference
extrapolation with the causal Hill estimator to construct an extremal
QTE estimator and examine its properties.  Full location invariance of
the resulting QTE estimator, however, requires the EVI estimator itself
to be location-invariant.

In summary, the difference extrapolation estimator
$\widehat{\Delta}(1-\tau_n)$ combines the location-invariant causal
Fraga estimator of Section~\ref{sec:causal_eviestimator} with a
difference extrapolation scheme to yield a location-invariant,
first-order unbiased estimator of the extremal QTE, together with a
consistent variance estimator and asymptotically valid confidence
intervals; its finite-sample behaviour is examined in
Section~\ref{sec:simulations}.

\section{Simulations}
\label{sec:simulations}

We conducted simulations to examine the finite-sample behavior of
the proposed difference extrapolation extremal QTE estimator
\eqref{eq:extremal_qte} and the associated confidence interval
\eqref{eq:confidence_interval}, and to compare them with alternative
methods.  All simulations were carried out in Python; the code is
available on GitHub.

\subsection{Experimental Setup}
\label{sec:sim_setup}

\paragraph{Data generating process.}
We follow a setting similar to Zhang~\cite{zhang2018extremal} and
Deuber et al.~\cite{deuber2024estimation}.  Let
$X \sim \mathrm{Uniform}(0, 1)$ and $U \sim \mathrm{Uniform}(0, 1)$ be
independent uniform random variables; the treatment is assigned by
$D = \mathbf{1}\{U \le \pi(X)\}$ with the propensity score
$\pi(x) = 0.5\,x^{2} + 0.25$.  Conditional on $X$, the potential
outcomes, shifted by a common location offset $u$, are generated from
one of the following three heavy-tailed models:
\begin{equation*}
  \begin{cases}
    H_1:\ Y(1) = 5\,S\,(1 + X) + u,\quad Y(0) = S\,(1 + X) + u; \\
    H_2:\ Y(1) = C_2\,\exp(X) + u,\quad Y(0) = C_3\,\exp(X) + u; \\
    H_3:\ Y(1) = P_{1.75 + X,\, 2} + u,\quad Y(0) = P_{1.75 + 5X,\, 1} + u,
  \end{cases}
\end{equation*}
where $S$ follows a Student-$t$ distribution with $3$ degrees of
freedom, $C_s$ is Fréchet-distributed with shape parameter $s$,
location $0$ and scale $1$, and $P_{a,\, b}$ is Pareto distributed
with shape parameter $a$ and scale $b$.  The extreme value indices of
the three models are $\gamma_1 = \gamma_0 = 1/3$ under $H_1$,
$\gamma_1 = 1/2$, $\gamma_0 = 1/3$ under $H_2$, and
$\gamma_1 = \gamma_0 = 4/7$ under $H_3$.  We consider
two sample sizes $n \in \{1000, 5000\}$, three extreme levels
$\tau_n \in \{5/n,\; 1/n,\; 5/(n\log n)\}$, the location shift
$u \in \{0, 10, 20, 30\}$, and set the primary threshold
parameter to $k = n^{0.65}$, in line with the choice of Deuber
et al.~\cite{deuber2024estimation} to facilitate comparison.  In addition,
we analyse the sensitivity of the methods to the choice of $k$.  Each
setting is replicated over $1000$ Monte Carlo runs, and all experiments
involving bootstrap methods use $1000$ bootstrap replications.

\paragraph{Methods under comparison.}
For the point estimation of the extreme value index and the extremal
quantile treatment effect, we mainly compare two methods and their
corresponding extrapolation variants.

\paragraph{Causal Hill estimator.}
The causal Hill estimator of~\cite{deuber2024estimation} is
\begin{equation*}
  \widehat{\gamma}_j^{H}(\alpha_n)
  = \frac{1}{k}\sum_{i=1}^{n}
    \left(\frac{D_i}{\widehat{\pi}(X_i)}\right)^{\!j}
    \left(\frac{1 - D_i}{1 - \widehat{\pi}(X_i)}\right)^{\!1-j}
    \mathbf{1}_{\{Y_i > \widehat{q}_j(1 - \alpha_n)\}}\,
    \bigl(\log Y_i - \log \widehat{q}_j(1 - \alpha_n)\bigr),
\end{equation*}
with $k = n\alpha_n$.  Substituting $\widehat{\gamma}_j^{H}$ for
$\widehat{\gamma}_j^{F}$ in
Equations~\eqref{eq:extremal_qte_estimator}
and~\eqref{eq:extremal_qte_estimator_add} yields the causal Hill
multiplicative and difference extrapolation QTE estimators,
respectively; the causal Hill estimator combined with the multiplicative
extrapolation is precisely the estimator used by
Deuber et al.~\cite{deuber2024estimation}.

\paragraph{Causal Fraga estimator.}
The causal Fraga estimator, our proposed location-invariant causal
extreme value index estimator, is defined in
Equations~\eqref{eq:causal_fraga_treated}
and~\eqref{eq:causal_fraga_control}; the corresponding difference
extrapolation QTE estimator is given by
Equations~\eqref{eq:extremal_qte_estimator_add}
and~\eqref{eq:extremal_qte}.

For the confidence intervals of the extremal QTE, both extrapolation
methods based on the causal Hill estimator use the bootstrap to
estimate the standard error $\widehat{\mathrm{se}}_{\mathrm{boot}}$ together
with the full-sample point estimate, yielding intervals of the form
\begin{equation*}
  \bigl[\,\widehat{\Delta}(1-\tau_n)
    - z_{1-\alpha/2}\,\widehat{\mathrm{se}}_{\mathrm{boot}},\;
    \widehat{\Delta}(1-\tau_n)
    + z_{1-\alpha/2}\,\widehat{\mathrm{se}}_{\mathrm{boot}}\,\bigr],
\end{equation*}
where $\widehat{\Delta}(1-\tau_n)$ denotes the full-sample point
estimate of the extremal QTE.  Although Deuber
et~al.~\cite{deuber2024estimation} proposed an analytical variance
estimator for their extremal QTE estimator, their simulation results
show that confidence intervals based on the bootstrap are equally
valid.  We therefore also rely on the bootstrap, which additionally
allows a direct comparison with the difference extrapolation variant of
their estimator.
The extrapolation methods based on the causal Fraga estimator construct
confidence intervals both by this bootstrap rule and by the analytical
rule of Equation~\eqref{eq:confidence_interval}.  Since both the
conditions and the derivation of the confidence intervals are
essentially the same for the multiplicative and the difference
extrapolation schemes (see Section~\ref{sec:mult_extrapolation} for
details), we only verify the analytical rule for the difference
extrapolation estimator.

\paragraph{Experimental procedure.}
The main focus of the simulation study is twofold: to investigate the
effect of the location shift $u$ on the QTE estimators of all
methods, and to analyse the sensitivity of the QTE estimators
to the choice of $k$.  In addition, we report the empirical coverage of the confidence
intervals to assess the asymptotic properties of the QTE estimators.

For the causal Fraga estimator, the auxiliary threshold $k_0$ is
chosen in an adaptive manner based on the extremal value index
estimates from both arms: we set $k_0 = k^{m}$.
Unbiasedness requires $m < 2\,\gamma_j / (1 + 2\,\gamma_j)$, whereas
the variance of the estimator increases as $k_0$ decreases.  To obtain
a suitable $k_0$, we adapt the automatic Fraga
procedure~\cite{fraga2001location} with the following modifications:
\begin{enumerate}
  \item set the initial $k_0^\star = 2\,k^{2/3}$;
  \item with $\beta_n = k_0^\star / n$, estimate the causal Fraga EVI
        $\widehat{\gamma}_j^{F}$ from
        Equations~\eqref{eq:causal_fraga_treated}
        and~\eqref{eq:causal_fraga_control} for $j \in \{0, 1\}$;
  \item compute $m_j^\star = 2\,\widehat{\gamma}_j^{F} / (1 + 2\,\widehat{\gamma}_j^{F})$
        for $j \in \{0, 1\}$;
  \item set $m = \max\bigl(m_{\min},\; \min_{j \in \{0,1\}} (m_j^\star - \eta)\bigr)$;
  \item set $k_0 = k^{m}$ and $\beta_n = k_0 / n$, shared by both arms.
\end{enumerate}
In the experiments we set $\eta = 0.05$ as a small safety margin:
since $m_j^\star$ is based on the estimated $\widehat{\gamma}_j^{F}$,
subtracting a positive $\eta$ guards against overestimation and keeps
the unbiasedness requirement $m < 2\,\gamma_j/(1+2\,\gamma_j)$ satisfied
for both arms.  By Theorem~\ref{thm:causal_fraga_clt},
$\sqrt{k_0}\,(\widehat{\gamma}_j^{F} - \gamma_j) = O_p(1)$, so the sampling
fluctuation of $\widehat{\gamma}_j^{F}$ is of order $k_0^{-1/2}$, which at
$n = 1000$ (yielding $k_0 \approx 40$) is about $0.03$; the choice
$\eta = 0.05$ comfortably covers this fluctuation.  Since $m$ cannot
vanish, we impose the lower bound $m_{\min} = 0.1$.  The initial value $k_0^\star$
follows the choice of the original Fraga
procedure~\cite{fraga2001location}; the modification consists in
subtracting a constant $\eta$ from $\min_j m_j^\star$.  For the causal
Hill based \emph{difference} extrapolation QTE estimator, which
requires the auxiliary threshold $\beta_n$, we by contrast fix
$k_0 = 2\,k^{2/3}$, i.e.\
$\beta_n = 2\,k^{2/3} / n$, without the adaptive step.  For the size
$h_n$ of the sieve basis functions in the propensity score estimation,
we use $h_n = \lfloor 2\,n^{1/11} \rfloor$, in line with the choice of
Deuber et~al.~\cite{deuber2024estimation}.

\subsection{Simulation Results}
\label{sec:sim_results}

\begin{figure}[t]
  \centering
  \includegraphics[width=\linewidth]{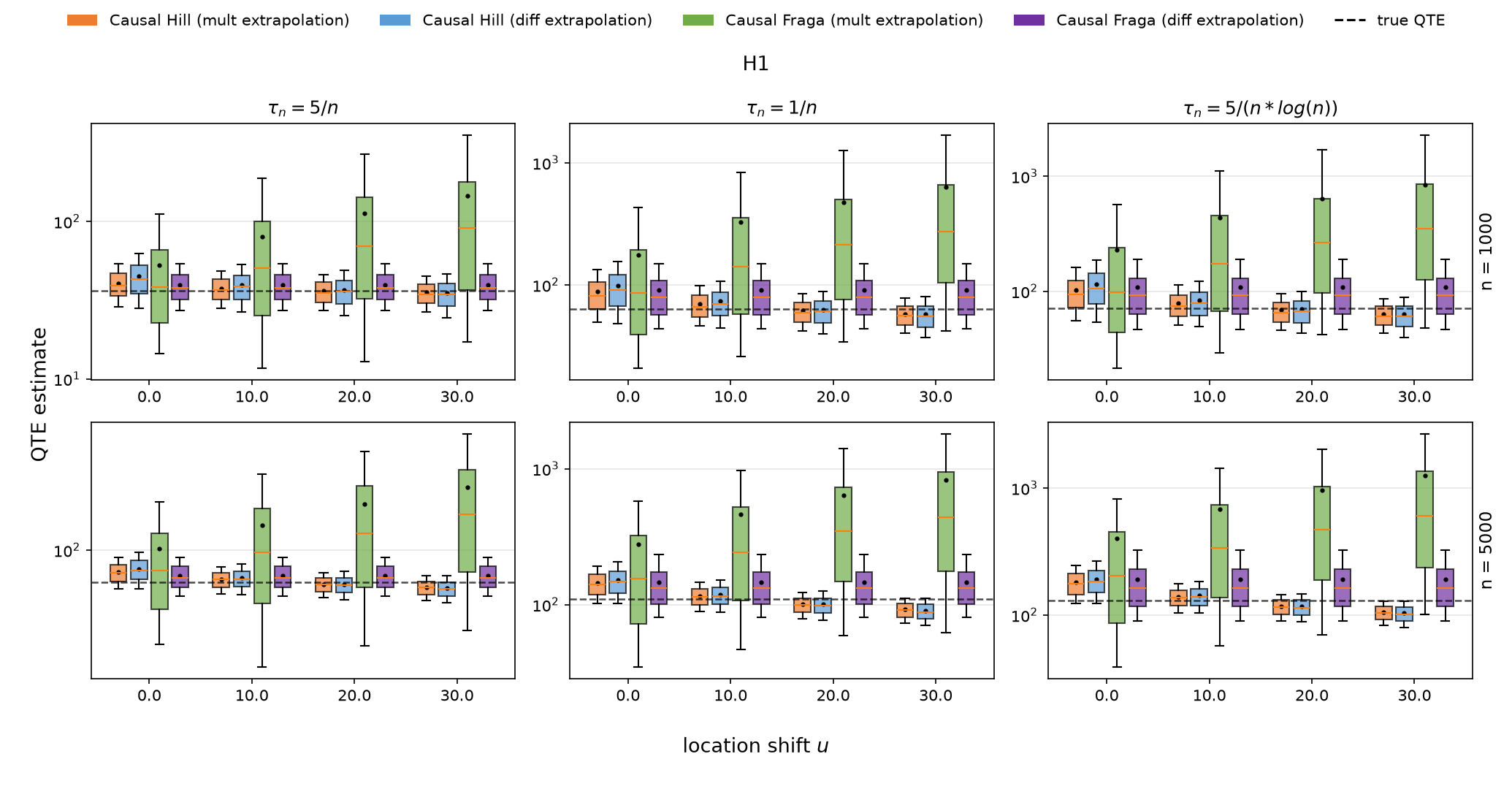}
  \caption{Quantile treatment effect estimations under the location
  shift $u$ for model $H_1$.  The boxplot whiskers correspond to the
  $0.1$ and $0.9$ quantiles, the orange line is the median, and the
  black dots are the means.}
  \label{fig:qte_shift_H1}
\end{figure}

\begin{figure}[t]
  \centering
  \includegraphics[width=\linewidth]{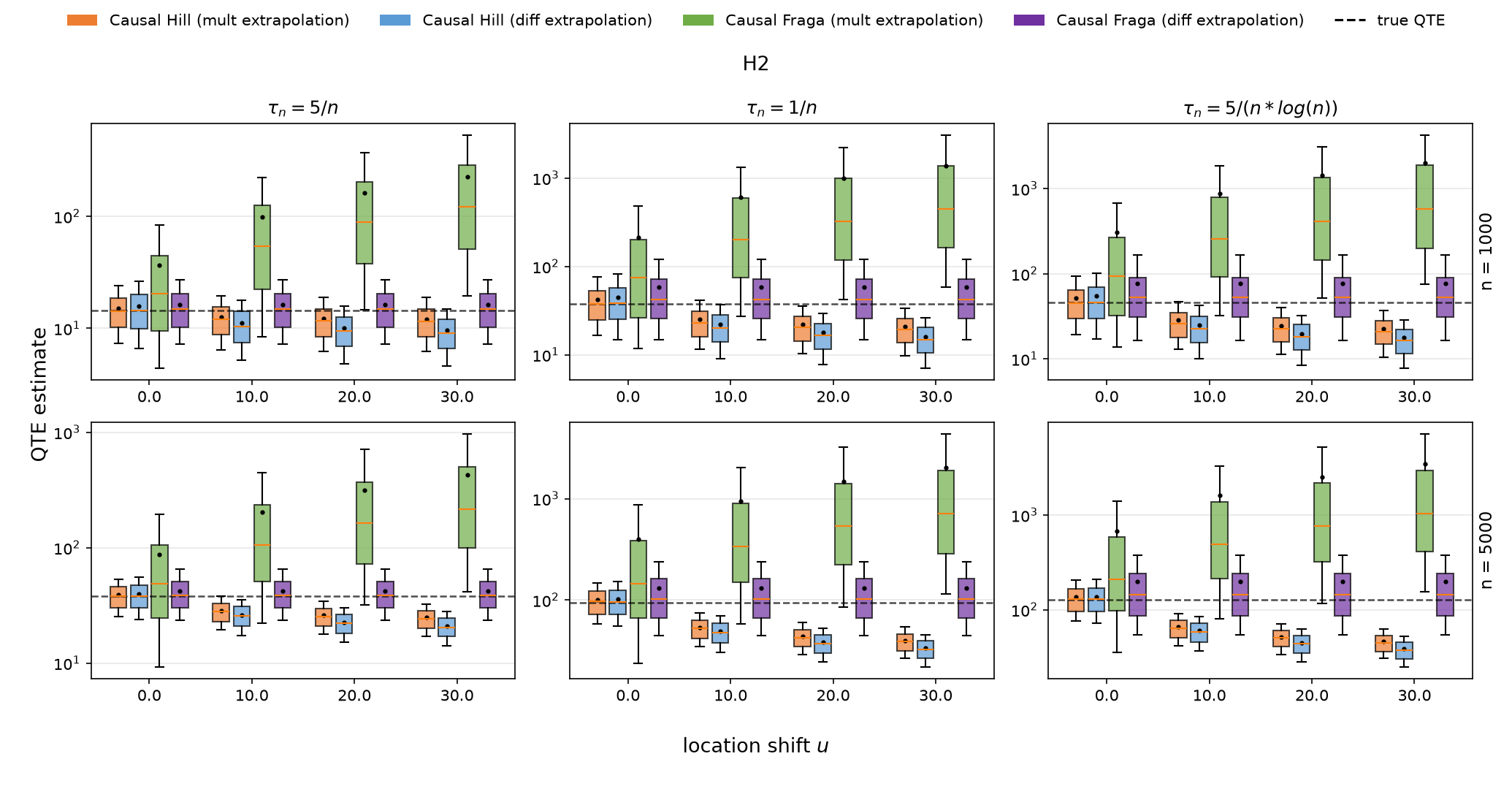}
  \caption{Quantile treatment effect estimations under the location
  shift $u$ for model $H_2$.  The boxplot whiskers correspond to the
  $0.1$ and $0.9$ quantiles, the orange line is the median, and the
  black dots are the means.}
  \label{fig:qte_shift_H2}
\end{figure}

\begin{figure}[t]
  \centering
  \includegraphics[width=\linewidth]{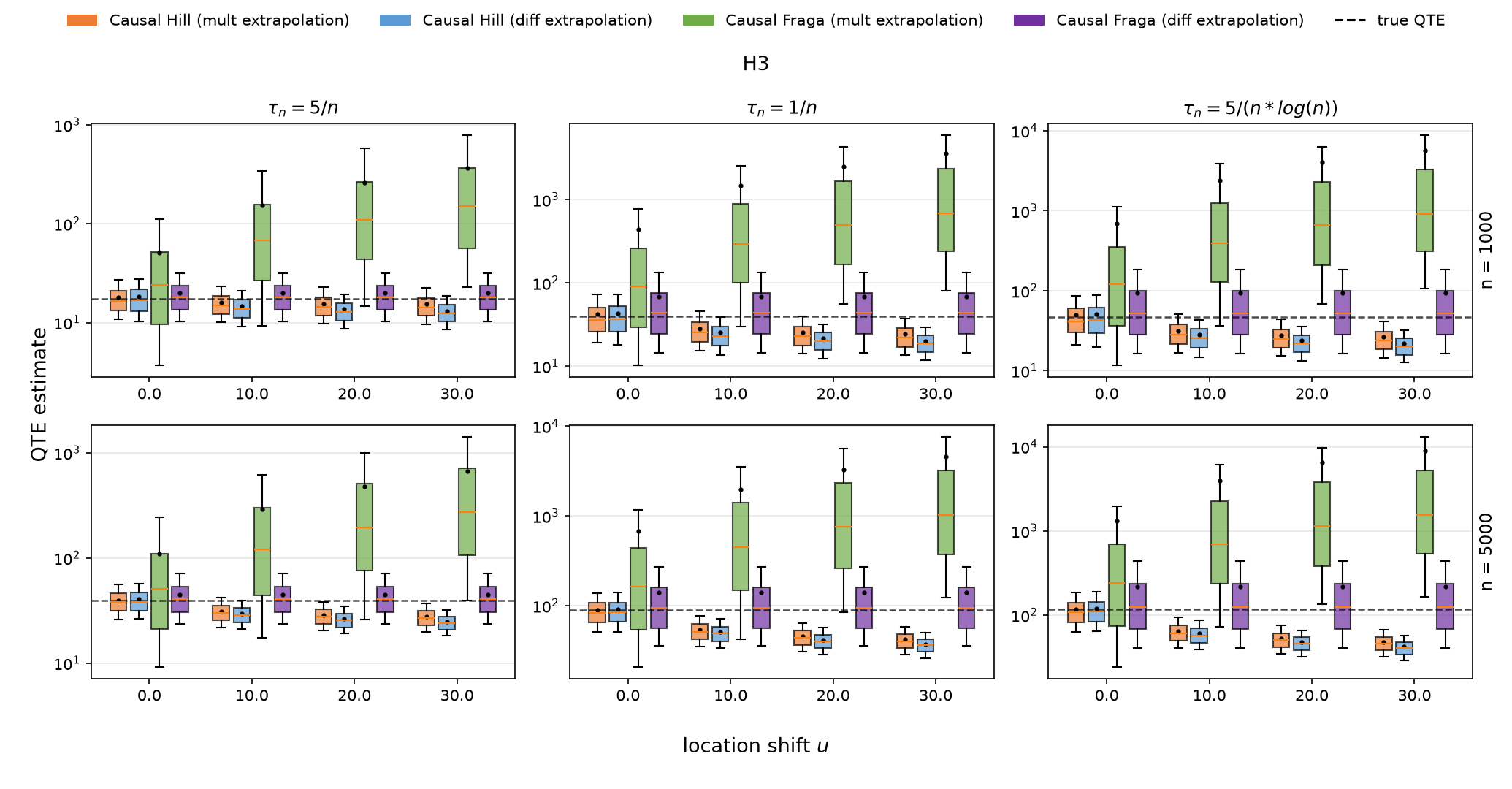}
  \caption{Quantile treatment effect estimations under the location
  shift $u$ for model $H_3$.  The boxplot whiskers correspond to the
  $0.1$ and $0.9$ quantiles, the orange line is the median, and the
  black dots are the means.}
  \label{fig:qte_shift_H3}
\end{figure}

\begin{figure}[t]
  \centering
  \includegraphics[width=\linewidth]{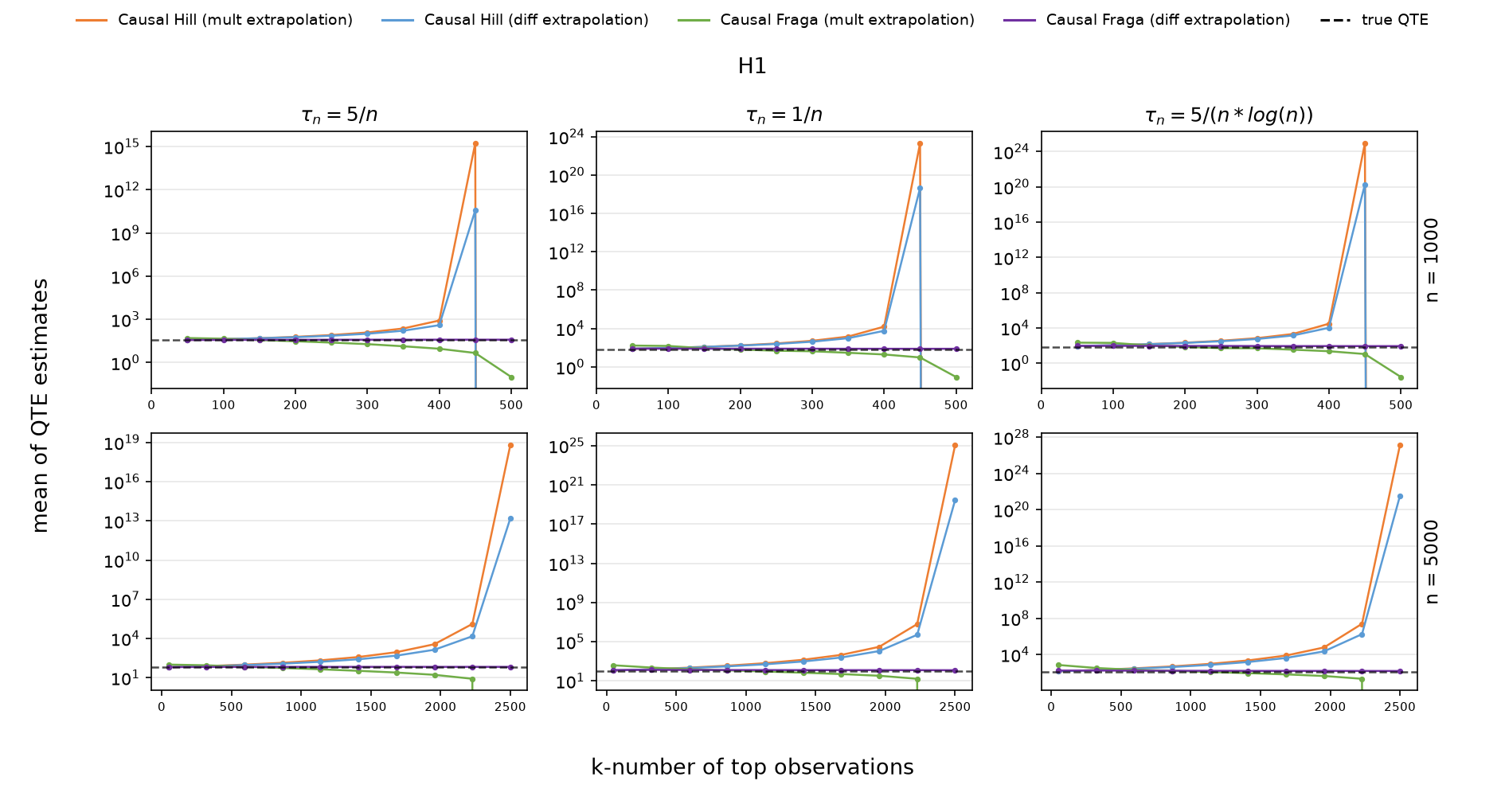}
  \caption{Mean of the quantile treatment effect estimations as a
  function of the threshold parameter $k$ for model $H_1$.}
  \label{fig:qte_k_mean_H1}
\end{figure}

\begin{figure}[t]
  \centering
  \includegraphics[width=\linewidth]{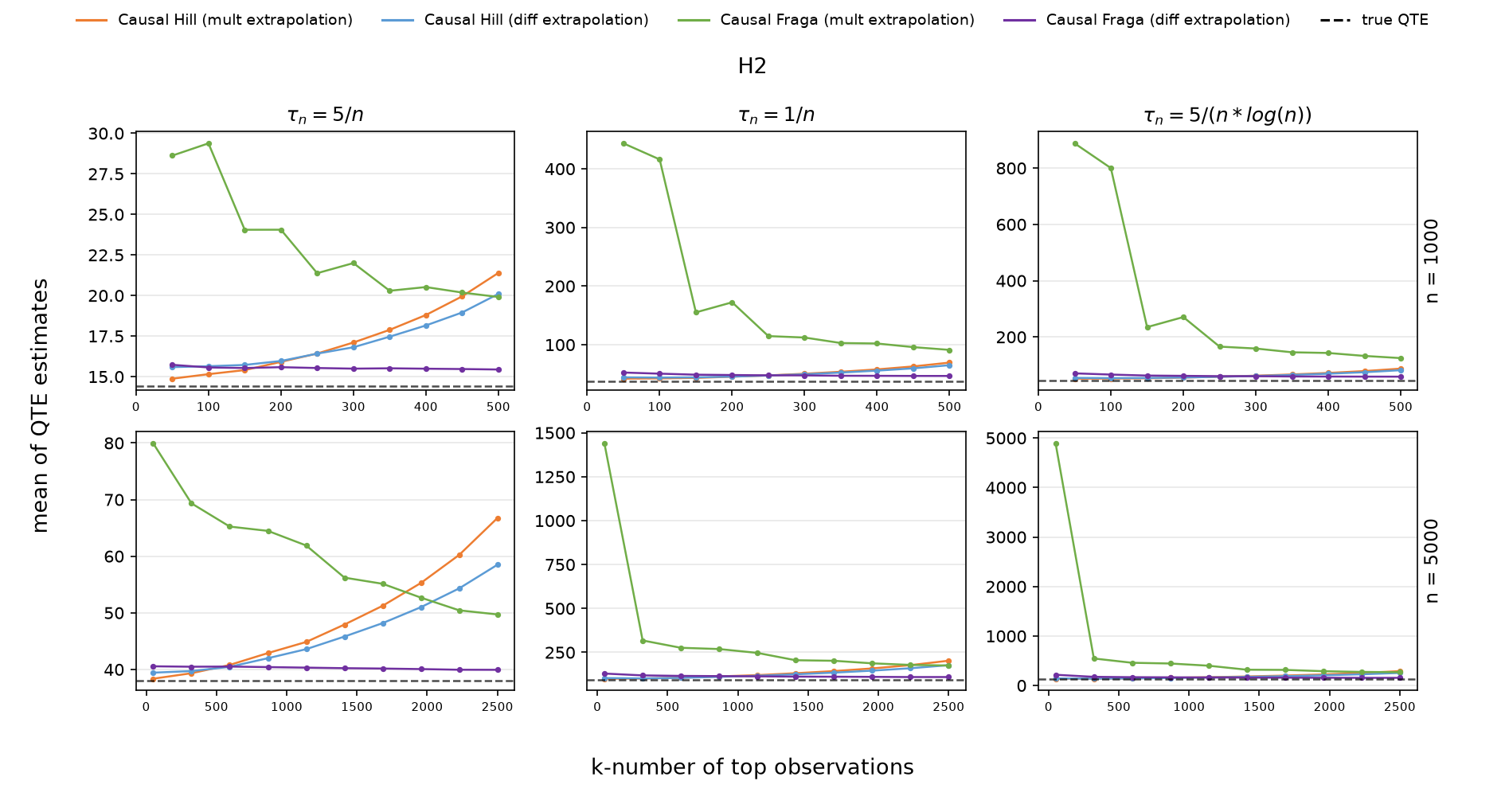}
  \caption{Mean of the quantile treatment effect estimations as a
  function of the threshold parameter $k$ for model $H_2$.}
  \label{fig:qte_k_mean_H2}
\end{figure}

\begin{figure}[t]
  \centering
  \includegraphics[width=\linewidth]{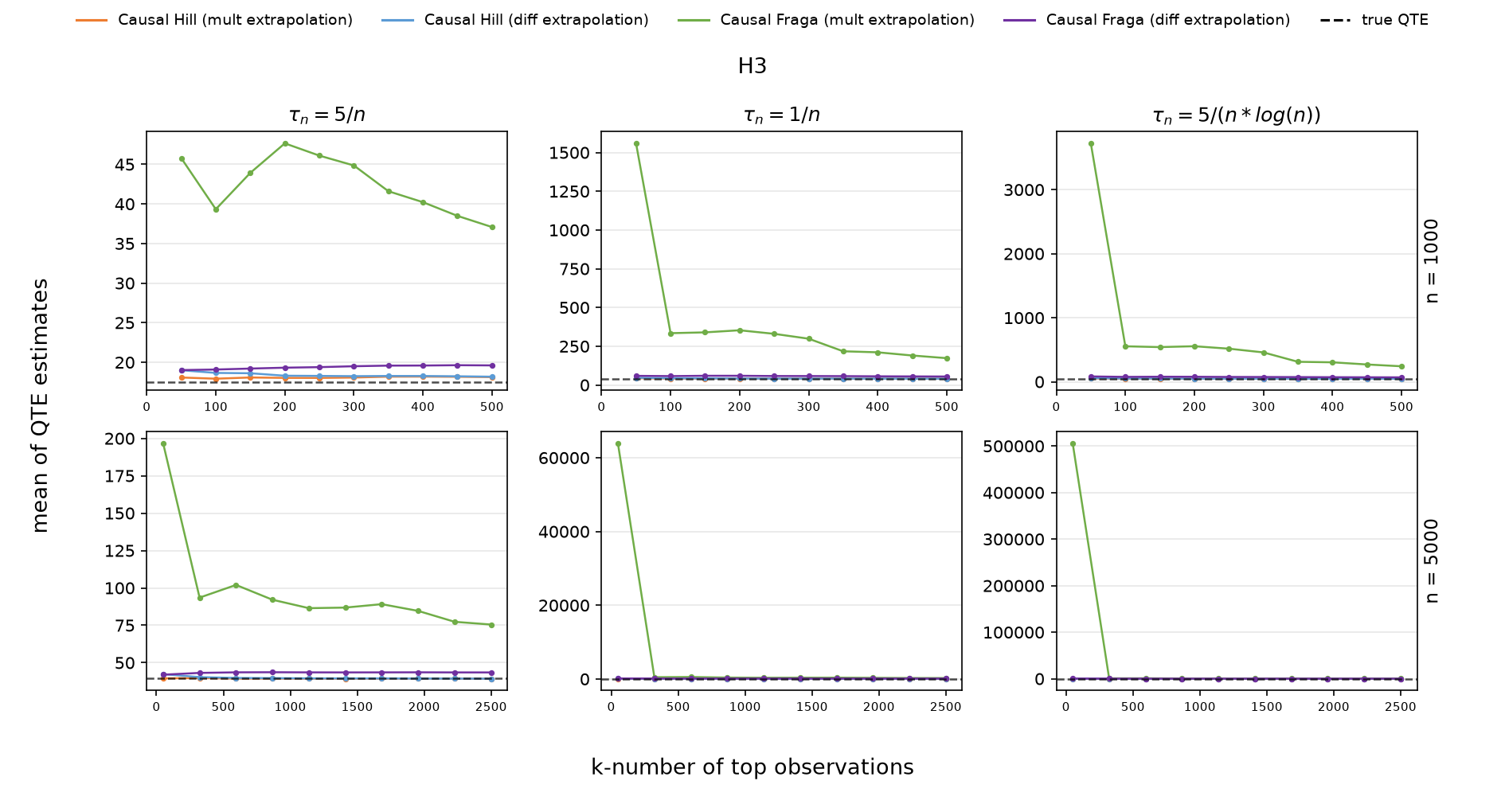}
  \caption{Mean of the quantile treatment effect estimations as a
  function of the threshold parameter $k$ for model $H_3$.}
  \label{fig:qte_k_mean_H3}
\end{figure}

Figures~\ref{fig:qte_shift_H1}--\ref{fig:qte_shift_H3} display the
distributions of the QTE estimates across the location shift $u$ for
the three models, and highlight the interplay between the EVI estimator
and the extrapolation scheme.  Only the location-invariant causal Fraga
estimator combined with the difference extrapolation yields genuinely
location-invariant QTE estimates, whose distribution remains stable as
$u$ grows.  For the causal Hill estimator with multiplicative
extrapolation, both the extreme value index estimator and the
intermediate quantile estimator shift with the location offset and in
opposite directions, so the net change of the resulting QTE estimate is
somewhat smaller than for its difference-extrapolation counterpart.
Finally, the QTE estimates based on the causal Fraga estimator are
generally more variable, reflecting the larger variance of the causal
Fraga extreme value index estimator itself.  Overall, the difference
extrapolation effectively reduces the variance inherited from the extreme
value index estimator, but it adds an extra variance component stemming
from estimating the intermediate quantile at the level $\beta_n$.

Figures~\ref{fig:qte_k_mean_H1}--\ref{fig:qte_k_mean_H3} report the
mean QTE estimates as a function of the threshold parameter $k$ for
the three extreme levels $\tau_n$ and the two sample sizes $n$.  The
QTE estimator that combines difference extrapolation with the causal
Fraga estimator behaves stably across all three models.  In contrast,
the two extrapolation variants based on the causal Hill estimator
fluctuate considerably under $H_1$ and $H_2$, and exhibit pronounced
spikes as $k$ increases, with estimates reaching extremely negative
values (truncated by the logarithmic scale), reflecting the fact that
the Hill extreme value index estimator itself can be highly unstable in
certain settings (see the discussion in
Fraga Alves~\cite{fraga2001location}).  Finally, under multiplicative
extrapolation, the high variance of the causal Fraga extreme value
index estimator, combined with small values of $k$, inflates the
variance and hence leads to a large mean bias.

Figures~\ref{fig:qte_ci_cov_H1}--\ref{fig:qte_ci_cov_H3} display the
empirical coverage of the confidence intervals for the three models.
The curve labelled ``Causal Fraga (asymp CI)'' corresponds to the
difference extrapolation estimator combined with the analytical rule of
Equation~\eqref{eq:confidence_interval}; all the other confidence
intervals are constructed with the bootstrap procedure.  The bootstrap
confidence intervals of the difference extrapolation estimator based on
the causal Fraga estimator attain coverage close to the nominal level,
supporting its asymptotic normality.  The analytical confidence
interval is closer to the nominal level at $n = 5000$ than at $n = 1000$,
indicating that the analytical variance estimator requires a larger
sample size to converge.  For the multiplicative extrapolation based on
the causal Fraga estimator, the bootstrap standard error remains
inflated at $n = 5000$ because of the large variance of the causal Fraga
extreme value index estimator itself, resulting in empirical coverages
above the nominal level.

Figure~\ref{fig:qte_ci_shift} compares the empirical coverage of the
bootstrap confidence intervals of the QTE estimator of Deuber et
al.~\cite{deuber2024estimation} and of the proposed estimator, for model
$H_2$ with $n = 5000$ and $\tau_n = 5/(n\log n)$, as the location shift
$u$ increases.  The coverage of the Deuber estimator drops rapidly as
$u$ grows, whereas that of the proposed estimator remains stable.
Additional results on the sensitivity of the QTE estimators to the
location shift $u$ and to the choice of $k$ are deferred to
Section~\ref{sec:additional_sim_results}.

\begin{figure}[t]
  \centering
  \includegraphics[width=\linewidth]{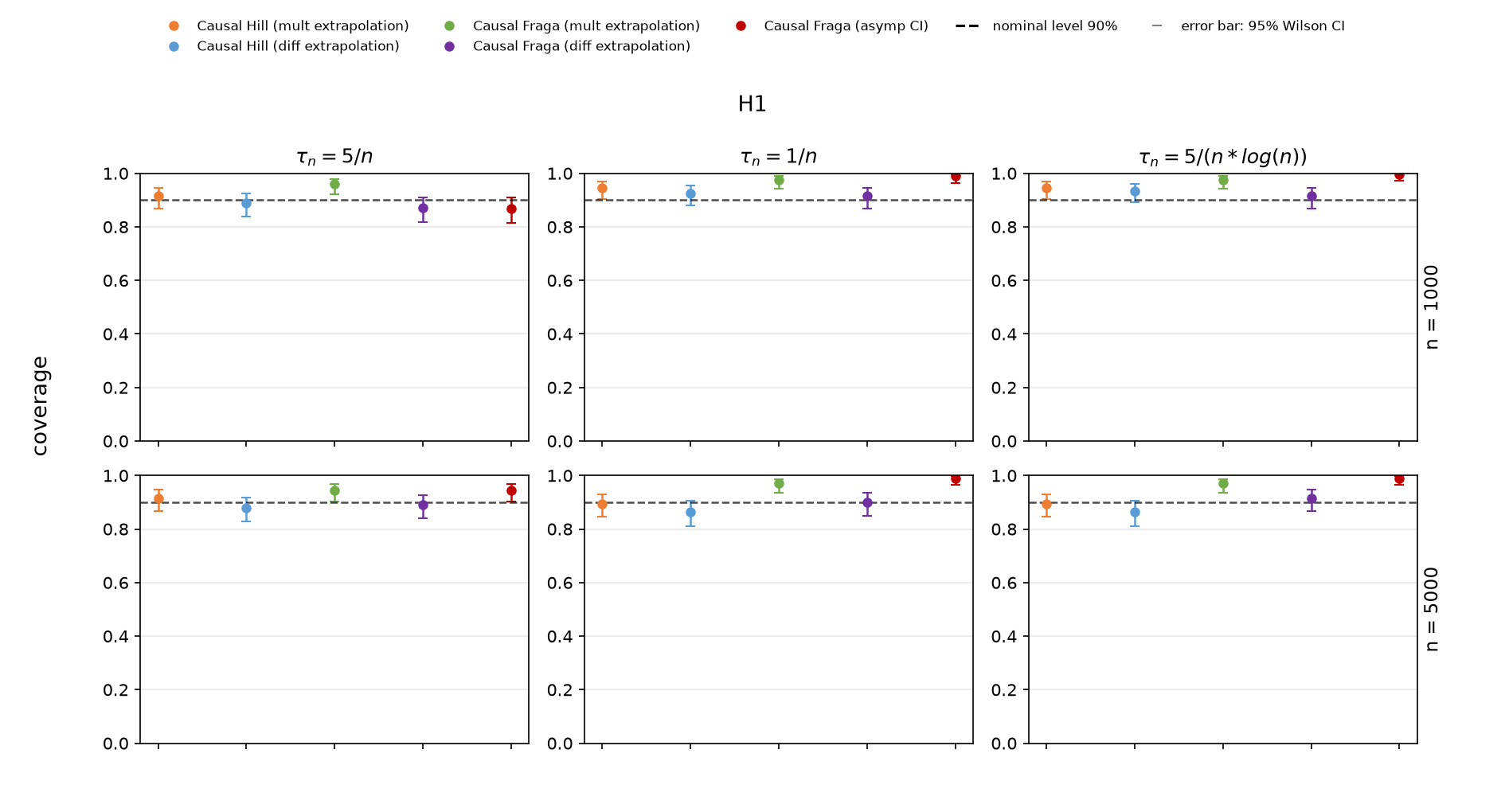}
  \caption{Empirical coverage of the confidence intervals for the
  quantile treatment effect for model $H_1$.}
  \label{fig:qte_ci_cov_H1}
\end{figure}

\begin{figure}[t]
  \centering
  \includegraphics[width=\linewidth]{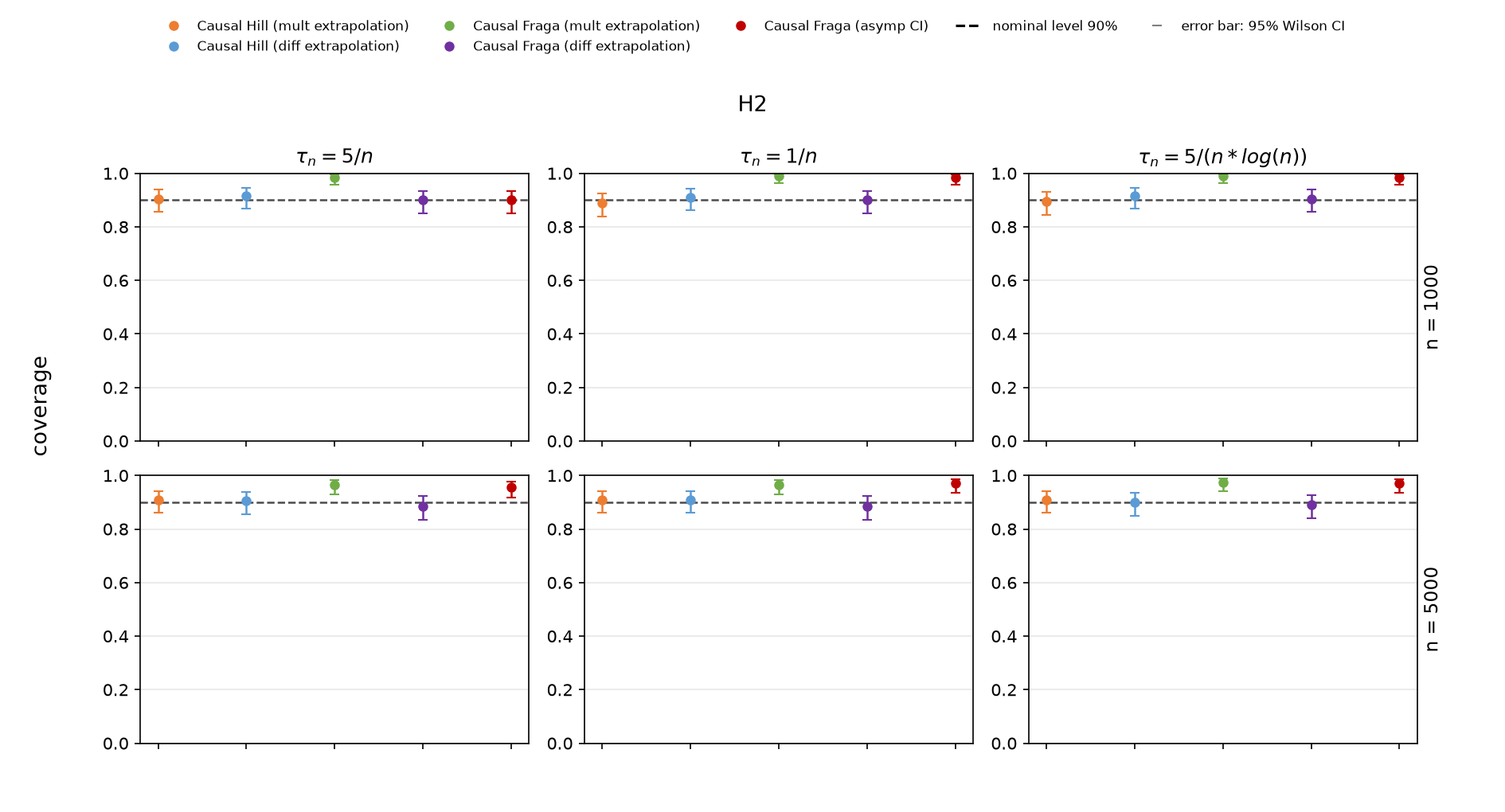}
  \caption{Empirical coverage of the confidence intervals for the
  quantile treatment effect for model $H_2$.}
  \label{fig:qte_ci_cov_H2}
\end{figure}

\begin{figure}[t]
  \centering
  \includegraphics[width=\linewidth]{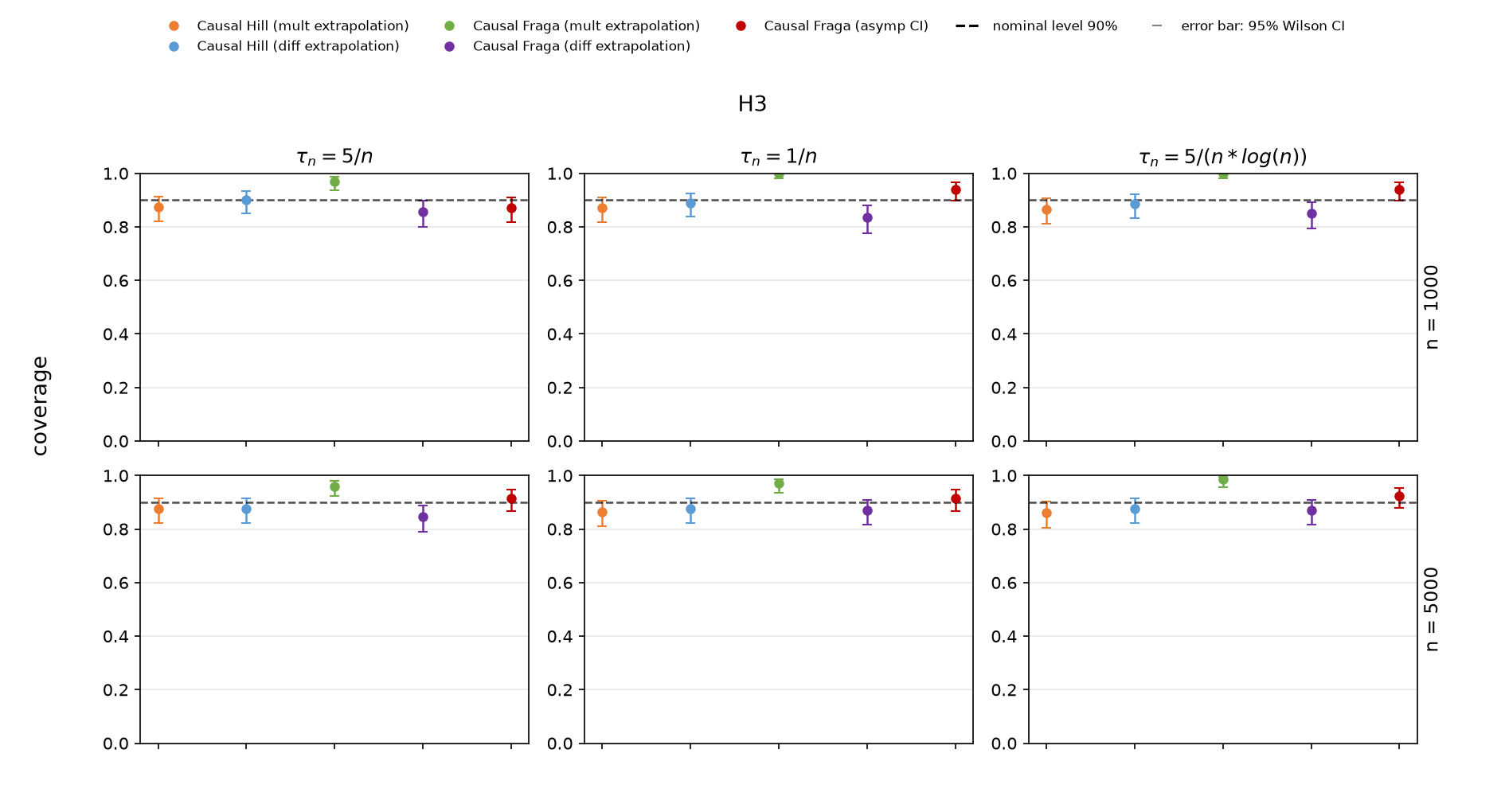}
  \caption{Empirical coverage of the confidence intervals for the
  quantile treatment effect for model $H_3$.}
  \label{fig:qte_ci_cov_H3}
\end{figure}

\begin{figure}[t]
  \centering
  \includegraphics[width=0.48\linewidth]{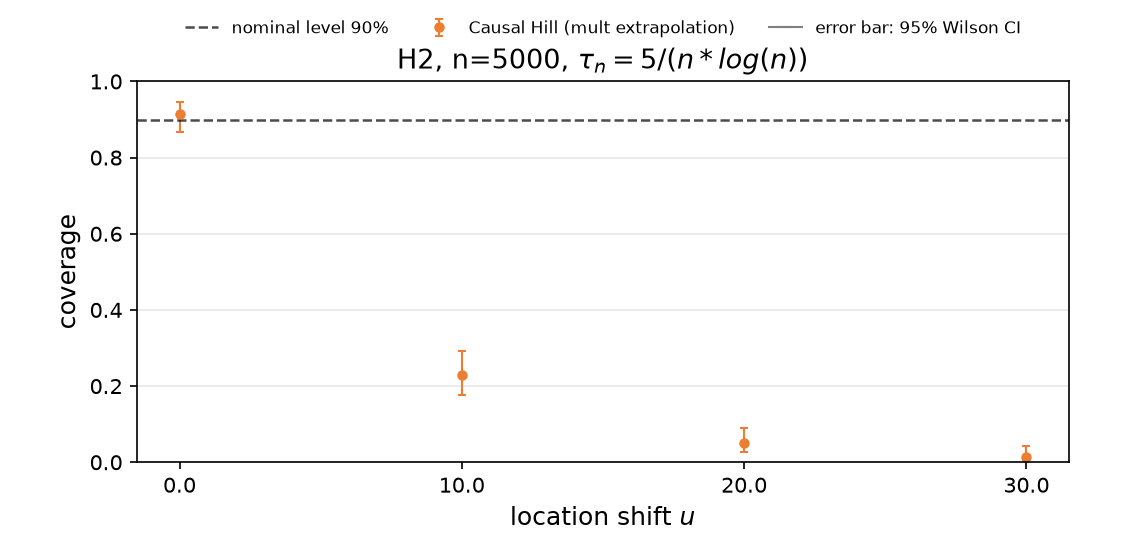}\hfill
  \includegraphics[width=0.48\linewidth]{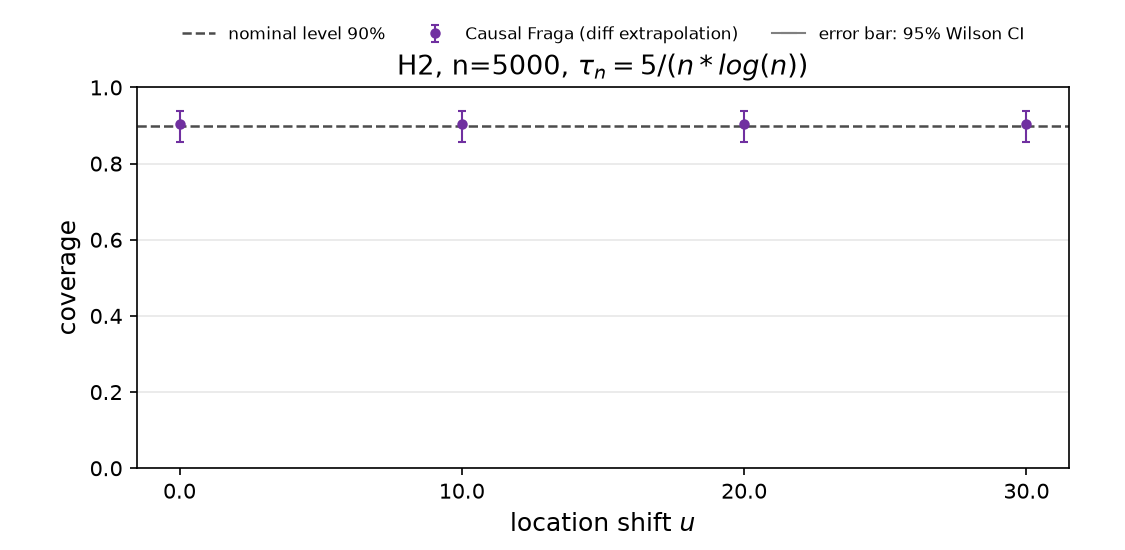}
  \caption{Empirical coverage of the bootstrap confidence intervals of
  the QTE estimator of Deuber et al.\ (left) and of the proposed QTE
  estimator (right), under the location shift $u$ for model $H_2$ with
  $n = 5000$ and $\tau_n = 5/(n\log n)$.}
  \label{fig:qte_ci_shift}
\end{figure}

\section{Conclusion}
\label{sec:conclusion}

We proposed a location-invariant estimator of the extremal quantile
treatment effect under unconfoundedness.  Building on the classical
location-invariant extreme value index estimator of Fraga
Alves~\cite{fraga2001location}, we constructed the causal Fraga
estimator by inverse propensity score weighting, and combined it with a
difference extrapolation scheme in which the location parameter cancels
in the differences of the intermediate quantiles.  We established the
consistency and asymptotic normality of the causal Fraga estimator, the
asymptotic normality of the resulting extremal QTE estimator, and the
consistency of an analytical variance estimator, which yields
asymptotically honest confidence intervals.  Through a simulation study
we verified the location invariance of the proposed estimator, its
stability as a function of the threshold $k$, and the coverage of the
associated confidence intervals.

The present work also suggests two directions for future research.
First, the Fraga-type estimator is known to have a relatively large
variance by construction.  Embedding causal counterparts of other
location-invariant extreme value index estimators with better
variance properties into our framework is a promising avenue for
improving the precision of the extremal QTE estimator.  Second, the
difference extrapolation scheme requires an additional auxiliary level
$\beta_n$, and the choice of $\beta_n$ currently follows an adaptive
rule without a formal optimality theory.  A thorough investigation of
how to choose $\beta_n$ is another potential direction, with the goal of
a data-driven rule balancing the bias and variance of the extrapolation.

\clearpage
\bibliographystyle{plain}
\bibliography{references}
\clearpage

\appendix

\numberwithin{equation}{section}
\numberwithin{theorem}{section}
\numberwithin{definition}{section}
\numberwithin{lemma}{section}
\numberwithin{proposition}{section}
\numberwithin{corollary}{section}
\numberwithin{assumption}{section}

\section{Proofs}
\label{sec:proofs}

\subsection{Asymptotic Theory}

This subsection lays down the asymptotic foundations used to prove the
asymptotic properties of the estimators proposed in the main text. Because
our analysis involves \emph{triangular arrays} of random variables
$\{Y_{n,k} : 1 \leq k \leq n,\ n \geq 1\}$, the classical law of large
numbers and central limit theorem for i.i.d.\ sequences no longer apply
directly. We therefore recall the weak law of large numbers and the
Lindeberg central limit theorem for triangular
arrays~\cite{durrett2019probability}, together with the
continuous mapping theorem (also known as the Mann--Wald
theorem~\cite{mann1943stochastic}) that
will be invoked repeatedly in the subsequent proofs.

We first collect the relationships and arithmetic rules among the order
symbols $o_p(1)$ and $O_p(1)$ that are used throughout the paper.
\begin{enumerate}
\def\theenumi{\roman{enumi}}
\def\labelenumi{(\theenumi)}
  \item If $Y_n = o_p(1)$ as $n \to +\infty$, then
        $Y_n = O_p(1)$ as $n \to +\infty$.
  \item $o_p(1) + o_p(1) = o_p(1)$, \quad
        $o_p(1) + O_p(1) = O_p(1)$, \quad
        $O_p(1) + O_p(1) = O_p(1)$.
  \item $o_p(1) \, o_p(1) = o_p(1)$, \quad
        $o_p(1) \, O_p(1) = o_p(1)$, \quad
        $O_p(1) \, O_p(1) = O_p(1)$.
\end{enumerate}

We next state the weak law of large numbers for triangular arrays. Note
that, unlike the Lindeberg central limit theorem presented next, the
weak law does not require the random variables within the same row to be
mutually independent.

\begin{theorem}[Weak law of large numbers for triangular arrays~\cite{durrett2019probability}]
\label{thm:wlln_triangular}
Let $S_n = \sum_{k=1}^{n} Y_{n,k}$, $\mu_n = \mathbb{E}(S_n)$ and
$\sigma_n^2 = \mathrm{Var}(S_n)$. If there exists a sequence $\{a_n\}$
such that $\sigma_n^2 / a_n^2 \to 0$ as $n \to +\infty$, then
\begin{equation*}
  \frac{S_n - \mu_n}{a_n} \xrightarrow{p} 0.
\end{equation*}
\end{theorem}

\begin{theorem}[Lindeberg central limit theorem for triangular arrays~\cite{durrett2019probability}]
\label{thm:lindeberg_clt}
For each $n \in \mathbb{Z}^+$, let $\{Y_{n,k} : 1 \leq k \leq n\}$ be
independent random variables with $\mathbb{E}(Y_{n,k}) = 0$. Define
$\sigma_{n,k}^2 = \mathbb{E}(Y_{n,k}^2)$,
$S_n = \sum_{k=1}^{n} Y_{n,k}$ and
$T_n^2 = \mathrm{Var}(S_n) = \sum_{k=1}^{n} \sigma_{n,k}^2$. If the
following Lindeberg condition holds: for every $\epsilon > 0$,
\begin{equation*}
  \frac{1}{T_n^2} \sum_{k=1}^{n}
    \mathbb{E}\!\left( Y_{n,k}^2 \,
      \mathbf{1}_{\{|Y_{n,k}| > \epsilon T_n\}} \right) \to 0,
  \quad n \to +\infty,
\end{equation*}
then
\begin{equation*}
  \frac{S_n}{T_n} \xrightarrow{d} \mathcal{N}(0, 1).
\end{equation*}
\end{theorem}

We conclude this section with the continuous mapping theorem
(Mann--Wald theorem), which allows the limit function to be discontinuous,
provided the random variables fall in the discontinuity set with
probability zero.

\begin{theorem}[Continuous mapping theorem~\cite{mann1943stochastic}]
\label{thm:cmt}
Let $Y_1, Y_2, \ldots$ be $k$-dimensional random vectors and let
$f : \mathbb{R}^k \to \mathbb{R}^l$ be continuous almost everywhere. Then
\begin{enumerate}
\def\theenumi{\roman{enumi}}
\def\labelenumi{(\theenumi)}
  \item If $Y_n \xrightarrow{p} Y$, then
        $f(Y_n) \xrightarrow{p} f(Y)$.
  \item If $Y_n \xrightarrow{d} Y$, then
        $f(Y_n) \xrightarrow{d} f(Y)$.
\end{enumerate}
\end{theorem}

We next state Slutsky's theorem, which is used throughout the paper
to combine convergence in distribution with convergence in
probability to a constant~\cite{vandervaart1998}.

\begin{theorem}[Slutsky's theorem~\cite{vandervaart1998}]
\label{thm:slutsky}
Let $X_n$, $Y_n$ be random variables and $c \in \mathbb{R}$ a
constant.  Then the following statements are equivalent:
\begin{enumerate}
\def\theenumi{\roman{enumi}}
\def\labelenumi{(\theenumi)}
  \item $X_n \xrightarrow{d} c$ if and only if
        $X_n \xrightarrow{p} c$;
  \item if $X_n \xrightarrow{d} X$ and the L\'evy--Prokhorov distance
        satisfies $d(X_n, Y_n) \xrightarrow{p} 0$, then
        $Y_n \xrightarrow{d} X$;
  \item if $X_n \xrightarrow{d} X$ and $Y_n \xrightarrow{p} c$, then
        $(X_n, Y_n) \xrightarrow{d} (X, c)$.
\end{enumerate}
\end{theorem}

\subsection{Proof of Theorem~\ref{thm:causal_fraga_consistency}}
\label{sec:consistency_proof}

\begin{lemma}
\label{lem:A.1}
Let Assumption~\ref{asm:causal_identification} hold, and let $f \ge 0$ be a measurable
function with $\mathbb{E}\bigl(f(Y(1))\bigr) < +\infty$ and
$\mathbb{E}\bigl(f(Y(0))\bigr) < +\infty$.  For any $r \in \mathbb{Z}^+$,
we have the uniform bounds
\begin{equation*}
  \frac{1}{(1 - c)^{\,r-1}}\;\mathbb{E}\bigl(f^{r}(Y(1))\bigr)
  \le \mathbb{E}\!\left( \left(f(Y)\,\frac{D}{\pi(X)}\right)^{\!r} \right)
  \le \frac{1}{c^{\,r-1}}\;\mathbb{E}\bigl(f^{r}(Y(1))\bigr),
\end{equation*}
\begin{equation*}
  \frac{1}{(1 - c)^{\,r-1}}\;\mathbb{E}\bigl(f^{r}(Y(0))\bigr)
  \le \mathbb{E}\!\left( \left(f(Y)\,\frac{1 - D}{1 - \pi(X)}\right)^{\!r} \right)
  \le \frac{1}{c^{\,r-1}}\;\mathbb{E}\bigl(f^{r}(Y(0))\bigr).
\end{equation*}
When $r = 1$, equality is attained in both cases:
\begin{equation*}
  \mathbb{E}\!\left(f(Y)\,\frac{D}{\pi(X)}\right)
    = \mathbb{E}\bigl(f(Y(1))\bigr),\qquad
  \mathbb{E}\!\left(f(Y)\,\frac{1 - D}{1 - \pi(X)}\right)
    = \mathbb{E}\bigl(f(Y(0))\bigr).
\end{equation*}
\end{lemma}

\begin{proof}[Proof of Lemma~\ref{lem:A.1}]
We prove the bound for the first inequality ($r \in \mathbb{Z}^+$) only;
the second follows by symmetry, and the $r = 1$ equalities are obtained
by noting that the inequality step below becomes an identity.
By the law of total expectation,
\begin{equation*}
  \mathbb{E}\!\left( \left(f(Y)\,\frac{D}{\pi(X)}\right)^{\!r} \right)
    = \mathbb{E}_X\!\left( \frac{1}{\pi(X)^r}\;
         \mathbb{E}\bigl( f(Y)^{r} D \mid X \bigr) \right).
\end{equation*}
Since $Y = D\,Y(1) + (1 - D)\,Y(0)$ and $D \in \{0, 1\}$,
$f(Y)^{r} D = f(Y(1))^{r} D$, hence
\begin{equation*}
  \mathbb{E}_X\!\left( \frac{1}{\pi(X)^r}\;
    \mathbb{E}\bigl( f(Y)^{r} D \mid X \bigr) \right)
    = \mathbb{E}_X\!\left( \frac{1}{\pi(X)^r}\;
         \mathbb{E}\bigl( f(Y(1))^{r} D \mid X \bigr) \right).
\end{equation*}
By Assumption~\ref{asm:causal_identification} (unconfoundedness),
$(Y(1), Y(0)) \perp D \mid X$, so the inner conditional expectation
factorises as
\begin{equation*}
  \mathbb{E}\bigl( f(Y(1))^{r} D \mid X \bigr)
    = \mathbb{E}\bigl( f(Y(1))^{r} \mid X \bigr)\;\mathbb{E}(D \mid X).
\end{equation*}
Using $\mathbb{E}(D \mid X) = \pi(X)$, the factor $\pi(X)$ cancels one
power in the denominator:
\begin{equation*}
  \mathbb{E}_X\!\left( \frac{1}{\pi(X)^r}\;
    \mathbb{E}\bigl( f(Y(1))^{r} \mid X \bigr)\,\pi(X) \right)
    = \mathbb{E}_X\!\left( \frac{1}{\pi(X)^{\,r-1}}\;
         \mathbb{E}\bigl( f(Y(1))^{r} \mid X \bigr) \right).
\end{equation*}
From Assumption~\ref{asm:causal_identification} (common support),
$c \le \pi(X) \le 1 - c$ almost surely,
hence $(1 - c)^{-(r-1)} \le 1 / \pi(X)^{\,r-1} \le c^{-(r-1)}$.
Applying these bounds to the expression after cancellation gives
the two-sided inequality
\begin{align*}
  \mathbb{E}_X\!\left( \frac{1}{\pi(X)^{\,r-1}}\;
    \mathbb{E}\bigl( f(Y(1))^{r} \mid X \bigr) \right)
  &\ge \frac{1}{(1 - c)^{\,r-1}}\;
     \mathbb{E}\bigl(f^{r}(Y(1))\bigr), \\[4pt]
  \mathbb{E}_X\!\left( \frac{1}{\pi(X)^{\,r-1}}\;
    \mathbb{E}\bigl( f(Y(1))^{r} \mid X \bigr) \right)
  &\le \frac{1}{c^{\,r-1}}\;
     \mathbb{E}\bigl(f^{r}(Y(1))\bigr),
\end{align*}
which is the claimed two-sided bound.
When $r = 1$, both bounds collapse to $1$ and the factor
$1 / \pi(X)^{\,r-1}$ equals $1$, so cancellation
with the $\pi(X)$ factor from $\mathbb{E}(D \mid X)$ leaves
\begin{align*}
  &\mathbb{E}_X\!\left( \frac{1}{\pi(X)}
       \mathbb{E}\bigl( f(Y(1))\,D \mid X \bigr) \right) \\
  =\, &\mathbb{E}_X\!\left( \frac{1}{\pi(X)}
       \mathbb{E}\bigl( f(Y(1)) \mid X \bigr)\,\pi(X) \right) \\
  =\, &\mathbb{E}_X\!\Bigl( \mathbb{E}\bigl( f(Y(1)) \mid X \bigr) \Bigr)
       = \mathbb{E}\bigl(f(Y(1))\bigr).
\end{align*} \qedhere
\end{proof}

\begin{lemma}
\label{lem:A.2}
Let Assumptions~\ref{asm:causal_identification}
and~\ref{asm:potential_outcome_distributions} hold.  For $j \in \{0, 1\}$, let
$Y_1(j), Y_2(j), \ldots, Y_n(j)$ be i.i.d.\ copies of the potential
outcome $Y(j)$, with cumulative distribution function $F_j$ and tail
quantile function $U_j(x) = \bigl(\frac{1}{1 - F_j}\bigr)^{\leftarrow}(x)$.
Define $Z_i(j) = \frac{1}{1 - F_j(Y_i(j))}$ for $i \in \{1, 2, \ldots, n\}$.
As $n \to +\infty$, assume $\beta_n \to 0$ with $k_0 = n \beta_n \to +\infty$.
Then
\begin{align}
  & \frac{1}{k_0} \sum_{i=1}^{n}
      \frac{D_i}{\pi(X_i)}\,
      \mathbf{1}_{\{Y_i > q_1(1 - \beta_n)\}}
      \xrightarrow{p} 1,
  \label{eq:lem_wlln_treated} \\
  & \frac{1}{k_0} \sum_{i=1}^{n}
      \frac{1 - D_i}{1 - \pi(X_i)}\,
      \mathbf{1}_{\{Y_i > q_0(1 - \beta_n)\}}
      \xrightarrow{p} 1,
  \label{eq:lem_wlln_control} \\
  & \frac{1}{k_0} \sum_{i=1}^{n}
      \frac{D_i}{\pi(X_i)}\,
      \mathbf{1}_{\{Y_i > q_1(1 - \beta_n)\}}
      \log\bigl(Z_i(1)\,\beta_n\bigr)
      \xrightarrow{p} 1,
  \label{eq:lem_wlln_treated_log} \\
  & \frac{1}{k_0} \sum_{i=1}^{n}
      \frac{1 - D_i}{1 - \pi(X_i)}\,
      \mathbf{1}_{\{Y_i > q_0(1 - \beta_n)\}}
      \log\bigl(Z_i(0)\,\beta_n\bigr)
      \xrightarrow{p} 1.
  \label{eq:lem_wlln_control_log}
\end{align}
\end{lemma}

\begin{proof}[Proof of Lemma~\ref{lem:A.2}]
We prove Equation~\eqref{eq:lem_wlln_treated}
and Equation~\eqref{eq:lem_wlln_treated_log}; the control-arm counterparts
\eqref{eq:lem_wlln_control} and~\eqref{eq:lem_wlln_control_log} follow by
replacing $D$ with $1 - D$ and $\pi(X)$ with $1 - \pi(X)$ throughout.

\paragraph{Proof of Equation~\eqref{eq:lem_wlln_treated}.}
Define
\[
  S_{n,1} = \sum_{i=1}^{n}
    \frac{D_i}{\pi(X_i)}\,\mathbf{1}_{\{Y_i > q_1(1-\beta_n)\}}.
\]
By the consistency relation $Y_i = D_i Y_i(1) + (1 - D_i) Y_i(0)$, the
summand vanishes unless $D_i = 1$, in which case $Y_i = Y_i(1)$
and
\[
  S_{n,1} = \sum_{i=1}^{n}
    \frac{D_i}{\pi(X_i)}\,\mathbf{1}_{\{Y_i(1) > q_1(1-\beta_n)\}}.
\]
Applying Lemma~\ref{lem:A.1} with $r = 1$ and $f = \mathbf{1}_{\{Y > q_1(1-\beta_n)\}} \ge 0$ gives
\[
  \mathbb{E}\!\left[\frac{D_i}{\pi(X_i)}\,\mathbf{1}_{\{Y_i(1) > q_1(1-\beta_n)\}}\right]
    = \mathbb{P}\!\left(Y(1) > q_1(1-\beta_n)\right)
    = \beta_n.
\]
Applying the lemma again with $r = 2$ yields
\[
  \mathbb{E}\!\left[\left(\frac{D_i}{\pi(X_i)}\,\mathbf{1}_{\{Y_i(1) > q_1(1-\beta_n)\}}\right)^{\!2}\right]
    \le \frac{1}{c}\;\mathbb{P}\!\left(Y(1) > q_1(1-\beta_n)\right)
    = \frac{\beta_n}{c}.
\]
Hence $\mathbb{E}(S_{n,1}) = n \beta_n = k_0$. Moreover,
\begin{align*}
  \frac{\mathrm{Var}(S_{n,1})}{k_0^{2}}
    &\le \frac{n}{k_0^{2}}\;
       \mathbb{E}\!\left[\left(\frac{D_i}{\pi(X_i)}\,\mathbf{1}_{\{Y_i(1) > q_1(1-\beta_n)\}}\right)^{\!2}\right] \\
    &\le \frac{n}{k_0^{2}}\cdot\frac{\beta_n}{c}
     = \frac{1}{n \beta_n\,c} \;\to\; 0,
\end{align*}
where the first inequality uses the standard bound
$\mathrm{Var}(S_{n,1}) \le n \mathbb{E}[\text{summand}^2]$ for independent
summands. By the weak law of large numbers for triangular arrays
(Theorem~\ref{thm:wlln_triangular}) applied with $a_n = k_0$,
\begin{equation*}
  \frac{S_{n,1} - \mathbb{E}(S_{n,1})}{k_0}
    = \frac{S_{n,1} - k_0}{k_0} \xrightarrow{p} 0.
\end{equation*}
Substituting $\mathbb{E}(S_{n,1}) = k_0$, we obtain
\[
  \frac{1}{k_0}\sum_{i=1}^{n}
    \frac{D_i}{\pi(X_i)}\,\mathbf{1}_{\{Y_i > q_1(1-\beta_n)\}}
    = \frac{S_{n,1}}{k_0} \xrightarrow{p} 1,
\]
which is Equation~\eqref{eq:lem_wlln_treated}.

\paragraph{Proof of Equation~\eqref{eq:lem_wlln_treated_log}.}
Define
\[
  S_{n,2} = \sum_{i=1}^{n}
    \frac{D_i}{\pi(X_i)}\,\mathbf{1}_{\{Y_i > q_1(1-\beta_n)\}}
      \log\bigl(Z_i(1)\,\beta_n\bigr),
\]
where $Z_i(1) = 1/(1 - F_1(Y_i(1)))$. The same substitution
$Y_i = D_i Y_i(1)$ gives
\[
  S_{n,2} = \sum_{i=1}^{n}
    \frac{D_i}{\pi(X_i)}\,\mathbf{1}_{\{Y_i(1) > q_1(1-\beta_n)\}}
      \log\bigl(Z_i(1)\,\beta_n\bigr).
\]
Because $F_1(Y_i(1)) \sim U(0,1)$ under
Assumption~\ref{asm:potential_outcome_distributions}, the random variable
$Z_i(1) = 1/(1 - F_1(Y_i(1)))$ has density
\begin{equation*}
  \psi_1(z) = z^{-2}, \quad z \ge 1,
\end{equation*}
so that $\log Z_i(1)$ follows the standard exponential distribution
with density
\begin{equation*}
  \psi_2(z) = e^{-z}, \quad z \ge 0.
\end{equation*}
By Lemma~\ref{lem:A.1} with $r = 1$ and the function
$f(y) = \mathbf{1}_{\{y > q_1(1-\beta_n)\}}\log\!\bigl(\beta_n / (1 - F_1(y))\bigr)$,
which is non-negative on its support, the preceding density gives
\begin{align*}
  \mathbb{E}\!\left[\frac{D_i}{\pi(X_i)}\,
      \mathbf{1}_{\{Z_i(1) > \beta_n^{-1}\}}\log(Z_i(1)\,\beta_n)\right]
    &= \mathbb{E}\!\left[\mathbf{1}_{\{Z_i(1) > \beta_n^{-1}\}}\log(Z_i(1)\,\beta_n)\right] \\
    &= \int_{\log(\beta_n^{-1})}^{+\infty}
        (z + \log \beta_n)\, e^{-z}\, dz \\
    &= \beta_n,
\end{align*}
the last step following because the integral equals
$e^{-\log(\beta_n^{-1})} = \beta_n$. Hence $\mathbb{E}(S_{n,2}) =
n \beta_n = k_0$. Applying Lemma~\ref{lem:A.1} with $r = 2$
yields
\[
  \mathbb{E}\!\left[\left(\frac{D_i}{\pi(X_i)}\,
        \mathbf{1}_{\{Z_i(1) > \beta_n^{-1}\}}\log(Z_i(1)\,\beta_n)\right)^{\!2}\right]
    \le \frac{1}{c}\;\mathbb{E}\!\left[\mathbf{1}_{\{Z_i(1) > \beta_n^{-1}\}}
        \log^{2}(Z_i(1)\,\beta_n)\right]
    = \frac{2\beta_n}{c},
\]
where the exponential moment is obtained by
$\int_{\log(\beta_n^{-1})}^{\infty} (z + \log \beta_n)^2 e^{-z}\, dz
= 2 \beta_n$. Therefore
\[
  \frac{\mathrm{Var}(S_{n,2})}{k_0^{2}}
    \le \frac{n}{k_0^{2}}\cdot\frac{2\beta_n}{c}
    = \frac{2}{n\beta_n\,c} \;\to\; 0.
\]
Theorem~\ref{thm:wlln_triangular} now yields, with $a_n = k_0$,
\begin{equation*}
  \frac{S_{n,2} - \mathbb{E}(S_{n,2})}{k_0}
    = \frac{S_{n,2} - k_0}{k_0} \xrightarrow{p} 0.
\end{equation*}
Substituting $\mathbb{E}(S_{n,2}) = k_0$, we obtain
\[
  \frac{1}{k_0}\sum_{i=1}^{n}
    \frac{D_i}{\pi(X_i)}\,\mathbf{1}_{\{Y_i > q_1(1-\beta_n)\}}
      \log\bigl(Z_i(1)\,\beta_n\bigr)
    = \frac{S_{n,2}}{k_0} \xrightarrow{p} 1,
\]
i.e., Equation~\eqref{eq:lem_wlln_treated_log}.
\end{proof}

\begin{lemma}
\label{lem:A.3}
Let Assumptions~\ref{asm:causal_identification},~\ref{asm:sieve_assumptions}
and~\ref{asm:potential_outcome_distributions} hold.  Denote by
$\widehat{q}_j(1 - \beta_n)$ the $\tau = 1 - \beta_n$ quantile
estimator of the potential outcome $Y(j)$ constructed
in Equation~\eqref{eq:qte_firpo}.  As $n \to +\infty$, assume $\beta_n \to 0$
with $k_0 = n \beta_n \to +\infty$.  Then
\begin{align}
  & \frac{1}{k_0} \sum_{i=1}^{n}
      \frac{D_i}{\widehat{\pi}(X_i)}\,
      \Bigl(\mathbf{1}_{\{Y_i > \widehat{q}_1(1-\beta_n)\}}
           - \mathbf{1}_{\{Y_i > q_1(1-\beta_n)\}}\Bigr)
    \xrightarrow{p} 0,
  \label{eq:lem_quantile_treated} \\[4pt]
  & \frac{1}{k_0} \sum_{i=1}^{n}
      \frac{1 - D_i}{1 - \widehat{\pi}(X_i)}\,
      \Bigl(\mathbf{1}_{\{Y_i > \widehat{q}_0(1-\beta_n)\}}
           - \mathbf{1}_{\{Y_i > q_0(1-\beta_n)\}}\Bigr)
    \xrightarrow{p} 0.
  \label{eq:lem_quantile_control}
\end{align}
\end{lemma}

\begin{proof}[Proof of Lemma~\ref{lem:A.3}]
We prove Equation~\eqref{eq:lem_quantile_treated}; the control-arm statement
\eqref{eq:lem_quantile_control} follows by the symmetric substitution of
$D$ with $1 - D$ and $\widehat{\pi}$ with $1 - \widehat{\pi}$ throughout.

Setting $j = 1$ and $\tau = 1 - \beta_n$ in Equation~\eqref{eq:qte_firpo} gives
\[
  \widehat{q}_1(1-\beta_n) = \mathop{\mathrm{argmin}}_{q \in \mathbb{R}}
    \sum_{i=1}^{n}
      \frac{D_i}{\widehat{\pi}(X_i)}\,(Y_i - q)\,
        \bigl((1 - \beta_n) - \mathbf{1}_{\{Y_i \leq q\}}\bigr).
\]
Denote the objective by $L_n(q)$. Since $Y_i$ is continuous under
Assumption~\ref{asm:potential_outcome_distributions}, $L_n(q)$ is piecewise linear
and differentiable at $q = \widehat{q}_1(1-\beta_n)$ almost surely. The
first-order condition gives
\[
  L'_n\bigl(\widehat{q}_1(1-\beta_n)\bigr)
    = -(1 - \beta_n) \sum_{i=1}^{n} \frac{D_i}{\widehat{\pi}(X_i)}
      + \sum_{i=1}^{n}
        \frac{D_i}{\widehat{\pi}(X_i)}\,\mathbf{1}_{\{Y_i \le \widehat{q}_1(1-\beta_n)\}}
    = 0.
\]
Combining with $\mathbf{1}_{\{Y_i \le q\}} + \mathbf{1}_{\{Y_i > q\}} = 1$ and
$k_0 = n \beta_n$ yields the identity
\begin{equation*}
  \sum_{i=1}^{n}
    \frac{D_i}{\widehat{\pi}(X_i)}\,\mathbf{1}_{\{Y_i > \widehat{q}_1(1-\beta_n)\}}
    = \beta_n \sum_{i=1}^{n} \frac{D_i}{\widehat{\pi}(X_i)}.
\end{equation*}
Crucially, this identity is purely algebraic and requires no asymptotic assumptions.

\begin{align*}
  &\biggl| \frac{1}{k_0} \sum_{i=1}^{n}
       \frac{D_i}{\widehat{\pi}(X_i)}\,
       \bigl(\mathbf{1}_{\{Y_i > \widehat{q}_1(1-\beta_n)\}}
           - \mathbf{1}_{\{Y_i > q_1(1-\beta_n)\}}\bigr) \biggr| \\
  &\qquad\le
     \underbrace{\biggl|
        \frac{1}{k_0} \sum_{i=1}^{n}
        \frac{D_i}{\widehat{\pi}(X_i)}\,
        \bigl(\mathbf{1}_{\{Y_i > \widehat{q}_1(1-\beta_n)\}} - \beta_n\bigr)
      \biggr|}_{T_{n,1}}
     \;+\;
     \underbrace{\biggl|
        \frac{1}{k_0} \sum_{i=1}^{n}
        \frac{D_i}{\widehat{\pi}(X_i)}\,
        \bigl(\mathbf{1}_{\{Y_i > q_1(1-\beta_n)\}} - \beta_n\bigr)
      \biggr|}_{T_{n,2}}.
\end{align*}

By the identity above,
\begin{align*}
  T_{n,1}
  &= \biggl|
      \frac{1}{k_0}\, \beta_n \sum_{i=1}^{n}
        \frac{D_i}{\widehat{\pi}(X_i)}
      - \frac{\beta_n}{k_0} \sum_{i=1}^{n}
        \frac{D_i}{\widehat{\pi}(X_i)}
    \biggr|
  = 0 \qquad \text{(a.s.)}.
\end{align*}

Moreover, using $k_0 = n \beta_n$, write $T_{n,2} = |T_{n,2}^{(1)} - T_{n,2}^{(2)}|$ with
\[
  T_{n,2}^{(1)} := \frac{1}{k_0} \sum_{i=1}^{n}
    \frac{D_i}{\widehat{\pi}(X_i)}\,
    \mathbf{1}_{\{Y_i > q_1(1-\beta_n)\}},
\qquad
  T_{n,2}^{(2)} := \frac{\beta_n}{k_0} \sum_{i=1}^{n}
    \frac{D_i}{\widehat{\pi}(X_i)}
    = \frac{1}{n} \sum_{i=1}^{n} \frac{D_i}{\widehat{\pi}(X_i)}.
\]
For $T_{n,2}^{(1)}$, Lemma~\ref{lem:A.2}
(Equation \eqref{eq:lem_wlln_treated}) gives
\[
  \frac{1}{k_0} \sum_{i=1}^{n}
    \frac{D_i}{\pi(X_i)}\,
    \mathbf{1}_{\{Y_i > q_1(1-\beta_n)\}} \xrightarrow{p} 1.
\]
By Deuber et al.~\cite{deuber2024estimation} (Lemma~G.3), which uses
Assumption~\ref{asm:sieve_assumptions} together with the continuity of
$g(u) = 1/u$ on $[c, 1]$ (recall $\pi(\cdot) \ge c > 0$ by
Assumption~\ref{asm:causal_identification}),
\begin{equation}\label{eq:pi_inverse_uniform}
  \sup_{x \in \mathrm{supp}(X)} \Bigl|\frac{1}{\widehat{\pi}(x)} - \frac{1}{\pi(x)}\Bigr| = o_p(1).
\end{equation}

For $T_{n,2}^{(1)}$, Lemma~\ref{lem:A.2}
(Equation~\eqref{eq:lem_wlln_treated}) gives
$\frac{1}{k_0}\sum_i \frac{D_i}{\pi(X_i)}\mathbf{1}_{\{Y_i > q_1(1-\beta_n)\}} \xrightarrow{p} 1$.
Expanding $\frac{D_i}{\widehat{\pi}(X_i)}
= \frac{D_i}{\pi(X_i)} + D_i\bigl(\frac{1}{\widehat{\pi}(X_i)} - \frac{1}{\pi(X_i)}\bigr)$
and applying~\eqref{eq:pi_inverse_uniform}, the difference satisfies
\begin{equation*}
  \biggl| T_{n,2}^{(1)} - \frac{1}{k_0}\sum_{i=1}^{n}
       \frac{D_i}{\pi(X_i)}\mathbf{1}_{\{Y_i > q_1(1-\beta_n)\}} \biggr|
  \le \sup_{x \in \mathrm{supp}(X)} \Bigl|\frac{1}{\widehat{\pi}(x)} - \frac{1}{\pi(x)}\Bigr|
      \cdot \frac{1}{k_0}\sum_{i=1}^{n}
        D_i\,\mathbf{1}_{\{Y_i > q_1(1-\beta_n)\}}.
\end{equation*}
The sup factor is $o_p(1)$ by~\eqref{eq:pi_inverse_uniform}.  For the
second factor, since $Y_i = Y_i(1)$ whenever $D_i = 1$, and by
Assumption~\ref{asm:causal_identification},
\begin{equation*}
  \mathbb{E}\bigl[D_i\,\mathbf{1}_{\{Y_i > q_1(1-\beta_n)\}}\bigr]
  = \mathbb{E}\bigl[\pi(X)\,\mathbf{1}_{\{Y(1) > q_1(1-\beta_n)\}}\bigr]
  \le (1 - c)\,\beta_n,
\end{equation*}
its mean is at most $(1-c)/n \to 0$, hence it is $o_p(1)$.  The product
of two $o_p(1)$ factors is $o_p(1)$, whence
$T_{n,2}^{(1)} \xrightarrow{p} 1$.

For $T_{n,2}^{(2)}$, by Lemma~\ref{lem:A.1} with $r = 1$ and
$f \equiv 1$ we have $\mathbb{E}\!\bigl[D/\pi(X)\bigr] = 1$, so the
Khinchine law of large numbers gives
$\frac{1}{n}\sum_i D_i/\pi(X_i) \xrightarrow{p} 1$.  The same expansion
yields
\begin{equation*}
  \biggl| \frac{1}{n}\sum_{i=1}^{n}\frac{D_i}{\widehat{\pi}(X_i)}
          - \frac{1}{n}\sum_{i=1}^{n}\frac{D_i}{\pi(X_i)} \biggr|
  \le \sup_{x \in \mathrm{supp}(X)} \Bigl|\frac{1}{\widehat{\pi}(x)} - \frac{1}{\pi(x)}\Bigr|
      \cdot \frac{1}{n}\sum_{i=1}^{n} D_i
  \le \sup_{x \in \mathrm{supp}(X)} \Bigl|\frac{1}{\widehat{\pi}(x)} - \frac{1}{\pi(x)}\Bigr| = o_p(1),
\end{equation*}
since $\frac{1}{n}\sum_i D_i \le 1$.  Hence
$T_{n,2}^{(2)} \xrightarrow{p} 1$.

Therefore $T_{n,2} = T_{n,2}^{(1)} - T_{n,2}^{(2)} \xrightarrow{p} 0$.
Putting the two estimates together gives
\[
  \frac{1}{k_0} \sum_{i=1}^{n}
    \frac{D_i}{\widehat{\pi}(X_i)}\,
    \Bigl(\mathbf{1}_{\{Y_i > \widehat{q}_1(1-\beta_n)\}}
        - \mathbf{1}_{\{Y_i > q_1(1-\beta_n)\}}\Bigr)
  \xrightarrow{p} 0,
\]
which is Equation~\eqref{eq:lem_quantile_treated}.
\end{proof}

Now we give the proof of Theorem~\ref{thm:causal_fraga_consistency}.
\begin{proof}
We prove the consistency for $j = 1$; the case $j = 0$ follows by the
symmetric substitution of $D$ with $1 - D$ and $\pi$ with $1 - \pi$
throughout.  Write the estimator as
\begin{equation*}
  \widehat{\gamma}_{1}^{F}(\beta_n, \alpha_n)
    = \frac{1}{k_0} \sum_{i=1}^{n}
      \frac{D_i}{\widehat{\pi}(X_i)}
      \mathbf{1}_{\{Y_i > \widehat{q}_1(1 - \beta_n)\}}
      \log \frac{Y_i - \widehat{q}_1(1 - \alpha_n)}
                {\widehat{q}_1(1 - \beta_n) - \widehat{q}_1(1 - \alpha_n)},
\end{equation*}
which decomposes additively as
$\widehat{\gamma}_{1}^{F} = A_{n,1}^{1} + A_{n,1}^{2} + A_{n,1}^{3} + A_{n,1}^{4} + A_{n,1}^{5}$,
where
\begin{align*}
  A_{n,1}^{1} &=
    \frac{1}{k_0} \sum_{i=1}^{n}
      \frac{D_i}{\pi(X_i)}\,
      \mathbf{1}_{\{Y_i > q_1(1 - \beta_n)\}}
      \log \frac{Y_i - q_1(1 - \alpha_n)}
                {q_1(1 - \beta_n) - q_1(1 - \alpha_n)}, \\[4pt]
  A_{n,1}^{2} &=
    \frac{1}{k_0} \sum_{i=1}^{n}
      \frac{D_i}{\pi(X_i)}\,
      \mathbf{1}_{\{Y_i > q_1(1 - \beta_n)\}}
      \log \frac{q_1(1 - \beta_n) - q_1(1 - \alpha_n)}
                {\widehat{q}_1(1 - \beta_n) - \widehat{q}_1(1 - \alpha_n)}, \\[4pt]
  A_{n,1}^{3} &=
    \frac{1}{k_0} \sum_{i=1}^{n}
      \frac{D_i}{\pi(X_i)}\,
      \mathbf{1}_{\{Y_i > q_1(1 - \beta_n)\}}
      \log \frac{Y_i - \widehat{q}_1(1 - \alpha_n)}
                {Y_i - q_1(1 - \alpha_n)}, \\[4pt]
  A_{n,1}^{4} &=
    \frac{1}{k_0} \sum_{i=1}^{n}
      \frac{D_i}{\widehat{\pi}(X_i)}\,
      \Bigl(
        \mathbf{1}_{\{Y_i > \widehat{q}_1(1 - \beta_n)\}}
        - \mathbf{1}_{\{Y_i > q_1(1 - \beta_n)\}}
      \Bigr)
      \log \frac{Y_i - \widehat{q}_1(1 - \alpha_n)}
                {\widehat{q}_1(1 - \beta_n) - \widehat{q}_1(1 - \alpha_n)}, \\[4pt]
  A_{n,1}^{5} &=
    \frac{1}{k_0} \sum_{i=1}^{n}
      D_i \!\left(
        \frac{1}{\widehat{\pi}(X_i)} - \frac{1}{\pi(X_i)}
      \right)\!
      \mathbf{1}_{\{Y_i > q_1(1 - \beta_n)\}}
      \log \frac{Y_i - \widehat{q}_1(1 - \alpha_n)}
                {\widehat{q}_1(1 - \beta_n) - \widehat{q}_1(1 - \alpha_n)}.
\end{align*}

\paragraph{Step 1: $A_{n,1}^{1} \xrightarrow{p} \gamma_1$.}
By the consistency relation $Y_i = D_i Y_i(1) + (1 - D_i) Y_i(0)$,
the summand of $A_{n,1}^{1}$ vanishes unless $D_i = 1$, in which case
$Y_i = Y_i(1)$.  Define
$Z_i(1) = 1 / (1 - F_1(Y_i(1)))$.
For $j = 1$, Assumption~\ref{asm:potential_outcome_distributions} gives
$F_1 \in \mathcal{D}(G_{\gamma_1})$, so that
$Y_i(1) = U_1(Z_i(1))$ and
$q_1(1 - \beta_n) = U_1(\beta_n^{-1})$,
$q_1(1 - \alpha_n) = U_1(\alpha_n^{-1})$.
We have 
\begin{equation*}
  A_{n,1}^{1} = \frac{1}{k_0} \sum_{i=1}^{n}
    \frac{D_i}{\pi(X_i)}\,
    \mathbf{1}_{\{Y_i > q_1(1 - \beta_n)\}}\,
    \log\!\left(
      \frac{U_1(Z_i(1)) - U_1(\alpha_n^{-1})}
           {U_1(\beta_n^{-1}) - U_1(\alpha_n^{-1})}
    \right).
\end{equation*}

Following the statement of Proposition B.1.9 in de Haan and Ferreira~\cite{de2006extreme}, by the heavy-tail representation in
Equation~\eqref{eq:equiv_second_order} and Assumption~\ref{asm:potential_outcome_distributions}, for any
$\epsilon_0 \in (0, 1)$ and $\epsilon_1 \in (0, \gamma_1)$, there exists
$t_0 > 0$ such that for all $x > y > 1$ and $t > t_0$,
\begin{equation*}
  (1 - \epsilon_0)\,
  \frac{x^{\gamma_1 - \epsilon_1} - 1}{y^{\gamma_1 - \epsilon_1} - 1}
    \;\le\;
  \frac{U_1(t x) - U_1(t)}{U_1(t y) - U_1(t)}
    \;\le\;
  (1 + \epsilon_0)\,
  \frac{x^{\gamma_1 + \epsilon_1} - 1}{y^{\gamma_1 + \epsilon_1} - 1}.
\end{equation*}

Therefore,
\begin{align*}
  &\log(1 - \epsilon_0) + \log\!\left(\frac{x^{\gamma_1 - \epsilon_1} - 1}{y^{\gamma_1 - \epsilon_1} - 1}\right) \\
  &\quad = \log(1 - \epsilon_0) + \log\!\left(\frac{x^{\gamma_1 - \epsilon_1} - 1}{x^{\gamma_1 - \epsilon_1}}\right) - \log\!\left(\frac{y^{\gamma_1 - \epsilon_1} - 1}{y^{\gamma_1 - \epsilon_1}}\right) + (\gamma_1 - \epsilon_1) \log\!\left(\frac{x}{y}\right) \\
  &\quad < \log\!\left(\frac{U_1(tx) - U_1(t)}{U_1(ty) - U_1(t)}\right) \\
  &\quad < \log(1 + \epsilon_0) + \log\!\left(\frac{x^{\gamma_1 + \epsilon_1} - 1}{y^{\gamma_1 + \epsilon_1} - 1}\right) \\
  &\quad = \log(1 + \epsilon_0) + \log\!\left(\frac{x^{\gamma_1 + \epsilon_1} - 1}{x^{\gamma_1 + \epsilon_1}}\right) - \log\!\left(\frac{y^{\gamma_1 + \epsilon_1} - 1}{y^{\gamma_1 + \epsilon_1}}\right) + (\gamma_1 + \epsilon_1) \log\!\left(\frac{x}{y}\right).
\end{align*}

Take $t = \alpha_n^{-1}$, $x = \alpha_n Z_i(1)$,
$y = \alpha_n \beta_n^{-1}$.  Since
$\beta_n / \alpha_n \to 0$, we have $\alpha_n \beta_n^{-1} \to +\infty$,
and for $Z_i(1) > \beta_n^{-1}$ we have
$x > y > 1$ eventually.
Substituting in yields
\begin{align*}
  & \epsilon^- + (\gamma_1 - \epsilon_1)\log\bigl(Z_i(1) \beta_n\bigr) \\
  & \qquad \le
  \log \frac{U_1(Z_i(1)) - U_1(\alpha_n^{-1})}
           {U_1(\beta_n^{-1}) - U_1(\alpha_n^{-1})}
  \;\le\;
  \epsilon^+ + (\gamma_1 + \epsilon_1)\log\bigl(Z_i(1) \beta_n\bigr),
\end{align*}
where
$\epsilon^- = \log(1 - \epsilon_0) + o(1)$ and
$\epsilon^+ = \log(1 + \epsilon_0) + o(1)$,
with the $o(1)$ terms vanishing uniformly in $i$ as
$\alpha_n Z_i(1),\, \alpha_n \beta_n^{-1} \to +\infty$
(here $\alpha_n Z_i(1) \ge \alpha_n \beta_n^{-1} \to +\infty$ on the event
$\{Z_i(1) > \beta_n^{-1}\}$, so the remainder is bounded by the
$y$-dependent remainder, which vanishes uniformly).
Then we can obtain the sandwich
\begin{align*}
  & \epsilon^- \frac{1}{k_0} \sum_{i=1}^{n} \frac{D_i}{\pi(X_i)}\,
        \mathbf{1}_{\{Y_i > q_1(1 - \beta_n)\}}
      + (\gamma_1 - \epsilon_1)
        \frac{1}{k_0} \sum_{i=1}^{n} \frac{D_i}{\pi(X_i)}\,
        \mathbf{1}_{\{Y_i > q_1(1 - \beta_n)\}}
        \log\bigl(Z_i(1) \beta_n\bigr) \\
  & \qquad \le \frac{1}{k_0} \sum_{i=1}^{n} \frac{D_i}{\pi(X_i)}\,
        \mathbf{1}_{\{Y_i > q_1(1 - \beta_n)\}}
        \log \frac{U_1(Z_i(1)) - U_1(\alpha_n^{-1})}
                {U_1(\beta_n^{-1}) - U_1(\alpha_n^{-1})} \\
  & \qquad \le \epsilon^+ \frac{1}{k_0} \sum_{i=1}^{n} \frac{D_i}{\pi(X_i)}\,
        \mathbf{1}_{\{Y_i > q_1(1 - \beta_n)\}}
      + (\gamma_1 + \epsilon_1)
        \frac{1}{k_0} \sum_{i=1}^{n} \frac{D_i}{\pi(X_i)}\,
        \mathbf{1}_{\{Y_i > q_1(1 - \beta_n)\}}
        \log\bigl(Z_i(1) \beta_n\bigr).
\end{align*}
Applying Lemma~\ref{lem:A.2} (Equations~\eqref{eq:lem_wlln_treated}
and~\eqref{eq:lem_wlln_treated_log}) gives
\begin{align*}
  \frac{1}{k_0} \sum_{i=1}^{n} \frac{D_i}{\pi(X_i)}\,
    \mathbf{1}_{\{Y_i > q_1(1 - \beta_n)\}} &\xrightarrow{p} 1, \\
  \frac{1}{k_0} \sum_{i=1}^{n} \frac{D_i}{\pi(X_i)}\,
    \mathbf{1}_{\{Y_i > q_1(1 - \beta_n)\}}
    \log\bigl(Z_i(1) \beta_n\bigr) &\xrightarrow{p} 1.
\end{align*}
Hence the lower bound on $A_{n,1}^{1}$ converges in probability to
$\epsilon^- + (\gamma_1 - \epsilon_1)$,
and the upper bound to
$\epsilon^+ + (\gamma_1 + \epsilon_1)$.
Since $\epsilon^-$, $\epsilon^+$ and $\epsilon_1$ can be made arbitrarily small,
so $A_{n,1}^{1} \xrightarrow{p} \gamma_1$.

\paragraph{Step 2: $A_{n,1}^{2} \xrightarrow{p} 0$.}
Factor $A_{n,1}^{2}$ as
$A_{n,1}^{2} = \delta_n \cdot
  \frac{1}{k_0} \sum_{i=1}^{n}
    \frac{D_i}{\pi(X_i)}\,
    \mathbf{1}_{\{Y_i > q_1(1 - \beta_n)\}}$,
where
\begin{equation*}
  \delta_n = \log \frac{q_1(1 - \beta_n) - q_1(1 - \alpha_n)}
                     {\widehat{q}_1(1 - \beta_n) - \widehat{q}_1(1 - \alpha_n)}.
\end{equation*}
By Lemma~G.1 and Theorem~G.1 of Deuber et al.~\cite{deuber2024estimation}, under
Assumptions~\ref{asm:causal_identification}--\ref{asm:potential_outcome_distributions},
\begin{equation*}
  \sqrt{k}\!\left(\frac{\widehat{q}_1(1 - \alpha_n)}{q_1(1 - \alpha_n)} - 1\right)
    = O_p(1), \qquad
  \sqrt{k_0}\!\left(\frac{\widehat{q}_1(1 - \beta_n)}{q_1(1 - \beta_n)} - 1\right)
    = O_p(1),
\end{equation*}
so that
$\widehat{q}_1(1 - \alpha_n) / q_1(1 - \alpha_n) \xrightarrow{p} 1$ and
$\widehat{q}_1(1 - \beta_n) / q_1(1 - \beta_n) \xrightarrow{p} 1$.
By the continuous mapping theorem (Theorem~\ref{thm:cmt}),
\begin{equation*}
  \log\!\left(\frac{\widehat{q}_1(1 - \beta_n)}{q_1(1 - \beta_n)}\right) = o_p(1).
\end{equation*}
Since $q_1(1 - \alpha_n) = U_1(1/\alpha_n)$ and
$q_1(1 - \beta_n) = U_1(1/\beta_n)$, by Equation~\eqref{eq:equiv_quantile},
\begin{equation*}
  \frac{q_1(1 - \alpha_n)}{q_1(1 - \beta_n)}
    \sim (\beta_n / \alpha_n)^{\gamma_1},
\end{equation*}
which together with the rate above gives
\begin{equation*}
  \frac{\widehat{q}_1(1 - \alpha_n)}{\widehat{q}_1(1 - \beta_n)} = o_p(1).
\end{equation*}
Hence
$\log\!\Bigl(1 - \frac{q_1(1 - \alpha_n)}{q_1(1 - \beta_n)}\Bigr) = o(1)$ and
$\log\!\Bigl(1 - \frac{\widehat{q}_1(1 - \alpha_n)}{\widehat{q}_1(1 - \beta_n)}\Bigr)
  = o_p(1)$.
Substituting these estimates into the decomposition
\begin{equation*}
  \delta_n = \log\frac{q_1(1-\beta_n)}{\widehat{q}_1(1-\beta_n)}
            + \log\!\left(1 - \frac{q_1(1-\alpha_n)}{q_1(1-\beta_n)}\right)
            - \log\!\left(1 - \frac{\widehat{q}_1(1-\alpha_n)}
                                   {\widehat{q}_1(1-\beta_n)}\right)
\end{equation*}
yields $|\delta_n| = o_p(1)$.  Therefore
\begin{equation*}
  |A_{n,1}^{2}|
    = |\delta_n| \cdot
      \frac{1}{k_0} \sum_{i=1}^{n}
        \frac{D_i}{\pi(X_i)}\,
        \mathbf{1}_{\{Y_i > q_1(1 - \beta_n)\}}
    \xrightarrow{p} 0
\end{equation*}
by Lemma~\ref{lem:A.2}.

\paragraph{Step 3: $A_{n,1}^{3} \xrightarrow{p} 0$.}
By Step~2, $\widehat{q}_1(1 - \alpha_n) / q_1(1 - \alpha_n) \xrightarrow{p} 1$.
Define $\delta_n = (\widehat{q}_1(1 - \alpha_n) - q_1(1 - \alpha_n)) / q_1(1 - \alpha_n) = o_p(1)$.
For $Y_i > q_1(1 - \beta_n)$, write the ratio inside the log as $1 - x$ with
\begin{equation*}
  x \;=\; \frac{\widehat{q}_1(1 - \alpha_n) - q_1(1 - \alpha_n)}{Y_i - q_1(1 - \alpha_n)}
       \;=\; \delta_n \cdot \frac{q_1(1 - \alpha_n)}{Y_i - q_1(1 - \alpha_n)}.
\end{equation*}
Since $\delta_n = o_p(1)$ and the second factor is bounded by
$c\,(\beta_n / \alpha_n)^{\gamma_1} \to 0$ (by the bound established below), we have
$|x| < 1/2$ for large enough $n$.  Hence
\begin{align*}
  \Bigl|\log\frac{Y_i - \widehat{q}_1(1 - \alpha_n)}{Y_i - q_1(1 - \alpha_n)}\Bigr|
  &= \bigl|\log(1 - x)\bigr| \\
  &\le 2|x| \\
  &= 2|\delta_n|\,\frac{q_1(1 - \alpha_n)}{Y_i - q_1(1 - \alpha_n)}.
\end{align*}
Here we used the elementary bound $|\log(1 - x)| \le 2|x|$, valid for $|x| \le 1/2$.
For $Y_i > q_1(1 - \beta_n)$,
\begin{equation*}
  Y_i - q_1(1 - \alpha_n)
  \ge q_1(1 - \beta_n) - q_1(1 - \alpha_n)
  = q_1(1 - \alpha_n)\Bigl(\frac{q_1(1 - \beta_n)}{q_1(1 - \alpha_n)} - 1\Bigr).
\end{equation*}
Since $q_1(1 - \beta_n) / q_1(1 - \alpha_n) \sim (\beta_n / \alpha_n)^{-\gamma_1} \to +\infty$
(by Step~2 and $\beta_n / \alpha_n \to 0$ with $\gamma_1 > 0$), there exists a
constant $c > 0$ such that for all large $n$,
\begin{equation*}
  \frac{q_1(1 - \alpha_n)}{Y_i - q_1(1 - \alpha_n)}
  \le c\,\Bigl(\frac{\beta_n}{\alpha_n}\Bigr)^{\!\gamma_1},
\end{equation*}
uniformly over all $i$ with $Y_i > q_1(1 - \beta_n)$.  Consequently, for every $i$ with
$Y_i > q_1(1 - \beta_n)$,
\begin{equation*}
  \Bigl|\log\frac{Y_i - \widehat{q}_1(1 - \alpha_n)}{Y_i - q_1(1 - \alpha_n)}\Bigr|
  \le 2c\,|\delta_n|\,\Bigl(\frac{\beta_n}{\alpha_n}\Bigr)^{\!\gamma_1}
  = o_p(1) \cdot o(1) = o_p(1),
\end{equation*}
where the right-hand side is independent of $i$.  Hence,
\begin{align*}
  |A_{n,1}^{3}|
  &\le 2c\,|\delta_n|\,\Bigl(\frac{\beta_n}{\alpha_n}\Bigr)^{\!\gamma_1}
     \cdot \frac{1}{k_0} \sum_{i=1}^{n}
       \frac{D_i}{\pi(X_i)}\,
       \mathbf{1}_{\{Y_i > q_1(1 - \beta_n)\}} \\
  &= o_p(1) \cdot O_p(1)
   = o_p(1),
\end{align*}
again by Lemma~\ref{lem:A.2}.

\paragraph{Step 4: $A_{n,1}^{4} \xrightarrow{p} 0$.}
We first establish a per-$i$ sandwich on the summand of $A_{n,1}^{4}$.
Let $\Delta_i :=
  \mathbf{1}_{\{Y_i > \widehat{q}_1(1 - \beta_n)\}}
 - \mathbf{1}_{\{Y_i > q_1(1 - \beta_n)\}} \in \{-1, 0, 1\}$.
For every $i \in \{1, 2, \ldots, n\}$ we claim
\begin{align}
  0
  &\le \Delta_i\,
       \log \frac{Y_i - \widehat{q}_1(1 - \alpha_n)}
                {\widehat{q}_1(1 - \beta_n) - \widehat{q}_1(1 - \alpha_n)}
       \nonumber \\
  &\le \Delta_i\,
       \log \frac{q_1(1 - \beta_n) - \widehat{q}_1(1 - \alpha_n)}
                {\widehat{q}_1(1 - \beta_n) - \widehat{q}_1(1 - \alpha_n)}.
       \label{eq:lem_quantile_per_i}
\end{align}
To verify Equation~\eqref{eq:lem_quantile_per_i}, observe that
$\widehat{q}_1(1 - \beta_n) < Y_i \le q_1(1 - \beta_n)$ when $\Delta_i = +1$, and
$q_1(1 - \beta_n) < Y_i \le \widehat{q}_1(1 - \beta_n)$ when $\Delta_i = -1$.
In either case, substituting the appropriate bound on $Y_i$ into the numerator and
using the monotonicity of the logarithm directly gives the claimed sandwich; the
case $\Delta_i = 0$ is trivial.

By Step~2, $q_1(1 - \beta_n) = \widehat{q}_1(1 - \beta_n)(1 + o_p(1))$ and
$\widehat{q}_1(1 - \alpha_n) = \widehat{q}_1(1 - \beta_n) \cdot o_p(1)$, so
\begin{align*}
  \frac{q_1(1 - \beta_n) - \widehat{q}_1(1 - \alpha_n)}
       {\widehat{q}_1(1 - \beta_n) - \widehat{q}_1(1 - \alpha_n)}
  &= \frac{\widehat{q}_1(1 - \beta_n)(1 + o_p(1))}
       {\widehat{q}_1(1 - \beta_n)(1 - o_p(1))} \\
  &= 1 + o_p(1),
\end{align*}
and therefore
$\log \frac{q_1(1 - \beta_n) - \widehat{q}_1(1 - \alpha_n)}
            {\widehat{q}_1(1 - \beta_n) - \widehat{q}_1(1 - \alpha_n)}
 = o_p(1)$.

Applying the sandwich Equation~\eqref{eq:lem_quantile_per_i} and summing
over $i$ -- noting that the factor
$\log\frac{q_1(1 - \beta_n) - \widehat{q}_1(1 - \alpha_n)}
        {\widehat{q}_1(1 - \beta_n) - \widehat{q}_1(1 - \alpha_n)}$
does not depend on $i$ -- we obtain
\begin{equation*}
  0 \le A_{n,1}^{4}
  \le o_p(1)\,
     \frac{1}{k_0} \sum_{i=1}^{n}
       \frac{D_i}{\widehat{\pi}(X_i)}\,
       \Bigl(
         \mathbf{1}_{\{Y_i > \widehat{q}_1(1 - \beta_n)\}}
         - \mathbf{1}_{\{Y_i > q_1(1 - \beta_n)\}}
       \Bigr)
  \xrightarrow{p} 0,
\end{equation*}
where the last convergence follows from
Lemma~\ref{lem:A.3}.

\paragraph{Step 5: $A_{n,1}^{5} \xrightarrow{p} 0$.}
Factor $A_{n,1}^{5}$ as
\begin{align*}
  |A_{n,1}^{5}|
  &=
  \left|
    \frac{1}{k_0} \sum_{i=1}^{n}
      D_i \!\left(
        \frac{1}{\widehat{\pi}(X_i)} - \frac{1}{\pi(X_i)}
      \right)\!
      \mathbf{1}_{\{Y_i > q_1(1 - \beta_n)\}}
      \log \frac{Y_i - \widehat{q}_1(1 - \alpha_n)}
                {\widehat{q}_1(1 - \beta_n) - \widehat{q}_1(1 - \alpha_n)}
  \right| \\
  &\le
  \sup_{x \in \mathrm{supp}(X)}
    \left|\frac{1}{\widehat{\pi}(x)} - \frac{1}{\pi(x)}\right|
  \cdot
  \frac{1}{k_0} \sum_{i=1}^{n}
    D_i\,
    \mathbf{1}_{\{Y_i > q_1(1 - \beta_n)\}}
    \left|\log \frac{Y_i - \widehat{q}_1(1 - \alpha_n)}
                  {\widehat{q}_1(1 - \beta_n) - \widehat{q}_1(1 - \alpha_n)}\right|.
\end{align*}
By Deuber et al.~\cite{deuber2024estimation} (Lemma~G.3) (which uses Assumption~\ref{asm:sieve_assumptions} and the continuity of $g(\pi) = 1 / \pi$ on $[c, 1]$, recall $\pi(\cdot) \ge c > 0$ by Assumption~\ref{asm:causal_identification}),
\begin{equation*}
  \sup_{x \in \mathrm{supp}(X)}
    \left|\frac{1}{\widehat{\pi}(x)} - \frac{1}{\pi(x)}\right|
  = o_p(1).
\end{equation*}
For the second factor, the triangle inequality together with the
decomposition of the log in the preamble gives
\begin{align*}
  &\frac{1}{k_0} \sum_{i=1}^{n}
    D_i\,
    \mathbf{1}_{\{Y_i > q_1(1 - \beta_n)\}}
    \left|\log \frac{Y_i - \widehat{q}_1(1 - \alpha_n)}
                  {\widehat{q}_1(1 - \beta_n) - \widehat{q}_1(1 - \alpha_n)}\right| \\
  &\le
  (1 - c)\,
  \frac{1}{k_0} \sum_{i=1}^{n}
    \frac{D_i}{\pi(X_i)}\,
    \mathbf{1}_{\{Y_i > q_1(1 - \beta_n)\}}
    \left|\log \frac{Y_i - \widehat{q}_1(1 - \alpha_n)}
                  {\widehat{q}_1(1 - \beta_n) - \widehat{q}_1(1 - \alpha_n)}\right| \\
  &\le
  (1 - c)\bigl(|A_{n,1}^{1}| + |A_{n,1}^{2}| + |A_{n,1}^{3}|\bigr)
  = O_p(1).
\end{align*}
Combining, $|A_{n,1}^{5}| = o_p(1) \cdot O_p(1) = o_p(1)$.

\paragraph{Conclusion.}
Combining the five steps,
$\widehat{\gamma}_{1}^{F}(\beta_n, \alpha_n)
 = A_{n,1}^{1} + A_{n,1}^{2} + A_{n,1}^{3} + A_{n,1}^{4} + A_{n,1}^{5}
 \xrightarrow{p} \gamma_1 + 0 + 0 + 0 + 0 = \gamma_1$.
The proof for $j = 0$ is identical after replacing
$D_i$ with $1 - D_i$ and $\pi(X_i)$ with $1 - \pi(X_i)$
throughout.  \qedhere
\end{proof}

\subsection{Proof of Lemma~\ref{thm:causal_fraga_equation}}
\label{sec:causal_fraga_equation_proof}

\begin{proof}
We prove the theorem for $j = 1$; the case $j = 0$ follows by
replacing $D_i$ with $1 - D_i$ and $\pi$ with $1 - \pi$ throughout.

\paragraph{Decomposition.}
Following the same decomposition as in the consistency proof
(Theorem~\ref{thm:causal_fraga_consistency}), we write
\begin{equation*}
  \widehat{\gamma}_{1}^{F}(\beta_n, \alpha_n)
  = A_{n,1}^{1} + A_{n,1}^{2} + A_{n,1}^{3} + A_{n,1}^{4} + A_{n,1}^{5},
\end{equation*}
where $A_{n,1}^{1},\ldots,A_{n,1}^{5}$ are defined as in the proof of
Theorem~\ref{thm:causal_fraga_consistency} (we omit the repetition here).  From the consistency proof,
$A_{n,1}^{2}, A_{n,1}^{3}, A_{n,1}^{4}, A_{n,1}^{5} = o_p(1)$, but
$A_{n,1}^{1} \xrightarrow{p} \gamma_1$ is the dominant term and requires
a refined analysis for the CLT.

By the consistency relation $Y_i = D_i Y_i(1) + (1 - D_i) Y_i(0)$,
the summand of $A_{n,1}^{1}$ vanishes unless $D_i = 1$, in which case
$Y_i = Y_i(1)$.  Using the tail quantile function $U_1$ and writing
$Z_i(1) = 1 / (1 - F_1(Y_i(1)))$, we have
$Y_i(1) = U_1(Z_i(1))$,
$q_1(1 - \beta_n) = U_1(\beta_n^{-1})$,
$q_1(1 - \alpha_n) = U_1(\alpha_n^{-1})$, and therefore
\begin{equation}\label{eq:A1n_form}
  A_{n,1}^{1}
  = \frac{1}{k_0} \sum_{i=1}^{n}
    \frac{D_i}{\pi(X_i)}\,
    \mathbf{1}_{\{Z_i(1) > \beta_n^{-1}\}}
    \log\!\left(
      \frac{U_1(Z_i(1)) - U_1(\alpha_n^{-1})}
           {U_1(\beta_n^{-1}) - U_1(\alpha_n^{-1})}
    \right).
\end{equation}

\paragraph{Second-order expansion of the log-ratio.}
Under Assumption~\ref{asm:second_order}, $U_1$ satisfies the
second-order regular variation relation (with the convention
$\frac{x^{\rho} - 1}{\rho} := \log x$ when $\rho = 0$, understood via
continuous extension)
\begin{equation*}
  \frac{U_1(t x)}{U_1(t)}
  = x^{\gamma_1} + A_1(t)\, x^{\gamma_1}\,
    \frac{x^{\rho_1} - 1}{\rho_1}\,(1 + o(1)), \qquad t \to +\infty,
\end{equation*}
for every $x > 0$.  For $i \in \{1, 2, \ldots, n\}$,
\begin{align*}
  &\log\frac{U_1(Z_i(1)) - U_1(\alpha_n^{-1})}
            {U_1(\beta_n^{-1}) - U_1(\alpha_n^{-1})} \\
  &\quad = \log\!\left(\frac{U_1(Z_i(1))}{U_1(\alpha_n^{-1})} - 1\right)
            - \log\!\left(\frac{U_1(\beta_n^{-1})}{U_1(\alpha_n^{-1})} - 1\right) \\
  &\quad = \log\!\Bigl(
             (Z_i(1)\,\alpha_n)^{\gamma_1}
             + A_1(\alpha_n^{-1})\,(Z_i(1)\,\alpha_n)^{\gamma_1}\,
               \frac{(Z_i(1)\,\alpha_n)^{\rho_1} - 1}{\rho_1}\,
               (1 + o_p(1)) - 1
           \Bigr) \\
  &\qquad - \log\!\Bigl(
             \Bigl(\frac{\alpha_n}{\beta_n}\Bigr)^{\!\gamma_1}
             + A_1(\alpha_n^{-1})\,\Bigl(\frac{\alpha_n}{\beta_n}\Bigr)^{\!\gamma_1}\,
               \frac{(\alpha_n/\beta_n)^{\rho_1} - 1}{\rho_1}\,
               (1 + o_p(1)) - 1
           \Bigr) \\
  &\quad = \gamma_1 \log(Z_i(1)\,\alpha_n)
            + \log\!\Bigl(
                1 + A_1(\alpha_n^{-1})\,
                  \frac{(Z_i(1)\,\alpha_n)^{\rho_1} - 1}{\rho_1}\,
                  (1 + o_p(1)) - (Z_i(1)\,\alpha_n)^{-\gamma_1}
              \Bigr) \\
  &\qquad - \Bigl(\gamma_1 \log\Bigl(\frac{\alpha_n}{\beta_n}\Bigr)
            + \log\!\Bigl(
                1 + A_1(\alpha_n^{-1})\,
                  \frac{(\alpha_n/\beta_n)^{\rho_1} - 1}{\rho_1}\,
                  (1 + o_p(1)) - \Bigl(\frac{\alpha_n}{\beta_n}\Bigr)^{\!-\gamma_1}
              \Bigr)\Bigr) \\
  &\quad = \gamma_1 \log(Z_i(1)\,\beta_n)
            + \Bigl(\frac{\alpha_n}{\beta_n}\Bigr)^{\!-\gamma_1}
              \Bigl(1 - (Z_i(1)\,\beta_n)^{-\gamma_1}\Bigr)(1 + o_p(1)) \\
  &\qquad + A_1(\alpha_n^{-1})\,
              \Bigl(\frac{\alpha_n}{\beta_n}\Bigr)^{\!\rho_1}
              \frac{(Z_i(1)\,\beta_n)^{\rho_1} - 1}{\rho_1}\,(1 + o_p(1)).
\end{align*}
The last equality uses $\log(1 + u) = u\,(1 + o(1))$ as $u \to 0$: as
$n \to +\infty$, $\alpha_n^{-1} \to +\infty$ and $\alpha_n/\beta_n \to +\infty$;
for $Z_i(1) \ge \beta_n^{-1}$ we have $Z_i(1)\,\alpha_n \to +\infty$, so
$A_1(\alpha_n^{-1}) \to 0$,
$(Z_i(1)\,\alpha_n)^{\rho_1} \to 0$,
$(Z_i(1)\,\alpha_n)^{-\gamma_1} \to 0$,
$(\alpha_n/\beta_n)^{\rho_1} \to 0$, and
$(\alpha_n/\beta_n)^{-\gamma_1} \to 0$, all of which enter the
inner logarithms.  This yields the final compact form
\begin{align}
  &\log\!\left(
     \frac{U_1(Z_i(1)) - U_1(\alpha_n^{-1})}
          {U_1(\beta_n^{-1}) - U_1(\alpha_n^{-1})}
   \right) \nonumber \\
  &\quad = \gamma_1 \log\!\bigl(Z_i(1)\,\beta_n\bigr) \nonumber\\
  &\qquad + \Bigl(\frac{\alpha_n}{\beta_n}\Bigr)^{\!-\gamma_1}
           \Bigl(1 - (Z_i(1)\,\beta_n)^{-\gamma_1}\Bigr)
           \bigl(1 + o_p(1)\bigr) \nonumber\\
  &\qquad + A_1(\alpha_n^{-1})\,
           \Bigl(\frac{\alpha_n}{\beta_n}\Bigr)^{\!\rho_1}
           \frac{(Z_i(1)\,\beta_n)^{\rho_1} - 1}{\rho_1}
           \bigl(1 + o_p(1)\bigr).
  \label{eq:second_order_exp}
\end{align}

\paragraph{Three-term decomposition of $A_{n,1}^{1}$.}
Substituting~\eqref{eq:second_order_exp} into~\eqref{eq:A1n_form} yields
\begin{equation*}
  A_{n,1}^{1} = A_{n,1}^{1,1} + A_{n,1}^{1,2} + A_{n,1}^{1,3},
\end{equation*}
where
\begin{align*}
  A_{n,1}^{1,1}
  &= \frac{\gamma_1}{k_0} \sum_{i=1}^{n}
     \frac{D_i}{\pi(X_i)}\,
     \mathbf{1}_{\{Z_i(1) > \beta_n^{-1}\}}
     \log\!\bigl(Z_i(1)\,\beta_n\bigr), \\[4pt]
  A_{n,1}^{1,2}
  &= \frac{1}{k_0} \sum_{i=1}^{n}
     \frac{D_i}{\pi(X_i)}\,
     \mathbf{1}_{\{Z_i(1) > \beta_n^{-1}\}}
     \Bigl(\frac{\alpha_n}{\beta_n}\Bigr)^{\!-\gamma_1}
     \Bigl(1 - (Z_i(1)\,\beta_n)^{-\gamma_1}\Bigr)
     \bigl(1 + o_p(1)\bigr), \\[4pt]
  A_{n,1}^{1,3}
  &= \frac{1}{k_0} \sum_{i=1}^{n}
     \frac{D_i}{\pi(X_i)}\,
     \mathbf{1}_{\{Z_i(1) > \beta_n^{-1}\}}
     \,A_1(\alpha_n^{-1})\,
     \Bigl(\frac{\alpha_n}{\beta_n}\Bigr)^{\!\rho_1}
     \frac{(Z_i(1)\,\beta_n)^{\rho_1} - 1}{\rho_1}
     \bigl(1 + o_p(1)\bigr).
\end{align*}
Since $\alpha_n^{-1} = n/k$, we have
$(\alpha_n / \beta_n)^{-\gamma_1} = (k_0 / k)^{\gamma_1}$ and
$A_1(\alpha_n^{-1}) = A_1(n/k)$.

\paragraph{Uniformity of the $o_p(1)$ remainder.}
Although the realization of the $o_p(1)$ error in~\eqref{eq:second_order_exp}
depends on the sampled value of $Z_i(1)$, the \emph{order} of the
remainder is the same for all qualifying $i$.  Indeed, since
$t = \alpha_n^{-1}$ is common to all summands and the second-order
regular variation is uniform in $x$ on the sliding window
$x \in [\alpha_n / \beta_n, +\infty)$ (the convergence in second-order regular variation is uniform in~$x$; see de Haan and Ferreira~\cite{de2006extreme} for details),
\begin{equation*}
  \sup_{i\,:\,Z_i(1) > \beta_n^{-1}}
    \bigl|o_p^{(i)}(1)\bigr| = o_p(1),
\end{equation*}
i.e.\ there exists a single $o_p(1)$ random variable dominating all
per-$i$ errors.

We now analyse each term separately.

\paragraph{Analysis of $A_{n,1}^{1,1}$ --- Central limit theorem.}
Define
\begin{equation*}
  R_{n,1,i}
  = \frac{D_i}{\pi(X_i)}\,
    \mathbf{1}_{\{Z_i(1) > \beta_n^{-1}\}}
    \log\!\bigl(Z_i(1)\,\beta_n\bigr) - \beta_n.
\end{equation*}
Then
$A_{n,1}^{1,1} = \frac{\gamma_1}{k_0} \sum_{i=1}^{n} R_{n,1,i} + \gamma_1$.

\emph{Mean.}  Recall that under Assumption~\ref{asm:potential_outcome_distributions},
$F_1(Y_i(1)) \sim U(0,1)$, so $Z_i(1) = 1/(1 - F_1(Y_i(1)))$ has density
$\psi_1(z) = z^{-2}$ for $z \ge 1$.  By Lemma~\ref{lem:A.1} with $r = 1$,
\begin{equation*}
  \mathbb{E}(R_{n,1,i})
  = \int_{\beta_n^{-1}}^{+\infty} \log(z\,\beta_n)\, z^{-2}\, dz - \beta_n
  = \beta_n - \beta_n = 0.
\end{equation*}

\emph{Second moment.}  Since $\mathbb{E}(R_{n,1,i}) = 0$,
$\mathrm{Var}(R_{n,1,i}) = \mathbb{E}(R_{n,1,i}^{2})$.  Applying
Lemma~\ref{lem:A.1} with $r = 2$ yields the bounds
\begin{align*}
  \mathbb{E}(R_{n,1,i}^{2})
  &\le \frac{1}{c}\,
      \mathbb{E}\!\left[\mathbf{1}_{\{Z_i(1) > \beta_n^{-1}\}}
        \log^{2}(Z_i(1)\,\beta_n)\right] - \beta_n^{2}
   = \frac{2\beta_n}{c} - \beta_n^{2},\\[4pt]
  \mathbb{E}(R_{n,1,i}^{2})
  &\ge \frac{1}{1 - c}\,
      \mathbb{E}\!\left[\mathbf{1}_{\{Z_i(1) > \beta_n^{-1}\}}
        \log^{2}(Z_i(1)\,\beta_n)\right] - \beta_n^{2}
   = \frac{2\beta_n}{1 - c} - \beta_n^{2}.
\end{align*}
Hence $\mathbb{E}(R_{n,1,i}^{2}) = O(\beta_n)$ and
\begin{equation*}
  \sigma_{n,1}^{2} := \mathbb{E}(R_{n,1,i}^{2})
  \in \Bigl[\frac{2}{1 - c}\,\beta_n - \beta_n^{2},\;
            \frac{2}{c}\,\beta_n - \beta_n^{2}\Bigr].
\end{equation*}
Define $a_{n,1} := \sqrt{n \sigma_{n,1}^{2} / k_0}$; then $a_{n,1}$ is a
bounded sequence.

\emph{Lindeberg condition.}  Let
$T_{n,1} = \sum_{i=1}^{n} R_{n,1,i}$ and
$V_n^{2} = \mathrm{Var}(T_{n,1}) = n \sigma_{n,1}^{2}$.
To verify the Lindeberg condition, we use the bound
$\mathbf{1}_{\{|R_{n,1,i}| > \varepsilon V_n\}}
 \le R_{n,1,i}^{2} / (\varepsilon^{2} V_n^{2})$, giving
\begin{equation*}
  \frac{1}{V_n^{2}} \sum_{i=1}^{n}
    \mathbb{E}\!\left(R_{n,1,i}^{2}\,
      \mathbf{1}_{\{|R_{n,1,i}| > \varepsilon V_n\}}\right)
  \le \frac{1}{\varepsilon^{2} V_n^{4}}
    \sum_{i=1}^{n} \mathbb{E}(R_{n,1,i}^{4}).
\end{equation*}
Write $G_i := \frac{D_i}{\pi(X_i)}\, \mathbf{1}_{\{Z_i(1) > \beta_n^{-1}\}}
\log(Z_i(1)\,\beta_n)$, so that $R_{n,1,i} = G_i - \beta_n$.  Expanding,
\begin{equation*}
  \mathbb{E}(R_{n,1,i}^{4})
  = \mathbb{E}(G_i^{4}) - 4\beta_n \mathbb{E}(G_i^{3})
    + 6\beta_n^{2} \mathbb{E}(G_i^{2})
    - 4\beta_n^{3} \mathbb{E}(G_i) + \beta_n^{4}.
\end{equation*}
By Lemma~\ref{lem:A.1} with $r = 1, 2, 3, 4$ and the change of
variables $u = \log(z\,\beta_n)$,
\begin{align*}
  \mathbb{E}\!\bigl[\mathbf{1}_{\{Z_i(1) > \beta_n^{-1}\}}
    \log^{k}(Z_i(1)\,\beta_n)\bigr]
  &= \int_{\beta_n^{-1}}^{+\infty}
     \log^{k}(z\,\beta_n)\, z^{-2}\, dz
   = \beta_n \int_{0}^{\infty} u^{k} e^{-u}\, du
   = k!\,\beta_n,
\end{align*}
so that
\begin{align*}
  \mathbb{E}(G_i^{4}) &\le \frac{24}{c^{3}}\,\beta_n, &
  \mathbb{E}(G_i^{3}) &\ge \frac{6}{(1 - c)^{2}}\,\beta_n, \\
  \mathbb{E}(G_i^{2}) &\le \frac{2}{c}\,\beta_n, &
  \mathbb{E}(G_i) &= \beta_n.
\end{align*}
Plugging these bounds into the expansion,
\begin{equation*}
  \mathbb{E}(R_{n,1,i}^{4})
  \le \frac{24}{c^{3}}\,\beta_n
     - \frac{24}{(1 - c)^{2}}\,\beta_n^{2}
     + \frac{12}{c}\,\beta_n^{3}
     - 3\,\beta_n^{4}
   = O(\beta_n).
\end{equation*}
Consequently, for every $\varepsilon > 0$,
\begin{equation*}
  \frac{1}{V_n^{2}} \sum_{i=1}^{n}
    \mathbb{E}\!\left(R_{n,1,i}^{2}\,
      \mathbf{1}_{\{|R_{n,1,i}| > \varepsilon V_n\}}\right)
  \le \frac{1}{\varepsilon^{2}\, V_n^{4}}
    \sum_{i=1}^{n} \mathbb{E}(R_{n,1,i}^{4})
  \le \frac{1}{\varepsilon^{2}}\,
    \frac{n \cdot O(\beta_n)}{(n \cdot O(\beta_n))^{2}}
  = \frac{1}{\varepsilon^{2}}\,
    O\!\left(\frac{1}{n \beta_n}\right)
  \to 0.
\end{equation*}
Thus the Lindeberg condition holds, and by the Lindeberg CLT
(Theorem~\ref{thm:lindeberg_clt}),
\begin{equation*}
  \frac{T_{n,1}}{V_n} \xrightarrow{d} \mathcal{N}(0, 1).
\end{equation*}
Let $P_{n,1} := T_{n,1} / V_n \xrightarrow{d} \mathcal{N}(0, 1)$.  Since
$V_n = a_{n,1} \sqrt{k_0}$, we have
\begin{equation}\label{eq:A11_result}
  A_{n,1}^{1,1}
  = \gamma_1 + \frac{\gamma_1}{k_0}\, V_n P_{n,1}
  = \gamma_1 + \gamma_1 a_{n,1} \frac{P_{n,1}}{\sqrt{k_0}}.
\end{equation}

\paragraph{Analysis of $A_{n,1}^{1,2}$ --- Leading bias term.}
Define
\begin{equation*}
  v_{n,i}
  = \frac{D_i}{\pi(X_i)}\,
    \mathbf{1}_{\{Z_i(1) > \beta_n^{-1}\}}
    \Bigl(1 - (Z_i(1)\,\beta_n)^{-\gamma_1}\Bigr),\qquad
  T_{n,2} = \sum_{i=1}^{n} v_{n,i}.
\end{equation*}
By Lemma~\ref{lem:A.1} with $r = 1$,
\begin{equation*}
  \mathbb{E}(v_{n,i})
  = \int_{\beta_n^{-1}}^{+\infty}
    \bigl(1 - (\beta_n z)^{-\gamma_1}\bigr)\, z^{-2}\, dz
  = \frac{\gamma_1}{1 + \gamma_1}\,\beta_n.
\end{equation*}
Hence $\mathbb{E}(T_{n,2} / k_0) = \gamma_1 / (1 + \gamma_1) =: b_1$.
Applying Lemma~\ref{lem:A.1} with $r = 2$ gives
$\mathbb{E}(v_{n,i}^{2}) = O(\beta_n)$, so
$\mathrm{Var}(T_{n,2} / k_0) = O(1 / (n \beta_n)) \to 0$.
The WLLN for triangular arrays
(Theorem~\ref{thm:wlln_triangular}) yields
$T_{n,2} / k_0 \xrightarrow{p} b_1$.  Since the definition of $A_{n,1}^{1,2}$
carries a factor $(1 + o_p(1))$ inherited from the expansion,
we have
\begin{equation}\label{eq:A12_result}
  A_{n,1}^{1,2}
  = \Bigl(\frac{k_0}{k}\Bigr)^{\!\gamma_1}
    \cdot \frac{T_{n,2}}{k_0} \cdot \bigl(1 + o_p(1)\bigr)
  = b_1 \Bigl(\frac{k_0}{k}\Bigr)^{\!\gamma_1}
    \bigl(1 + o_p(1)\bigr).
\end{equation}

\paragraph{Analysis of $A_{n,1}^{1,3}$ --- Second-order bias term.}
Define
\begin{equation*}
  w_{n,i}
  = \frac{D_i}{\pi(X_i)}\,
    \mathbf{1}_{\{Z_i(1) > \beta_n^{-1}\}}
    \frac{(Z_i(1)\,\beta_n)^{\rho_1} - 1}{\rho_1},\qquad
  T_{n,3} = \sum_{i=1}^{n} w_{n,i}.
\end{equation*}
By Lemma~\ref{lem:A.1} with $r = 1$,
\begin{equation*}
  \mathbb{E}(w_{n,i})
  = \int_{\beta_n^{-1}}^{+\infty}
    \frac{(\beta_n z)^{\rho_1} - 1}{\rho_1}\, z^{-2}\, dz
  = \frac{1}{1 - \rho_1}\,\beta_n.
\end{equation*}
Hence $\mathbb{E}(T_{n,3} / k_0) = 1 / (1 - \rho_1) =: c_1$.
Applying Lemma~\ref{lem:A.1} with $r = 2$ again gives
$\mathbb{E}(w_{n,i}^{2}) = O(\beta_n)$, so
$\mathrm{Var}(T_{n,3} / k_0) = O(1 / (n \beta_n)) \to 0$.
The WLLN gives $T_{n,3} / k_0 \xrightarrow{p} c_1$.  Incorporating the
factor $(1 + o_p(1))$ inherited from the expansion, we obtain
\begin{equation}\label{eq:A13_result}
  A_{n,1}^{1,3}
  = A_1\!\Bigl(\frac{n}{k}\Bigr)\,
    \Bigl(\frac{k_0}{k}\Bigr)^{\!-\rho_1}
    \cdot \frac{T_{n,3}}{k_0} \cdot \bigl(1 + o_p(1)\bigr)
  = c_1\, A_1\!\Bigl(\frac{n}{k}\Bigr)\,
    \Bigl(\frac{k_0}{k}\Bigr)^{\!-\rho_1}
    \bigl(1 + o_p(1)\bigr).
\end{equation}

\paragraph{Final assembly.}
Combining~\eqref{eq:A11_result},~\eqref{eq:A12_result},
and~\eqref{eq:A13_result}, we obtain
\begin{equation*}
  A_{n,1}^{1}
  = \gamma_1
  + \gamma_1 a_{n,1} \frac{P_{n,1}}{\sqrt{k_0}}
  + b_1 \Bigl(\frac{k_0}{k}\Bigr)^{\!\gamma_1}
    \bigl(1 + o_p(1)\bigr)
  + c_1\, A_1\!\Bigl(\frac{n}{k}\Bigr)\,
    \Bigl(\frac{k_0}{k}\Bigr)^{\!-\rho_1}
    \bigl(1 + o_p(1)\bigr).
\end{equation*}
From the consistency proof,
$A_{n,1}^{2} + A_{n,1}^{3} + A_{n,1}^{4} + A_{n,1}^{5} = o_p(1)$.  Therefore
\begin{equation*}
  \widehat{\gamma}_{1}^{F}(\beta_n, \alpha_n)
  = \gamma_1
  + \gamma_1 a_{n,1} \frac{P_{n,1}}{\sqrt{k_0}}
  + b_1 \Bigl(\frac{k_0}{k}\Bigr)^{\!\gamma_1}
    \bigl(1 + o_p(1)\bigr)
  + c_1\, A_1\!\Bigl(\frac{n}{k}\Bigr)\,
    \Bigl(\frac{k_0}{k}\Bigr)^{\!-\rho_1}
    \bigl(1 + o_p(1)\bigr)
  + A_{n,1},
\end{equation*}
with $P_{n,1} \xrightarrow{d} \mathcal{N}(0, 1)$,
$a_{n,1}^{2} = n\,\mathbb{E}(R_{n,1,i}^{2}) / k_0$,
$b_1 = \gamma_1 / (1 + \gamma_1)$,
$c_1 = 1 / (1 - \rho_1)$.  The proof for $j = 0$ follows by the symmetric
substitution of $D$ with $1 - D$, $\pi$ with $1 - \pi$, $\widehat{\pi}$
with $1 - \widehat{\pi}$, $q_1$ with $q_0$, and $\widehat{q}_1$ with
$\widehat{q}_0$ throughout.
\qedhere
\end{proof}

\subsection{Proof of Theorem~\ref{thm:causal_fraga_clt}}

\begin{proof}
From Lemma~\ref{thm:causal_fraga_equation}, multiplying by $\sqrt{k_0}$ gives
\begin{equation*}
  \sqrt{k_0}\,\bigl(\widehat{\gamma}_{j}^{F} - \gamma_j\bigr)
  = \gamma_j a_{n,j} P_{n,j}
    + \sqrt{k_0}\,b_j \Bigl(\frac{k_0}{k}\Bigr)^{\!\gamma_j} (1 + o_p(1))
    + \sqrt{k_0}\,c_j A_j\Bigl(\frac{n}{k}\Bigr)
               \Bigl(\frac{k_0}{k}\Bigr)^{\!-\rho_j} (1 + o_p(1))
    + \sqrt{k_0}\,A_{n,j}.
\end{equation*}
Under the additional assumptions of Theorem~\ref{thm:causal_fraga_clt}
($\sqrt{k}\,A_j(n/k) \to \lambda_j$ and
$(k_0/k)^{\gamma_j} = o(k_0^{-1/2})$, $A_{n,j} = o_p(k_0^{-1/2})$),
the two bias terms and the remainder satisfy
\begin{align*}
  \sqrt{k_0}\,b_j (k_0 / k)^{\gamma_j} &= o_p(1), \\
  \sqrt{k_0}\,c_j\,A_j(n/k)\,(k_0 / k)^{-\rho_j} &= o_p(1), \\
  \sqrt{k_0}\,A_{n,j} &= o_p(1).
\end{align*}
Hence
\begin{equation*}
  \sqrt{k_0}\,\bigl(\widehat{\gamma}_{j}^{F} - \gamma_j\bigr)
  = \gamma_j a_{n,j} P_{n,j} + o_p(1).
\end{equation*}
By Assumption~\ref{asm:an_convergence}, $a_{n,j} \to a_j$, so
$\gamma_j a_{n,j} P_{n,j} \xrightarrow{d}
 \mathcal{N}(0, \gamma_j^{2} a_j^{2})$ by
Slutsky's theorem (Theorem~\ref{thm:slutsky}) and
$P_{n,j} \xrightarrow{d} \mathcal{N}(0, 1)$.
\qedhere
\end{proof}

\subsection{Proof of Theorem~\ref{thm:causal_fraga_joint_clt}}

\begin{proof}
We establish joint convergence.  Recall the decomposition from
the proof of Theorem~\ref{thm:causal_fraga_consistency}:
\begin{equation*}
  \widehat{\gamma}_{j}^{F}(\beta_n, \alpha_n)
  = A_{n,j}^{1} + A_{n,j},
\end{equation*}
where $A_{n,j} := A_{n,j}^{2} + A_{n,j}^{3} + A_{n,j}^{4} + A_{n,j}^{5}
= o_p(1)$.  Moreover, in the proof of
Lemma~\ref{thm:causal_fraga_equation}, the leading term
$A_{n,j}^{1}$ is further decomposed as
\begin{equation*}
  A_{n,j}^{1} = A_{n,j}^{1,1} + A_{n,j}^{1,2} + A_{n,j}^{1,3},
\end{equation*}
with
\begin{equation*}
  A_{n,j}^{1,1}
  = \frac{\gamma_j}{k_0} \sum_{i=1}^{n}
    \frac{D_{i,j}}{\pi_{j}(X_i)}\,
    \mathbf{1}_{\{Z_i(j) > \beta_n^{-1}\}}
    \log\!\bigl(Z_i(j)\,\beta_n\bigr),
\end{equation*}
where $(D_{i,1}, \pi_1) := (D_i, \pi)$ and
$(D_{i,0}, \pi_0) := (1 - D_i, 1 - \pi)$.

Define the centred leading sums $T_{n,j} := \sum_{i=1}^{n} R_{n,j,i}$,
where
\[
  R_{n,j,i} := \frac{D_{i,j}}{\pi_{j}(X_i)}\,
    \mathbf{1}_{\{Z_i(j) > \beta_n^{-1}\}}
    \log\!\bigl(Z_i(j)\,\beta_n\bigr) - \beta_n,
  \qquad j \in \{0, 1\},
\]
and $Z_i(j) = 1 / (1 - F_j(Y_i(j)))$.  Then
$A_{n,j}^{1,1} = \gamma_j + \gamma_j T_{n,j} / k_0$.

Recall from the proof of Lemma~\ref{thm:causal_fraga_equation} that
\begin{align*}
  A_{n,j}^{1,2}
  &= b_j \Bigl(\frac{k_0}{k}\Bigr)^{\!\gamma_j}
     \bigl(1 + o_p(1)\bigr), \\[4pt]
  A_{n,j}^{1,3}
  &= c_j\, A_j\Bigl(\frac{n}{k}\Bigr)\,
     \Bigl(\frac{k_0}{k}\Bigr)^{\!-\rho_j}
     \bigl(1 + o_p(1)\bigr).
\end{align*}
Under the rate assumptions of Theorem~\ref{thm:causal_fraga_joint_clt}, they satisfy
$\sqrt{k_0}\,A_{n,j}^{1,2} = o_p(1)$ and
$\sqrt{k_0}\,A_{n,j}^{1,3} = o_p(1)$, while the remainder satisfies
$A_{n,j} = o_p(k_0^{-1/2})$.  Chaining the equalities
\begin{align*}
  \widehat{\gamma}_{j}^{F} - \gamma_j
  &= A_{n,j}^{1} + A_{n,j} - \gamma_j \\
  &= (A_{n,j}^{1,1} - \gamma_j) + A_{n,j}^{1,2} + A_{n,j}^{1,3} + A_{n,j},
\end{align*}
and multiplying by $\sqrt{k_0}$,
\begin{align*}
  \sqrt{k_0}\,\bigl(\widehat{\gamma}_{j}^{F} - \gamma_j\bigr)
  &= \gamma_j\,\frac{T_{n,j}}{\sqrt{k_0}}
   + \sqrt{k_0}\,A_{n,j}^{1,2}
   + \sqrt{k_0}\,A_{n,j}^{1,3}
   + \sqrt{k_0}\,A_{n,j} \\
  &= \gamma_j\,\frac{T_{n,j}}{\sqrt{k_0}} + o_p(1).
\end{align*}
Hence it suffices to prove joint convergence of
$(\gamma_1\, T_{n,1} / \sqrt{k_0},\, \gamma_0\, T_{n,0} / \sqrt{k_0})$.

\paragraph{Step 1: Cross-covariance vanishes.}
For $i \neq \ell$, independence of the observations gives
$\operatorname{Cov}(R_{n,1,i}, R_{n,0,\ell}) = 0$.  For $i = \ell$,
observe that $D_i(1 - D_i) = 0$, so the product of the two uncentred
indicator factors in $R_{n,1,i}$ and $R_{n,0,i}$ vanishes.
Expanding $R_{n,1,i} R_{n,0,i}$ and using
$\mathbb{E}(R_{n,j,i}) = 0$,
\begin{align*}
  \mathbb{E}\!\bigl[R_{n,1,i} R_{n,0,i}\bigr]
  &= -\beta_n\;
     \mathbb{E}\!\left[\frac{D_i}{\pi(X_i)}\,
       \mathbf{1}_{\{Z_i(1) > \beta_n^{-1}\}}
       \log\!\bigl(Z_i(1)\,\beta_n\bigr)\right] \\
  &\quad -\beta_n\;
     \mathbb{E}\!\left[\frac{1 - D_i}{1 - \pi(X_i)}\,
       \mathbf{1}_{\{Z_i(0) > \beta_n^{-1}\}}
       \log\!\bigl(Z_i(0)\,\beta_n\bigr)\right] + \beta_n^{2}.
\end{align*}
By Lemma~\ref{lem:A.1} with $r = 1$, each of the two expectations on
the right-hand side equals $\beta_n$, hence
\begin{equation*}
  \mathbb{E}\!\bigl[R_{n,1,i} R_{n,0,i}\bigr]
  = -\beta_n^{2}.
\end{equation*}
Therefore
\begin{equation*}
  \operatorname{Cov}\!\left(\gamma_1\,\frac{T_{n,1}}{\sqrt{k_0}},\;
                        \gamma_0\,\frac{T_{n,0}}{\sqrt{k_0}}\right)
  = \gamma_1\gamma_0\,\frac{n}{k_0}\;
    \mathbb{E}\!\bigl[R_{n,1,i} R_{n,0,i}\bigr]
  = -\gamma_1\gamma_0\,\beta_n \;\longrightarrow\; 0,
\end{equation*}
i.e.\ the cross-covariance of the two scaled leading sums tends to
zero.

\paragraph{Step 2: Cram\'er--Wold device.}
Let $\mathbf{v} = (v_1, v_0)^T \in \mathbb{R}^2$ be an arbitrary
non-zero vector and set
\begin{equation*}
  L_n(\mathbf{v})
  := v_1\,\gamma_1\,\frac{T_{n,1}}{\sqrt{k_0}}
   + v_0\,\gamma_0\,\frac{T_{n,0}}{\sqrt{k_0}}
   = \frac{1}{\sqrt{k_0}} \sum_{i=1}^{n}
       \bigl(v_1 \gamma_1 R_{n,1,i} + v_0 \gamma_0 R_{n,0,i}\bigr).
\end{equation*}
The summands are independent with mean zero, and the variance of
$L_n(\mathbf{v})$ is, by Step~1,
\begin{align*}
  \mathrm{Var}\bigl(L_n(\mathbf{v})\bigr)
  &= v_1^{2}\,\mathrm{Var}\!\left(\gamma_1\,\frac{T_{n,1}}{\sqrt{k_0}}\right)
   + v_0^{2}\,\mathrm{Var}\!\left(\gamma_0\,\frac{T_{n,0}}{\sqrt{k_0}}\right)
   + 2 v_1 v_0\;\mathrm{Cov}\!\left(\gamma_1\,\frac{T_{n,1}}{\sqrt{k_0}},\;
                                     \gamma_0\,\frac{T_{n,0}}{\sqrt{k_0}}\right) \\
  &= v_1^{2}\, \gamma_1^{2} a_{n,1}^{2}
   + v_0^{2}\, \gamma_0^{2} a_{n,0}^{2}
   + o(1),
\end{align*}
since the marginal variances satisfy
$\mathrm{Var}(\gamma_j T_{n,j} / \sqrt{k_0}) = \gamma_j^{2} a_{n,j}^{2}$
(this follows from the proof of Lemma~\ref{thm:causal_fraga_equation}),
and the cross-covariance is $-\gamma_1\gamma_0 \beta_n \to 0$.  Let us
verify the Lindeberg condition explicitly.  Define the
standardized summand
\begin{equation*}
  W_{n,i} := \frac{1}{\sqrt{k_0}}
    \bigl(v_1 \gamma_1 R_{n,1,i} + v_0 \gamma_0 R_{n,0,i}\bigr),
\end{equation*}
so that $L_n(\mathbf{v}) = \sum_{i=1}^{n} W_{n,i}$.  Each $W_{n,i}$
has mean zero, and the variances are controlled by
\begin{align*}
  \mathrm{Var}\!\left(\sum_{i=1}^{n} W_{n,i}\right)
  &= \frac{1}{k_0}\,
     \mathrm{Var}\!\left(
       \sum_{i=1}^{n}
       \bigl(v_1 \gamma_1 R_{n,1,i} + v_0 \gamma_0 R_{n,0,i}\bigr)
     \right)
   = v_1^{2} \gamma_1^{2} a_{n,1}^{2}
    + v_0^{2} \gamma_0^{2} a_{n,0}^{2}
    + o(1),
\end{align*}
which is of order $\Theta(1)$.  Denote $T_n^{2} :=
\mathrm{Var}(\sum_{i=1}^{n} W_{n,i}) = \Theta(1)$.

By the elementary bound
$\mathbf{1}_{\{|W_{n,i}| > \epsilon T_n\}} \le W_{n,i}^{2}/(\epsilon^{2} T_n^{2})$,
for every $\epsilon > 0$ we have
\begin{align*}
  \frac{1}{T_n^{2}}
    \sum_{i=1}^{n} \mathbb{E}\!\left(W_{n,i}^{2}\,
      \mathbf{1}_{\{|W_{n,i}| > \epsilon T_n\}}\right)
  &\le \frac{1}{\epsilon^{2} T_n^{4}}
    \sum_{i=1}^{n} \mathbb{E}(W_{n,i}^{4}) \\
  &= \frac{1}{\epsilon^{2} T_n^{4}\, k_0^{2}}
    \sum_{i=1}^{n}
    \mathbb{E}\!\bigl[(v_1 \gamma_1 R_{n,1,i} + v_0 \gamma_0 R_{n,0,i})^{4}\bigr] \\
  &\le \frac{n \cdot O(\beta_n)}{\epsilon^{2} T_n^{4}\, k_0^{2}}
   = \frac{O(1)}{\epsilon^{2}\, n \beta_n}
   \;\longrightarrow\; 0,
\end{align*}
since $k_0^{2}\, T_n^{4} = \Theta(k_0^{2}) = \Theta(n^{2} \beta_n^{2})$.  The
Lindeberg condition (Theorem~\ref{thm:lindeberg_clt}) is therefore
satisfied and
\begin{equation*}
  L_n(\mathbf{v})
  \xrightarrow{d}
  \mathcal{N}\!\bigl(0,\; v_1^{2}\, \gamma_1^{2} a_1^{2}
                  + v_0^{2}\, \gamma_0^{2} a_0^{2}\bigr).
\end{equation*}
By the Cram\'er--Wold theorem~\cite{cramer1936some}, this implies the joint convergence
\begin{equation*}
  \sqrt{k_0}\,
  \begin{pmatrix}
    \widehat{\gamma}_1^{F}(\beta_n, \alpha_n) - \gamma_1 \\[4pt]
    \widehat{\gamma}_0^{F}(\beta_n, \alpha_n) - \gamma_0
  \end{pmatrix}
  \xrightarrow{d}
  \mathcal{N}\!\left(
    \mathbf{0},\;
    \begin{pmatrix}
      \gamma_1^{2} a_1^{2} & 0 \\[4pt]
      0 & \gamma_0^{2} a_0^{2}
    \end{pmatrix}
  \right),
\end{equation*}
which is the claimed joint asymptotic normality.
\qedhere
\end{proof}

\subsection{Proof of Lemma~\ref{lem:extreme_qte_linearization}}

\begin{proof}
Denote $d_n := \alpha_n / \tau_n$, $r_n := \alpha_n / \beta_n$ and
$e_n := d_n / r_n = \beta_n / \tau_n$.  The condition
$\log e_n = o(k_0^{1/2})$ implies $\log e_n / \sqrt{k_0} \to 0$.
We abbreviate
$\widehat{q}_\alpha := \widehat{q}_j(1-\alpha_n)$,
$\widehat{q}_\beta := \widehat{q}_j(1-\beta_n)$,
$q_\alpha := U_j(n/k)$, $q_\beta := U_j(n/k_0)$ and
$q_\tau := q_j(1-\tau_n) = U_j(1/\tau_n)$.
For $x \in \mathbb{R}$ define the weight
\begin{equation*}
  W(x) := \frac{d_n^{x} - 1}{r_n^{x} - 1}.
\end{equation*}
By the difference extrapolation
formula~\eqref{eq:extremal_qte_estimator_add},
\begin{equation*}
  \widehat{Q}_j(1-\tau_n)
  = \widehat{q}_\alpha
    + \bigl(\widehat{q}_\beta - \widehat{q}_\alpha\bigr)\,
      W\bigl(\widehat{\gamma}_j^{F}(\beta_n,\alpha_n)\bigr).
\end{equation*}

Since $U_j$ is regularly varying of index $\gamma_j$
(see Equation~\eqref{eq:equiv_quantile}) and
$1/\tau_n = d_n \cdot n/k$, $1/\beta_n = r_n \cdot n/k$, we have
\begin{equation*}
  q_\tau - q_\alpha
  = q_\alpha\, \bigl(d_n^{\gamma_j} - 1\bigr)\bigl(1 + o(1)\bigr),
  \qquad
  q_\beta - q_\alpha
  = q_\alpha\, \bigl(r_n^{\gamma_j} - 1\bigr)\bigl(1 + o(1)\bigr),
\end{equation*}
whence, dividing the two displays and using
$q_\alpha / q_\tau \to 0$ (again by regular variation),
\begin{equation*}
  \frac{q_\beta - q_\alpha}{q_\tau}\, W(\gamma_j)
  = \frac{q_\tau - q_\alpha}{q_\tau}\,\bigl(1 + o(1)\bigr)
  = 1 + o(1).
\end{equation*}

Adding and subtracting $q_\alpha + (q_\beta - q_\alpha)\, W(\gamma_j)$,
\begin{align*}
  \frac{\widehat{Q}_j(1-\tau_n)}{q_j(1-\tau_n)} - 1
  &= \underbrace{\frac{q_\beta - q_\alpha}{q_\tau}\,
       \bigl(W(\widehat{\gamma}_j^{F}) - W(\gamma_j)\bigr)}_{\text{Term II}}
   - \underbrace{\frac{q_\tau - q_\alpha - (q_\beta - q_\alpha)\, W(\gamma_j)}
                      {q_\tau}}_{\text{Term III}}
   \nonumber \\
  &\quad + \underbrace{\frac{(\widehat{q}_\alpha - q_\alpha)\,
       \bigl(1 - W(\widehat{\gamma}_j^{F})\bigr)
       + (\widehat{q}_\beta - q_\beta)\, W(\widehat{\gamma}_j^{F})}
      {q_\tau}}_{\text{Term I}},
\end{align*}
and we analyse the three terms after multiplying by
$\sqrt{k_0}/\log e_n$.

\paragraph{Preliminary: ratio of weights.}
By Theorem~\ref{thm:causal_fraga_clt},
$\widehat{\gamma}_j^{F} - \gamma_j = O_p(k_0^{-1/2})$, and
$\log e_n = o(k_0^{1/2})$, so that
$(\widehat{\gamma}_j^{F} - \gamma_j)\log e_n = o_p(1)$.  Moreover
\begin{equation*}
  \frac{W(\widehat{\gamma}_j^{F})}{W(\gamma_j)}
  = \frac{d_n^{\widehat{\gamma}_j^{F}} - 1}{d_n^{\gamma_j} - 1}
    \cdot \frac{r_n^{\gamma_j} - 1}{r_n^{\widehat{\gamma}_j^{F}} - 1}
  = e_n^{\widehat{\gamma}_j^{F} - \gamma_j}
    \cdot \frac{1 - d_n^{-\widehat{\gamma}_j^{F}}}{1 - d_n^{-\gamma_j}}
    \cdot \frac{1 - r_n^{-\gamma_j}}{1 - r_n^{-\widehat{\gamma}_j^{F}}},
\end{equation*}
and since $\widehat{\gamma}_j^{F} \to \gamma_j > 0$ in probability, we
have $\widehat{\gamma}_j^{F} \ge \gamma_j/2$ eventually with
probability tending to one, so the last two factors converge to $1$ in
probability.  Hence
\begin{equation*}
  W(\widehat{\gamma}_j^{F})
  = W(\gamma_j)\, e^{(\widehat{\gamma}_j^{F} - \gamma_j)\log e_n}\,
    \bigl(1 + o_p(1)\bigr)
  = O_p\bigl(e_n^{\gamma_j}\bigr),
\end{equation*}
where the second identity uses
$W(\gamma_j) = e_n^{\gamma_j}\bigl(1 + o(1)\bigr)$ and
$e^{(\widehat{\gamma}_j^{F} - \gamma_j)\log e_n} = 1 + o_p(1)$.

\paragraph{Step 1: Term I --- Quantile estimation error.}
By Lemma G.1 and Theorem G.1 of Deuber et al.~\cite{deuber2024estimation},
\begin{equation*}
  \widehat{q}_\alpha - q_\alpha = q_\alpha\, O_p(k^{-1/2}),
  \qquad
  \widehat{q}_\beta - q_\beta = q_\beta\, O_p(k_0^{-1/2}).
\end{equation*}
For the first summand of Term~I,
$q_\alpha/q_\tau = d_n^{-\gamma_j}\bigl(1 + o(1)\bigr) \to 0$ and
$1 - W(\widehat{\gamma}_j^{F}) = O_p\bigl(e_n^{\gamma_j}\bigr)$, so that
\begin{equation*}
  \frac{q_\alpha}{q_\tau}\,\bigl(1 - W(\widehat{\gamma}_j^{F})\bigr)
  = d_n^{-\gamma_j} e_n^{\gamma_j}\, O_p(1)
  = \bigl(\beta_n/\alpha_n\bigr)^{\gamma_j}\, O_p(1)
  = o_p(1),
\end{equation*}
and the first summand equals
$O_p(k^{-1/2}) \cdot o_p(1)$; after multiplying by
$\sqrt{k_0}/\log e_n$ it is
$\frac{\sqrt{k_0/k}}{\log e_n}\, O_p(1)\, o_p(1) = o_p(1)$,
since $k_0/k = \beta_n/\alpha_n \to 0$.
For the second summand, $q_\beta\, W(\widehat{\gamma}_j^{F})/q_\tau
= e_n^{-\gamma_j}\bigl(1+o(1)\bigr)\, O_p\bigl(e_n^{\gamma_j}\bigr)
= O_p(1)$, so that
\begin{equation*}
  \frac{(\widehat{q}_\beta - q_\beta)\, W(\widehat{\gamma}_j^{F})}{q_\tau}
  = O_p\bigl(k_0^{-1/2}\bigr),
\end{equation*}
and multiplied by $\sqrt{k_0}/\log e_n$ it is
$O_p(1/\log e_n) = o_p(1)$, because $\log e_n \to +\infty$
(here $e_n = (n\beta_n)/(n\tau_n)$ with $n\beta_n = k_0 \to +\infty$
and $n\tau_n \to a \geq 0$).  Hence Term~I $= o_p(1)$.

\paragraph{Step 2: Term II --- Core term.}
Using the preliminary ratio,
\begin{equation*}
  \text{Term II}
  = \frac{q_\beta - q_\alpha}{q_\tau}\, W(\gamma_j)\,
    \frac{e^{(\widehat{\gamma}_j^{F} - \gamma_j)\log e_n} - 1}{\log e_n}\,
    \bigl(1 + o_p(1)\bigr).
\end{equation*}
By the integral representation
\begin{equation*}
  \frac{e^{(\widehat{\gamma}_j^{F} - \gamma_j)\log e_n} - 1}{\log e_n}
  = (\widehat{\gamma}_j^{F} - \gamma_j)\,
    \cdot \frac{1}{\log e_n}\,
      \int_{1}^{e_n}
        e^{(\widehat{\gamma}_j^{F} - \gamma_j)\log s}\,\frac{ds}{s},
\end{equation*}
and the sandwich
$e^{-|u|} \le \frac{1}{\log e_n}\int_1^{e_n}
e^{u \log s}\,\frac{ds}{s} \le e^{|u|}$ with
$u := (\widehat{\gamma}_j^{F} - \gamma_j)\log e_n = o_p(1)$, the integral
factor is $1 + o_p(1)$.  Consequently, using
$\frac{q_\beta - q_\alpha}{q_\tau}\, W(\gamma_j) = 1 + o(1)$,
\begin{equation*}
  \text{Term II}
  = \sqrt{k_0}\,(\widehat{\gamma}_j^{F} - \gamma_j) + o_p(1).
\end{equation*}

\paragraph{Step 3: Term III --- Second-order bias term.}
By the second-order regular variation of $U_j$
(Assumption~\ref{asm:second_order},
Definition~\ref{def:second_order_regular_variation};
see de Haan and Ferreira~\cite{de2006extreme}, Theorem~2.3.9), with the convention
$\frac{x^{\rho} - 1}{\rho} := \log x$ when $\rho = 0$,
\begin{equation*}
  U_j(tx) - U_j(t)
  = U_j(t)\Bigl( x^{\gamma_j} - 1
      + A_j(t)\, x^{\gamma_j}\,
        \frac{x^{\rho_j} - 1}{\rho_j}\,\bigl(1 + o(1)\bigr) \Bigr),
\end{equation*}
uniformly for $x$ in the relevant range as $t \to +\infty$.  Taking
$t = n/k$ and $x = d_n$ or $x = r_n$,
\begin{align*}
  q_\tau - q_\alpha
  &= q_\alpha\Bigl( d_n^{\gamma_j} - 1
      + A_j(n/k)\, d_n^{\gamma_j}\,
        \frac{d_n^{\rho_j} - 1}{\rho_j}\,\bigl(1 + o(1)\bigr) \Bigr), \\
  q_\beta - q_\alpha
  &= q_\alpha\Bigl( r_n^{\gamma_j} - 1
      + A_j(n/k)\, r_n^{\gamma_j}\,
        \frac{r_n^{\rho_j} - 1}{\rho_j}\,\bigl(1 + o(1)\bigr) \Bigr).
\end{align*}
Subtracting $W(\gamma_j)$ times the second display from the first and
using $W(\gamma_j) = \frac{d_n^{\gamma_j} - 1}{r_n^{\gamma_j} - 1}$ and
$\frac{d_n^{\gamma_j} - 1}{r_n^{\gamma_j} - 1}\, r_n^{\gamma_j}
 = d_n^{\gamma_j}\,\bigl(1 + o(1)\bigr)$,
\begin{equation*}
  q_\tau - q_\alpha - (q_\beta - q_\alpha)\, W(\gamma_j)
  = q_\alpha\, A_j(n/k)\, d_n^{\gamma_j}\,
    \frac{d_n^{\rho_j} - r_n^{\rho_j}}{\rho_j}\,\bigl(1 + o(1)\bigr).
\end{equation*}
Since $\rho_j \le 0$ and $d_n > r_n$,
$\frac{d_n^{\rho_j} - r_n^{\rho_j}}{\rho_j}
 = \int_{r_n}^{d_n} x^{\rho_j - 1}\,dx \le \log(d_n/r_n)
 = \log e_n$, and $q_\tau = q_\alpha\, d_n^{\gamma_j}\,(1 + o(1))$;
hence
\begin{equation*}
  |\text{Term III}|
  \le C\, \frac{\sqrt{k_0}}{\log e_n}\,
     A_j(n/k)\, \log e_n\, \bigl(1 + o(1)\bigr)
  = C\, \sqrt{k_0}\, A_j(n/k)\, \bigl(1 + o(1)\bigr)
  = o(1),
\end{equation*}
where the last step uses
$\sqrt{k_0}\, A_j(n/k)
 = \sqrt{k_0/k} \cdot \sqrt{k}\, A_j(n/k)
 = o(1) \cdot O(1) = o(1)$, by $k_0/k = \beta_n/\alpha_n \to 0$ and
$\sqrt{k}\, A_j(n/k) \to \lambda_j \in \mathbb{R}$ by assumption.

\paragraph{Step 4: Synthesis.}
Combining Steps~1--3,
\begin{equation*}
  \frac{\sqrt{k_0}}{\log e_n}
    \left(\frac{\widehat{Q}_j(1-\tau_n)}{q_j(1-\tau_n)} - 1\right)
  = \sqrt{k_0}\,(\widehat{\gamma}_j^{F} - \gamma_j) + o_p(1),
\end{equation*}
which is the claim, since $\log e_n = \log(\beta_n/\tau_n)$.
\qedhere
\end{proof}

\subsection{Proof of Theorem~\ref{thm:extremal_qte_asymptotic}}

\begin{proof}
Recall the notation $e_n := \beta_n / \tau_n$ from the proof of
Lemma~\ref{lem:extreme_qte_linearization}.  By that lemma, for each
$j \in \{0, 1\}$, defining $S_j := \sqrt{k_0}\,(\widehat{\gamma}_j^{F} - \gamma_j)$,
we have
\begin{equation*}
  \frac{\sqrt{k_0}}{\log e_n}
  \left( \frac{\widehat{Q}_j(1-\tau_n)}{q_j(1-\tau_n)} - 1 \right)
  = S_j + o_p(1).
\end{equation*}
By Theorem~\ref{thm:causal_fraga_joint_clt}, $(S_1, S_0)$ are jointly
asymptotically normal with covariance matrix
$\mathrm{diag}(\gamma_1^2 a_1^2, \gamma_0^2 a_0^2)$ (asymptotically
independent).

Expanding $\widehat{\Delta}(1-\tau_n) - \Delta(1-\tau_n) =
(\widehat{Q}_1 - q_1) - (\widehat{Q}_0 - q_0)$ and applying
Lemma~\ref{lem:extreme_qte_linearization} to each arm yields
\begin{align*}
  \widehat{\Delta}(1-\tau_n) - \Delta(1-\tau_n)
  &= \frac{q_1 \log e_n}{\sqrt{k_0}}
     \cdot \frac{\sqrt{k_0}}{\log e_n}
     \left( \frac{\widehat{Q}_1}{q_1} - 1 \right)
     - \frac{q_0 \log e_n}{\sqrt{k_0}}
     \cdot \frac{\sqrt{k_0}}{\log e_n}
     \left( \frac{\widehat{Q}_0}{q_0} - 1 \right) \\
  &= \frac{q_1 \log e_n}{\sqrt{k_0}}\bigl(S_1 + o_p(1)\bigr)
     - \frac{q_0 \log e_n}{\sqrt{k_0}}\bigl(S_0 + o_p(1)\bigr).
\end{align*}

Introduce the auxiliary population normalising factor
\begin{equation*}
  \phi_n
  := \frac{\sqrt{k_0}}
         {\log e_n \cdot \max\bigl\{q_1(1-\tau_n),\, q_0(1-\tau_n)\bigr\}}.
\end{equation*}
Multiplying the previous display by $\phi_n$ gives
\begin{equation*}
  \phi_n\bigl(\widehat{\Delta} - \Delta\bigr)
  = \frac{q_1}{\max\{q_1, q_0\}}\,S_1
    - \frac{q_0}{\max\{q_1, q_0\}}\,S_0
    + o_p(1).
\end{equation*}

By Lemma~\ref{lem:extreme_qte_linearization}, $\widehat{Q}_j / q_j
\xrightarrow{p} 1$, so that
\begin{equation*}
  \frac{\max\{\widehat{Q}_1, \widehat{Q}_0\}}{\max\{q_1, q_0\}}
  = \frac{\max\bigl\{q_1(1 + o_p(1)),\, q_0(1 + o_p(1))\bigr\}}
       {\max\{q_1, q_0\}}
  = 1 + o_p(1).
\end{equation*}
Therefore
\begin{equation*}
  \widehat{\phi}_n
  = \phi_n \cdot \bigl(1 + o_p(1)\bigr)^{-1}
  = \phi_n\bigl(1 + o_p(1)\bigr).
\end{equation*}

Multiplying both sides by $\widehat{\phi}_n$ gives
\begin{align*}
  \widehat{\phi}_n\bigl(\widehat{\Delta} - \Delta\bigr)
  &= \bigl(1 + o_p(1)\bigr)
     \left( \frac{q_1}{\max\{q_1, q_0\}}\,S_1
          - \frac{q_0}{\max\{q_1, q_0\}}\,S_0
          + o_p(1) \right) \\
  &= \frac{q_1}{\max\{q_1, q_0\}}\,S_1
     - \frac{q_0}{\max\{q_1, q_0\}}\,S_0
     + o_p(1).
\end{align*}

Taking $n \to \infty$ and using $q_1 / q_0 \to \kappa$, we obtain
\begin{equation*}
  \frac{q_1}{\max\{q_1, q_0\}}
  \to \min\{1, \kappa\}, \qquad
  \frac{q_0}{\max\{q_1, q_0\}}
  \to \min\{1, 1/\kappa\}.
\end{equation*}

By Slutsky's theorem (Theorem~\ref{thm:slutsky}) and the asymptotic
independence of $S_1$ and $S_0$,
\begin{equation*}
  \widehat{\phi}_n\bigl(\widehat{\Delta}(1-\tau_n) - \Delta(1-\tau_n)\bigr)
  \xrightarrow{d}
  \mathcal{N}\!\bigl(0,\,
    \min\{1, \kappa\}^2 \gamma_1^2 a_1^2
    + \min\{1, 1/\kappa\}^2 \gamma_0^2 a_0^2\bigr).
\end{equation*}
This completes the proof.
\qedhere
\end{proof}

\subsection{Proof of Theorem~\ref{thm:sigma2_consistency}}
\label{sec:sigma2_consistency_proof}

\begin{proof}
We prove $\widehat{a}_{n,1}^{2} \xrightarrow{p} a_1^{2}$; the case
$j = 0$ follows by the symmetric substitution of $D$ with $1 - D$
and $\pi$ with $1 - \pi$.  Together with
$\widehat{\kappa} \xrightarrow{p} \kappa$ (Lemma~\ref{lem:extreme_qte_linearization}
and Assumption~\ref{asm:tail_comparability}) and
$\widehat{\gamma}_j^{F} \xrightarrow{p} \gamma_j$
(Theorem~\ref{thm:causal_fraga_consistency}), Slutsky's theorem
(Theorem~\ref{thm:slutsky}) then yields
\begin{equation*}
  \widehat{\sigma}^{2} \xrightarrow{p}
  \min\{1, \kappa\}^{2}\, \gamma_1^{2}\, a_1^{2}
  + \min\{1, 1/\kappa\}^{2}\, \gamma_0^{2}\, a_0^{2}
  = \sigma^{2}.
\end{equation*}

\paragraph{Step 1: Oracle reduction.}
Let $R_{n,1,i}$ denote the per-observation oracle quantity
introduced in the proof of
Lemma~\ref{thm:causal_fraga_equation} (under $j = 1$):
\begin{equation*}
  R_{n,1,i}
  = \frac{D_i}{\pi(X_i)}\,
    \mathbf{1}_{\{Z_i(1) > \beta_n^{-1}\}}\,
    \log(Z_i(1)\,\beta_n) - \beta_n.
\end{equation*}
The empirical counterpart $\widehat{R}_{n,1,i}$ defined in
Equation~\eqref{eq:an_estimator} uses the estimated propensity score
$\widehat{\pi}$ and the Fraga estimator $\widehat{\gamma}_1^{F}$ in
place of $\pi$ and $\gamma_1$.

\paragraph{Step 2: Oracle version $\to a_{n,1}^{2}$.}
Since $\mathbb{E}(R_{n,1,i}) = 0$,
$\mathbb{E}(R_{n,1,i}^{2}) = \mathrm{Var}(R_{n,1,i})
 = (k_0/n)\, a_{n,1}^{2} = \beta_n a_{n,1}^{2}$,
with $a_{n,1}^{2} = n\,\mathbb{E}(R_{n,1,i}^{2})/k_0 = O(1)$
by Lemma~\ref{lem:A.1} and the Lindeberg condition.  Because the
distribution of $R_{n,1,i}$ depends on $n$ through $\beta_n$, we
apply the weak law of large numbers for triangular arrays
(Theorem~\ref{thm:wlln_triangular}) with
$S_n = \sum_{i=1}^{n} R_{n,1,i}^{2}$,
$\mu_n = n\,\mathbb{E}(R_{n,1,i}^{2}) = k_0 a_{n,1}^{2}$
and $a_n = n$ (so that
$\mathrm{Var}(S_n) / a_n^{2} = O(\beta_n / n) \to 0$ by
Lemma~\ref{lem:A.1}).  This yields
$\dfrac{1}{n}\sum_{i=1}^{n} R_{n,1,i}^{2} \xrightarrow{p}
\beta_n a_{n,1}^{2}$, and multiplying by $n/k_0 = 1/\beta_n$ gives
\begin{equation*}
  \frac{1}{k_0}\sum_{i=1}^{n} R_{n,1,i}^{2}
  \xrightarrow{p} a_{n,1}^{2}.
\end{equation*}

\paragraph{Step 3: Uniform estimation error.}
We show that $\widehat{R}_{n,1,i} - R_{n,1,i} = o_p(1)$,
uniformly in $i$ on the tail event $\{Z_i(1) > \beta_n^{-1}\}$.  By
the second-order expansion
(Equation~\eqref{eq:second_order_exp}) of
Lemma~\ref{thm:causal_fraga_equation}, on
$\{Y_i > q_1(1-\beta_n)\}$,
\begin{equation*}
  \log\frac{Y_i - q_1(1-\alpha_n)}
            {q_1(1-\beta_n) - q_1(1-\alpha_n)}
  = \gamma_1 \log(Z_i(1)\,\beta_n) + r_{n,i},
\end{equation*}
where the remainder satisfies
$\max_i |r_{n,i}| = o_p(1)$ under the bias-decay conditions of
Theorem~\ref{thm:causal_fraga_clt}.  Combined with
$\widehat{\gamma}_1^{F} \xrightarrow{p} \gamma_1$
(Theorem~\ref{thm:causal_fraga_consistency}),
\begin{equation*}
  \frac{1}{\widehat{\gamma}_1^{F}}\,
    \log\frac{Y_i - \widehat{q}_1(1-\alpha_n)}
              {\widehat{q}_1(1-\beta_n) - \widehat{q}_1(1-\alpha_n)}
  = \log(Z_i(1)\,\beta_n) + o_p(1)
\end{equation*}
uniformly in $i$ on the tail event.  Substituting this into
$\widehat{R}_{n,1,i}$, the three sources of error combine to give
$\widehat{R}_{n,1,i} - R_{n,1,i} = o_p(1)$:
\emph{(i)}~the propensity-score replacement, controlled by
Assumption~\ref{asm:causal_identification} and
Lemma~G.3 of Deuber et al.~\cite{deuber2024estimation},
$\sup_{x \in \mathrm{supp}(X)} |1/\widehat{\pi}(x) - 1/\pi(x)| = o_p(1)$;
\emph{(ii)}~the indicator-function replacement, controlled by
Lemma~\ref{lem:A.3} via the consistency of
$\widehat{q}_1(1-\beta_n)$
(Zhang~\cite{zhang2018extremal}),
$\mathbf{1}_{\{Y_i > \widehat{q}_1(1-\beta_n)\}}
  - \mathbf{1}_{\{Z_i(1) > \beta_n^{-1}\}} = o_p(1)$.
The $\beta_n$ terms cancel in the subtraction.

\paragraph{Step 4: Propagation to $\widehat{a}_{n,1}^{2}$.}
By Step~3, $\sup_{i}|\widehat{R}_{n,1,i} - R_{n,1,i}| = o_p(1)$ over all
$i$ (on the tail event this is Step~3; on its complement the indicator
difference is handled by Lemma~\ref{lem:A.3} and the log factor is
$o_p(1)$ by the sandwich of the proof of
Lemma~\ref{thm:causal_fraga_equation}).  Moreover, since the logarithm
is non-negative on the tail event, by Lemma~\ref{lem:A.1} with $r = 1$,
\begin{equation*}
  \mathbb{E}\!\bigl[|R_{n,1,i}|\bigr]
  \le \mathbb{E}\!\left[\frac{D_i}{\pi(X_i)}\,
      \mathbf{1}_{\{Z_i(1) > \beta_n^{-1}\}}
      \log\!\bigl(Z_i(1)\,\beta_n\bigr)\right] + \beta_n
  = 2\beta_n,
\end{equation*}
so that
$\frac{1}{n}\sum_{i=1}^{n}|R_{n,1,i}| = O_p(\beta_n)$ (by Markov's
inequality).  Moreover, by the triangle inequality,
$\frac{1}{n}\sum_{i=1}^{n}|\widehat{R}_{n,1,i}|
 \le \frac{1}{n}\sum_{i=1}^{n}|R_{n,1,i}|
    + \sup_{i}|\widehat{R}_{n,1,i} - R_{n,1,i}|
 = O_p(\beta_n) + o_p(1) = O_p(1)$,
and likewise
$\frac{1}{n}\sum_{i=1}^{n}|\widehat{R}_{n,1,i} + R_{n,1,i}|
 = O_p(1)$.  Hence
\begin{equation*}
  \Bigl|\frac{1}{n}\sum_{i=1}^{n}\bigl(\widehat{R}_{n,1,i}^{2}
     - R_{n,1,i}^{2}\bigr)\Bigr|
  \le \sup_{i}\bigl|\widehat{R}_{n,1,i} - R_{n,1,i}\bigr|
     \cdot \frac{1}{n}\sum_{i=1}^{n}\bigl|\widehat{R}_{n,1,i} + R_{n,1,i}\bigr|
  = o_p(1) \cdot O_p(1)
  = o_p(1).
\end{equation*}
Consequently
\begin{equation*}
  \frac{1}{n}\sum_{i=1}^{n} \widehat{R}_{n,1,i}^{2}
  = \frac{1}{n}\sum_{i=1}^{n} R_{n,1,i}^{2}
    + o_p(1)
  \xrightarrow{p} \beta_n a_{n,1}^{2}.
\end{equation*}
Multiplying by $n/k_0 = 1/\beta_n$,
$\widehat{a}_{n,1}^{2} \xrightarrow{p} a_{n,1}^{2}
 \to a_1^{2}$ by Assumption~\ref{asm:an_convergence}.  By symmetry,
$\widehat{a}_{n,0}^{2} \xrightarrow{p} a_0^{2}$.

\paragraph{Step 5: Combining.}
Together with $\widehat{\kappa} \xrightarrow{p} \kappa$ and
$\widehat{\gamma}_j^{F} \xrightarrow{p} \gamma_j$, Slutsky's theorem
(Theorem~\ref{thm:slutsky}) gives $\widehat{\sigma}^{2}
\xrightarrow{p} \sigma^{2}$, completing the proof.
\end{proof}

\numberwithin{equation}{section}
\numberwithin{theorem}{section}
\numberwithin{definition}{section}
\numberwithin{lemma}{section}
\numberwithin{proposition}{section}
\numberwithin{corollary}{section}
\numberwithin{assumption}{section}

\section{Multiplicative Extrapolation}
\label{sec:mult_extrapolation}

This appendix collects the asymptotic theory for the \emph{conventional
multiplicative extrapolation estimator} introduced in
Section~\ref{sec:extremal_qte}, namely
\begin{equation*}
  \widehat{Q}_j^{\mathrm{mult}}(1-\tau_n)
  = \widehat{q}_j(1-\alpha_n)
    \left( \frac{\alpha_n}{\tau_n} \right)^{\widehat{\gamma}_j^{F}(\beta_n, \alpha_n)},
  \qquad j \in \{0, 1\},
\end{equation*}
and the associated multiplicative extrapolation QTE estimator
\begin{equation*}
  \widehat{\Delta}^{\mathrm{mult}}(1-\tau_n)
  = \widehat{Q}_1^{\mathrm{mult}}(1-\tau_n)
    - \widehat{Q}_0^{\mathrm{mult}}(1-\tau_n).
\end{equation*}
As explained in Section~\ref{sec:extremal_qte}, this estimator is not
location-invariant; its asymptotic theory, derived below, serves as the
baseline against which the location-invariant difference extrapolation
estimator of the main text is compared.

The following lemma is the analogue of
Lemma~\ref{lem:extreme_qte_linearization} for the multiplicative
extrapolation estimator, with the extrapolation performed from the
primary threshold $\alpha_n$ rather than the auxiliary threshold
$\beta_n$.

\begin{lemma}
\label{lem:mult_linearization}
Let Assumptions~\ref{asm:causal_identification}--\ref{asm:second_order}
hold. As $n \to +\infty$, suppose $\alpha_n \to 0$ with
$k = n\alpha_n \to +\infty$, $\beta_n \to 0$ with
$k_0 = n\beta_n \to +\infty$, and $\beta_n / \alpha_n \to 0$.
If for each $j \in \{0, 1\}$,
$\sqrt{k}\,A_j(n/k) \to \lambda_j \in \mathbb{R}$,
$(k_0 / k)^{\gamma_j} = o(k_0^{-1/2})$,
$A_{n,j} = o_p(k_0^{-1/2})$ and $\log(\alpha_n / \tau_n) = o(k_0^{1/2})$, then
\begin{equation*}
  \frac{\sqrt{k_0}}{\log(\alpha_n / \tau_n)}
  \left( \frac{\widehat{Q}_j^{\mathrm{mult}}(1-\tau_n)}{q_j(1-\tau_n)} - 1 \right)
  = \sqrt{k_0}\,
    \bigl( \widehat{\gamma}_j^{F}(\beta_n, \alpha_n) - \gamma_j \bigr)
    + o_p(1).
\end{equation*}
In particular, $\widehat{Q}_j^{\mathrm{mult}}(1-\tau_n) / q_j(1-\tau_n) \xrightarrow{p} 1$.
\end{lemma}

\begin{proof}
Denote $d_n := \alpha_n / \tau_n$. The condition
$\log d_n = o(k_0^{1/2})$ implies $\log d_n / \sqrt{k_0} \to 0$.
Using the extrapolation formula
\eqref{eq:extremal_qte_estimator} together with
$q_j(1-\alpha_n) = U_j(n/k)$ and
$q_j(1-\tau_n) = U_j(1 / \tau_n)$, we write
\begin{equation*}
  \frac{\widehat{Q}_j^{\mathrm{mult}}(1-\tau_n)}{q_j(1-\tau_n)}
  = \frac{\widehat{q}_j(1-\alpha_n)}{U_j(n/k)}
    \cdot \frac{U_j(n/k)}{U_j(1 / \tau_n)}
    \cdot d_n^{\widehat{\gamma}_j^{F}}.
\end{equation*}

Note that $U_j(1 / \tau_n) = U_j(d_n \cdot n/k)$ and by
Equation~\eqref{eq:equiv_quantile},
$U_j(d_n \cdot n/k) / U_j(n/k) \sim d_n^{\gamma_j}$.
Define
\begin{equation*}
  R := \frac{U_j(n/k)\, d_n^{\gamma_j}}{U_j(1 / \tau_n)},
  \qquad
  \delta := \frac{\widehat{q}_j(1-\alpha_n)}{U_j(n/k)} - 1,
\end{equation*}
and note that $R \xrightarrow{p} 1$ and, by
Theorem~\ref{thm:causal_fraga_consistency},
$\widehat{\gamma}_j^{F} \xrightarrow{p} \gamma_j$.
Then
\begin{align*}
  \frac{\widehat{Q}_j^{\mathrm{mult}}(1-\tau_n)}{q_j(1-\tau_n)} - 1
  &= (1 + \delta)\, R\, d_n^{\widehat{\gamma}_j^{F} - \gamma_j} - 1 \\
  &= \underbrace{(1 + \delta)\, R\, \bigl(d_n^{\widehat{\gamma}_j^{F} - \gamma_j} - 1\bigr)}_{\text{Term II}}
   + \underbrace{\delta\, R\, d_n^{\widehat{\gamma}_j^{F} - \gamma_j}}_{\text{Term I}}
   + \underbrace{(R - 1)}_{\text{Term III}}.
\end{align*}
Multiplying by $\sqrt{k_0}/\log d_n$ and using
$R = 1 + o_p(1)$ and $d_n^{\widehat{\gamma}_j^{F} - \gamma_j} = 1 + o_p(1)$
(the latter is shown in Step~1 below), the three terms above become
\begin{align*}
  \frac{\sqrt{k_0}}{\log d_n}
  \left( \frac{\widehat{Q}_j^{\mathrm{mult}}(1-\tau_n)}{q_j(1-\tau_n)} - 1 \right)
  &= \underbrace{\bigl(1 + o(1)\bigr)\, d_n^{\widehat{\gamma}_j^{F} - \gamma_j}
     \cdot \frac{\sqrt{k_0}}{\log d_n}
       \left(\frac{\widehat{q}_j(1-\alpha_n)}{U_j(n/k)} - 1\right)}_{\text{Term I}}
     \\
  &\quad + \underbrace{\bigl(1 + o(1)\bigr)
     \cdot \frac{\sqrt{k_0}}{\log d_n}
     \bigl(d_n^{\widehat{\gamma}_j^{F} - \gamma_j} - 1\bigr)}_{\text{Term II}}
     \\
  &\quad + \underbrace{\frac{\sqrt{k_0}}{\log d_n}(R - 1)}_{\text{Term III}}.
\end{align*}
We analyse each term in turn.

\paragraph{Step 1: Term I --- Intermediate quantile estimation error.}
By Lemma G.1 and Theorem G.1 of Deuber et al.~\cite{deuber2024estimation}, we have
\begin{equation*}
  \sqrt{k}\left(\frac{\widehat{q}_j(1-\alpha_n)}{U_j(n/k)} - 1\right)
  = O_p(1).
\end{equation*}
From Theorem~\ref{thm:causal_fraga_clt},
$\widehat{\gamma}_j^{F} - \gamma_j = O_p(k_0^{-1/2})$, which together
with $\log d_n = o(k_0^{1/2})$ gives
\begin{equation*}
  (\widehat{\gamma}_j^{F} - \gamma_j)\log d_n
  = O_p(k_0^{-1/2}) \cdot o(k_0^{1/2}) = o_p(1),
\end{equation*}
hence $d_n^{\widehat{\gamma}_j^{F} - \gamma_j}
  = \exp\bigl((\widehat{\gamma}_j^{F} - \gamma_j)\log d_n\bigr)
  = 1 + o_p(1)$.
Therefore
\begin{align*}
  \text{Term I}
  &= \bigl(1 + o_p(1)\bigr)
     \cdot \frac{\sqrt{k_0}}{\log d_n}
     \cdot O_p(k^{-1/2}) \\
  &= O_p\!\Bigl(\frac{\sqrt{k_0/k}}{\log d_n}\Bigr) = o_p(1),
\end{align*}
where the last step uses $\log d_n \to +\infty$ (which holds since
$\tau_n < \alpha_n$, $\tau_n \to 0$ and $n \tau_n \to a \geq 0$, so that
$d_n \sim k / (n \tau_n) \to +\infty$) and
$k_0 / k = \beta_n / \alpha_n \to 0$.

\paragraph{Step 2: Term II --- Core term.}
Using the integral representation
\begin{equation*}
  \frac{d_n^{\widehat{\gamma}_j^{F} - \gamma_j} - 1}{\log d_n}
  = (\widehat{\gamma}_j^{F} - \gamma_j)
    \cdot \frac{1}{\log d_n}
    \int_{1}^{d_n}
      e^{(\widehat{\gamma}_j^{F} - \gamma_j)\log s}
    \,\frac{ds}{s},
\end{equation*}
we obtain
\begin{equation*}
  \text{Term II}
  = (1+o(1)) \sqrt{k_0}\,(\widehat{\gamma}_j^{F} - \gamma_j)
    \cdot \frac{1}{\log d_n}
    \int_{1}^{d_n}
      e^{(\widehat{\gamma}_j^{F} - \gamma_j)\log s}
    \,\frac{ds}{s}.
\end{equation*}

For all $s \in [1, d_n]$, we have the sandwich
\begin{equation*}
  e^{-|\widehat{\gamma}_j^{F} - \gamma_j| \log d_n}
  \leq \frac{1}{\log d_n}
    \int_{1}^{d_n}
      e^{(\widehat{\gamma}_j^{F} - \gamma_j)\log s}
    \,\frac{ds}{s}
  \leq e^{|\widehat{\gamma}_j^{F} - \gamma_j| \log d_n}.
\end{equation*}
By Theorem~\ref{thm:causal_fraga_clt},
$\sqrt{k_0}\,(\widehat{\gamma}_j^{F} - \gamma_j) = O_p(1)$, and
$\log d_n / \sqrt{k_0} = o(1)$, so that
\begin{equation*}
  |\widehat{\gamma}_j^{F} - \gamma_j| \log d_n
  = \bigl|\sqrt{k_0}\,(\widehat{\gamma}_j^{F} - \gamma_j)\bigr|
    \cdot \frac{\log d_n}{\sqrt{k_0}}
  = O_p(1) \cdot o(1) = o_p(1).
\end{equation*}
The sandwich argument therefore yields
\begin{equation*}
  \frac{1}{\log d_n}
    \int_{1}^{d_n}
      e^{(\widehat{\gamma}_j^{F} - \gamma_j)\log s}
    \,\frac{ds}{s}
  = 1 + o_p(1),
\end{equation*}
and consequently
\begin{equation*}
  \text{Term II}
  = (1+o(1)) \sqrt{k_0}\,(\widehat{\gamma}_j^{F} - \gamma_j)(1 + o_p(1))
  = \sqrt{k_0}\,(\widehat{\gamma}_j^{F} - \gamma_j) + o_p(1).
\end{equation*}

\paragraph{Step 3: Term III --- Second-order bias term.}
By Theorem~2.3.9 of de Haan and Ferreira~\cite{de2006extreme}, we have
(with the convention $\frac{d_n^{\rho_j} - 1}{\rho_j} := \log d_n$
when $\rho_j = 0$, understood via continuous extension)
\begin{equation*}
  \frac{U_j(1/\tau_n)\, d_n^{-\gamma_j} / U_j(n/k) - 1}
       {A_j(n/k)}
  = \frac{d_n^{\rho_j} - 1}{\rho_j}\bigl(1 + o(1)\bigr)
  = -\frac{1}{\rho_j}\bigl(1 + o(1)\bigr), \quad \rho_j < 0,
\end{equation*}
where the last equality holds in the limit $d_n \to +\infty$ for
$\rho_j < 0$; for $\rho_j = 0$ the left-hand side behaves as
$\log d_n\,(1 + o(1))$, in which case
Term~III $= -\sqrt{k_0}\,A_j(n/k)(1 + o(1)) = o(1)$ since
$\sqrt{k_0}\,A_j(n/k) \to 0$ (shown below), so the conclusion is
unchanged.
Hence, for $\rho_j < 0$, $1/R - 1 = -\frac{A_j(n/k)}{\rho_j}(1 + o(1))$
(recall that $1/R - 1 = A_j(n/k)\frac{d_n^{\rho_j}-1}{\rho_j}(1+o(1))$
in general). Substituting into Term~III yields
\begin{equation*}
  \text{Term III}
  = -\frac{\sqrt{k_0}}{\log d_n}
    \cdot \left(-\frac{A_j(n/k)}{\rho_j}(1 + o(1))\right)
    \cdot (1 + o(1))
  = \frac{\sqrt{k_0}\,A_j(n/k)}{\rho_j\,\log d_n}(1 + o(1)).
\end{equation*}
Observe that
\begin{equation*}
  \sqrt{k_0}\,A_j(n/k)
  = \sqrt{k_0/k} \cdot \sqrt{k}\,A_j(n/k)
  = o(1) \cdot O(1) = o(1),
\end{equation*}
since $k_0/k = \beta_n/\alpha_n \to 0$ and
$\sqrt{k}\,A_j(n/k) \to \lambda_j \in \mathbb{R}$ by assumption.
Therefore
\begin{equation*}
  \text{Term III}
  = \frac{o(1)}{\rho_j\,\log d_n}(1 + o(1))
  = o(1) \cdot O(1 / \log d_n)
  = o_p(1).
\end{equation*}

\paragraph{Step 4: Synthesis.}
Combining Steps~1--3,
\begin{equation*}
  \frac{\sqrt{k_0}}{\log d_n}
    \left(\frac{\widehat{Q}_j^{\mathrm{mult}}(1-\tau_n)}{q_j(1-\tau_n)} - 1\right)
  = \sqrt{k_0}\,(\widehat{\gamma}_j^{F} - \gamma_j) + o_p(1),
\end{equation*}
which is the desired statement.
\qedhere
\end{proof}

The following theorem is the analogue of
Theorem~\ref{thm:extremal_qte_asymptotic} for
$\widehat{\Delta}^{\mathrm{mult}}(1-\tau_n)$, with the normalising
factor
\begin{equation*}
  \widehat{\phi}_n^{\mathrm{mult}}
  := \frac{\sqrt{k_0}}
         {\log(\alpha_n/\tau_n) \cdot
          \max\bigl\{\widehat{Q}_1^{\mathrm{mult}}(1-\tau_n),\,
                 \widehat{Q}_0^{\mathrm{mult}}(1-\tau_n)\bigr\}}.
\end{equation*}

\begin{theorem}
\label{thm:mult_qte_asymptotic}
Let Assumptions~\ref{asm:causal_identification}--\ref{asm:tail_comparability}
hold. As $n \to +\infty$, suppose $\alpha_n \to 0$ with
$k = n\alpha_n \to +\infty$, $\beta_n \to 0$ with
$k_0 = n\beta_n \to +\infty$, and $\beta_n / \alpha_n \to 0$.
If for each $j \in \{0, 1\}$,
$\sqrt{k}\,A_j(n/k) \to \lambda_j \in \mathbb{R}$,
$(k_0 / k)^{\gamma_j} = o(k_0^{-1/2})$,
$A_{n,j} = o_p(k_0^{-1/2})$ and $\log(\alpha_n / \tau_n) = o(k_0^{1/2})$, then
\begin{equation*}
  \widehat{\phi}_n^{\mathrm{mult}}
  \bigl(\widehat{\Delta}^{\mathrm{mult}}(1-\tau_n) - \Delta(1-\tau_n)\bigr)
  \xrightarrow{d} \mathcal{N}(0, \sigma^2),
\end{equation*}
where
\begin{equation*}
  \sigma^2
  = \min\{1, \kappa\}^2 \gamma_1^2 a_1^2
    + \min\{1, 1/\kappa\}^2 \gamma_0^2 a_0^2.
\end{equation*}
\end{theorem}

\begin{proof}
Recall the notation $d_n := \alpha_n / \tau_n$ from the proof of
Lemma~\ref{lem:mult_linearization}.  By that lemma, for each
$j \in \{0, 1\}$, defining
$S_j := \sqrt{k_0}\,(\widehat{\gamma}_j^{F} - \gamma_j)$,
we have
\begin{equation*}
  \frac{\sqrt{k_0}}{\log d_n}
  \left( \frac{\widehat{Q}_j^{\mathrm{mult}}(1-\tau_n)}{q_j(1-\tau_n)} - 1 \right)
  = S_j + o_p(1).
\end{equation*}
By Theorem~\ref{thm:causal_fraga_joint_clt}, $(S_1, S_0)$ are jointly
asymptotically normal with covariance matrix
$\mathrm{diag}(\gamma_1^2 a_1^2, \gamma_0^2 a_0^2)$ (asymptotically
independent).

Expanding $\widehat{\Delta}^{\mathrm{mult}}(1-\tau_n) - \Delta(1-\tau_n) =
(\widehat{Q}_1^{\mathrm{mult}} - q_1) - (\widehat{Q}_0^{\mathrm{mult}} - q_0)$
and applying Lemma~\ref{lem:mult_linearization} to each arm yields
\begin{align*}
  \widehat{\Delta}^{\mathrm{mult}}(1-\tau_n) - \Delta(1-\tau_n)
  &= \frac{q_1 \log d_n}{\sqrt{k_0}}
     \cdot \frac{\sqrt{k_0}}{\log d_n}
     \left( \frac{\widehat{Q}_1^{\mathrm{mult}}}{q_1} - 1 \right)
     - \frac{q_0 \log d_n}{\sqrt{k_0}}
     \cdot \frac{\sqrt{k_0}}{\log d_n}
     \left( \frac{\widehat{Q}_0^{\mathrm{mult}}}{q_0} - 1 \right) \\
  &= \frac{q_1 \log d_n}{\sqrt{k_0}}\bigl(S_1 + o_p(1)\bigr)
     - \frac{q_0 \log d_n}{\sqrt{k_0}}\bigl(S_0 + o_p(1)\bigr).
\end{align*}

Introduce the auxiliary population normalising factor
\begin{equation*}
  \phi_n^{\mathrm{mult}}
  := \frac{\sqrt{k_0}}
         {\log d_n \cdot \max\bigl\{q_1(1-\tau_n),\, q_0(1-\tau_n)\bigr\}}.
\end{equation*}
Multiplying the previous display by $\phi_n^{\mathrm{mult}}$ gives
\begin{equation*}
  \phi_n^{\mathrm{mult}}\bigl(\widehat{\Delta}^{\mathrm{mult}} - \Delta\bigr)
  = \frac{q_1}{\max\{q_1, q_0\}}\,S_1
    - \frac{q_0}{\max\{q_1, q_0\}}\,S_0
    + o_p(1).
\end{equation*}

By Lemma~\ref{lem:mult_linearization}, $\widehat{Q}_j^{\mathrm{mult}} / q_j
\xrightarrow{p} 1$, so that
\begin{equation*}
  \frac{\max\{\widehat{Q}_1^{\mathrm{mult}}, \widehat{Q}_0^{\mathrm{mult}}\}}
       {\max\{q_1, q_0\}}
  = 1 + o_p(1).
\end{equation*}
Therefore
\begin{equation*}
  \widehat{\phi}_n^{\mathrm{mult}}
  = \phi_n^{\mathrm{mult}} \cdot \bigl(1 + o_p(1)\bigr)^{-1}
  = \phi_n^{\mathrm{mult}}\bigl(1 + o_p(1)\bigr).
\end{equation*}

Multiplying both sides by $\widehat{\phi}_n^{\mathrm{mult}}$ gives
\begin{align*}
  \widehat{\phi}_n^{\mathrm{mult}}\bigl(\widehat{\Delta}^{\mathrm{mult}} - \Delta\bigr)
  &= \bigl(1 + o_p(1)\bigr)
     \left( \frac{q_1}{\max\{q_1, q_0\}}\,S_1
          - \frac{q_0}{\max\{q_1, q_0\}}\,S_0
          + o_p(1) \right) \\
  &= \frac{q_1}{\max\{q_1, q_0\}}\,S_1
     - \frac{q_0}{\max\{q_1, q_0\}}\,S_0
     + o_p(1).
\end{align*}

Taking $n \to \infty$ and using $q_1 / q_0 \to \kappa$, we obtain
\begin{equation*}
  \frac{q_1}{\max\{q_1, q_0\}}
  \to \min\{1, \kappa\}, \qquad
  \frac{q_0}{\max\{q_1, q_0\}}
  \to \min\{1, 1/\kappa\}.
\end{equation*}

By Slutsky's theorem (Theorem~\ref{thm:slutsky}) and the asymptotic
independence of $S_1$ and $S_0$,
\begin{equation*}
  \widehat{\phi}_n^{\mathrm{mult}}
  \bigl(\widehat{\Delta}^{\mathrm{mult}}(1-\tau_n) - \Delta(1-\tau_n)\bigr)
  \xrightarrow{d}
  \mathcal{N}\!\bigl(0,\,
    \min\{1, \kappa\}^2 \gamma_1^2 a_1^2
    + \min\{1, 1/\kappa\}^2 \gamma_0^2 a_0^2\bigr).
\end{equation*}
This completes the proof.
\qedhere
\end{proof}

The variance estimator $\widehat{\sigma}^{2}$ defined in
Equation~\eqref{eq:sigma2_estimator} does not depend on the choice of
extrapolation method (it is built from the Fraga estimator
$\widehat{\gamma}_j^{F}$ and the auxiliary thresholds only), so its
consistency, Theorem~\ref{thm:sigma2_consistency}, applies verbatim to
the multiplicative extrapolation QTE estimator.  We therefore refer to
Section~\ref{sec:sigma2_consistency_proof} for the proof, and state here only the resulting
inference statement.

Combining Theorem~\ref{thm:mult_qte_asymptotic} with
Theorem~\ref{thm:sigma2_consistency} yields
\begin{equation*}
  \frac{\widehat{\phi}_n^{\mathrm{mult}}
         \bigl(\widehat{\Delta}^{\mathrm{mult}}(1-\tau_n)
              - \Delta(1-\tau_n)\bigr)}
       {\widehat{\sigma}}
  \xrightarrow{d} \mathcal{N}(0, 1).
\end{equation*}
Hence an asymptotic $(1-\alpha)$-level confidence interval for the
extremal quantile treatment effect $\Delta(1-\tau_n)$ based on the
multiplicative extrapolation is
\begin{equation*}
  \bigl[\,
    \widehat{\Delta}^{\mathrm{mult}}(1-\tau_n)
    - z_{1-\alpha/2}\,\frac{\widehat{\sigma}}{\widehat{\phi}_n^{\mathrm{mult}}},
    \;
    \widehat{\Delta}^{\mathrm{mult}}(1-\tau_n)
    + z_{1-\alpha/2}\,\frac{\widehat{\sigma}}{\widehat{\phi}_n^{\mathrm{mult}}}
  \,\bigr],
\end{equation*}
where $z_{1-\alpha/2}$ denotes the $(1-\alpha/2)$-quantile of the
standard normal distribution.

\numberwithin{equation}{section}
\numberwithin{theorem}{section}
\numberwithin{definition}{section}
\numberwithin{lemma}{section}
\numberwithin{proposition}{section}
\numberwithin{corollary}{section}
\numberwithin{assumption}{section}

\section{Additional Simulation Results}
\label{sec:additional_sim_results}

\subsection{Squared Errors of QTE estimations under the Location Shift}
\label{subsec:sqerr_location_shift}

Figures~\ref{fig:qte_shift_sqerr_H1}--\ref{fig:qte_shift_sqerr_H3}
report the squared errors of the QTE estimations for the three models.
For an extreme value index estimator
with a large variance, such as the causal Fraga estimator, the
difference extrapolation scheme is advantageous compared with the
multiplicative extrapolation scheme.  For an extreme value index
estimator whose variance is inherently small, such as the causal Hill
estimator, the difference extrapolation scheme is slightly inferior to
the multiplicative extrapolation scheme, because of the additional
variance introduced by estimating the intermediate quantile at the level
$\beta_n$.

\begin{figure}[t]
  \centering
  \includegraphics[width=\linewidth]{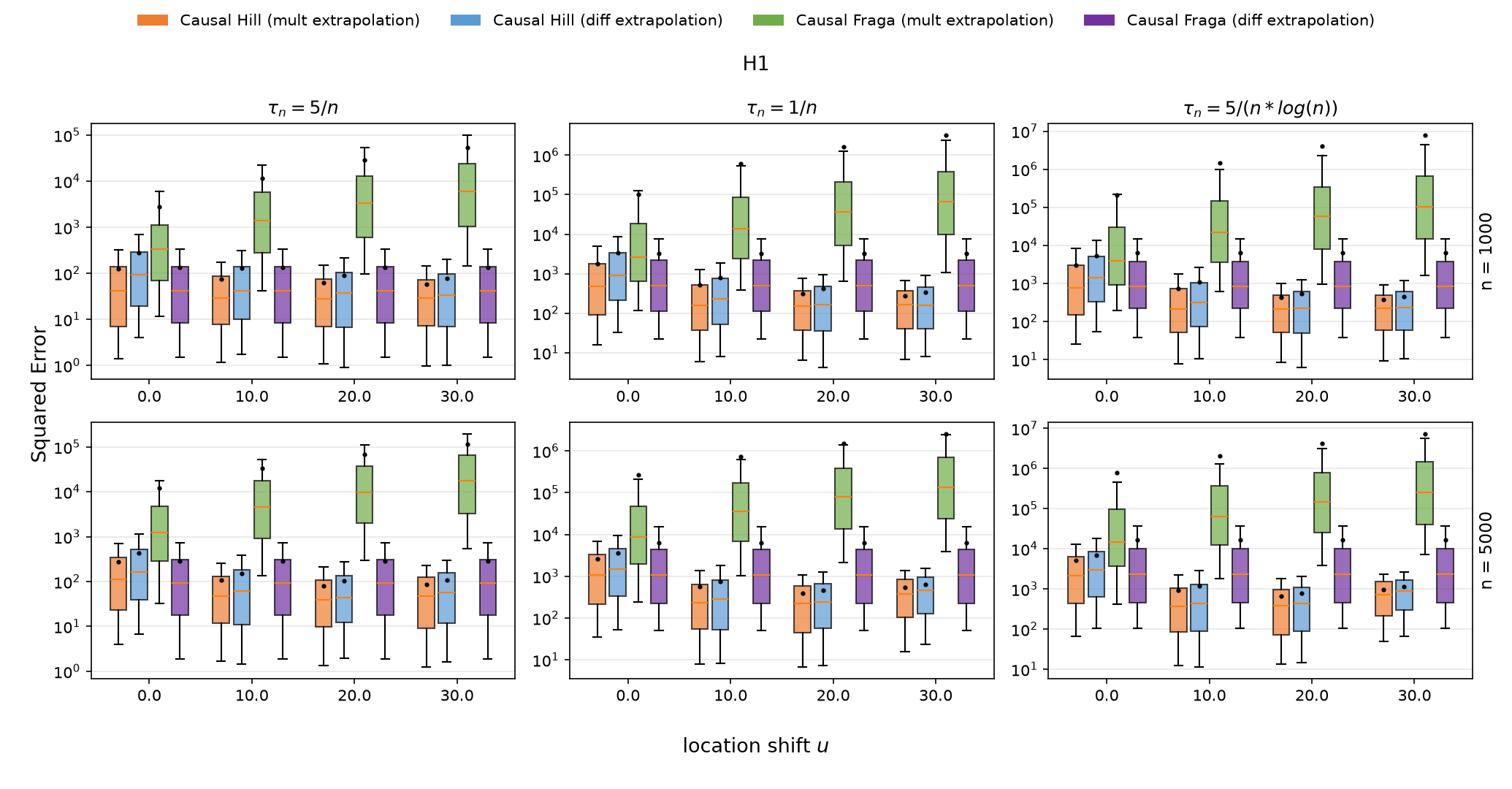}
  \caption{Squared errors of the quantile treatment effect
  estimations under the location shift $u$ for model $H_1$.  The boxplot
  whiskers correspond to the $0.1$ and $0.9$ quantiles, the orange line is
  the median, and the black dots are the means.}
  \label{fig:qte_shift_sqerr_H1}
\end{figure}

\begin{figure}[t]
  \centering
  \includegraphics[width=\linewidth]{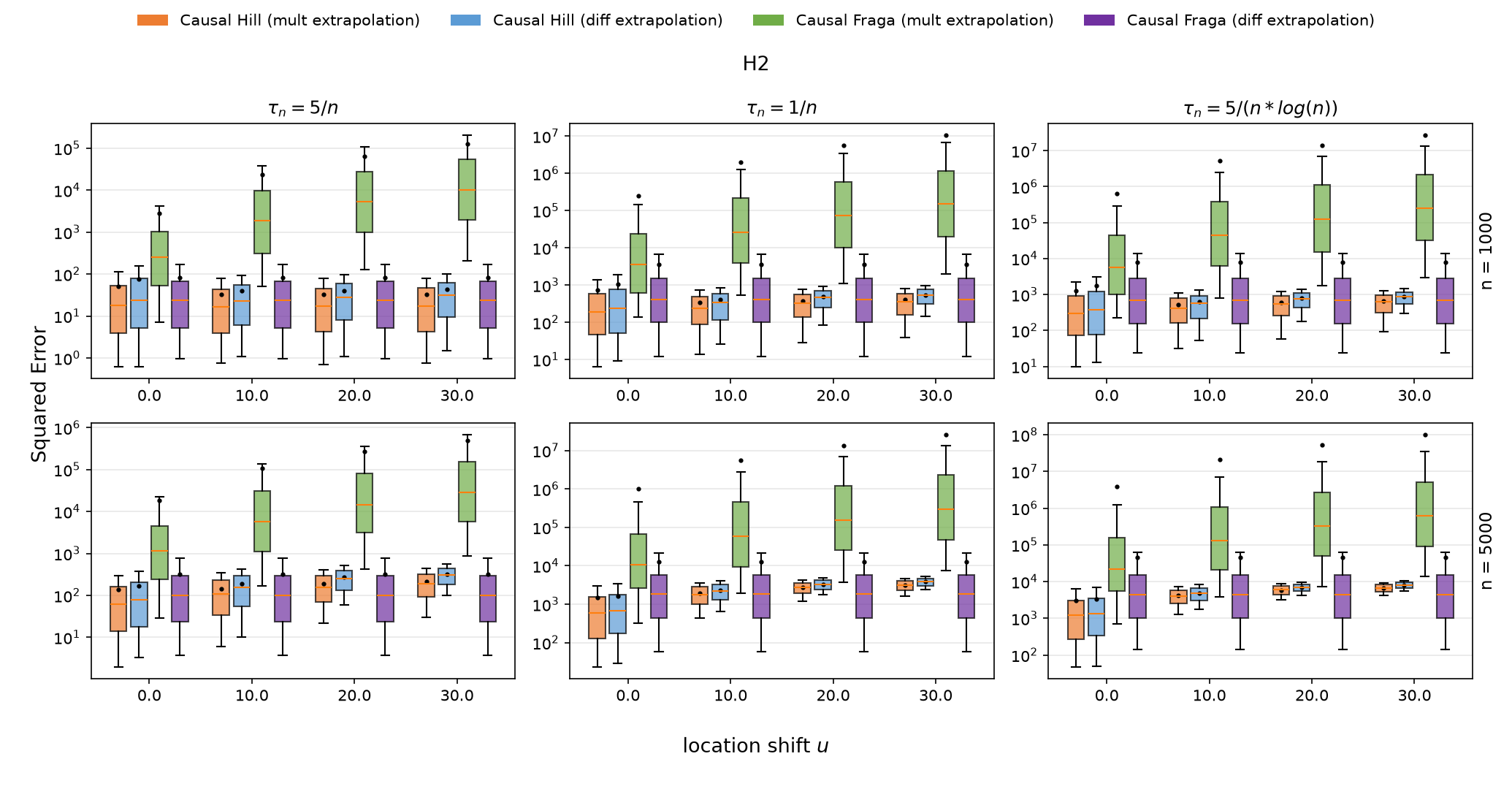}
  \caption{Squared errors of the quantile treatment effect
  estimations under the location shift $u$ for model $H_2$.  The boxplot
  whiskers correspond to the $0.1$ and $0.9$ quantiles, the orange line is
  the median, and the black dots are the means.}
  \label{fig:qte_shift_sqerr_H2}
\end{figure}

\begin{figure}[t]
  \centering
  \includegraphics[width=\linewidth]{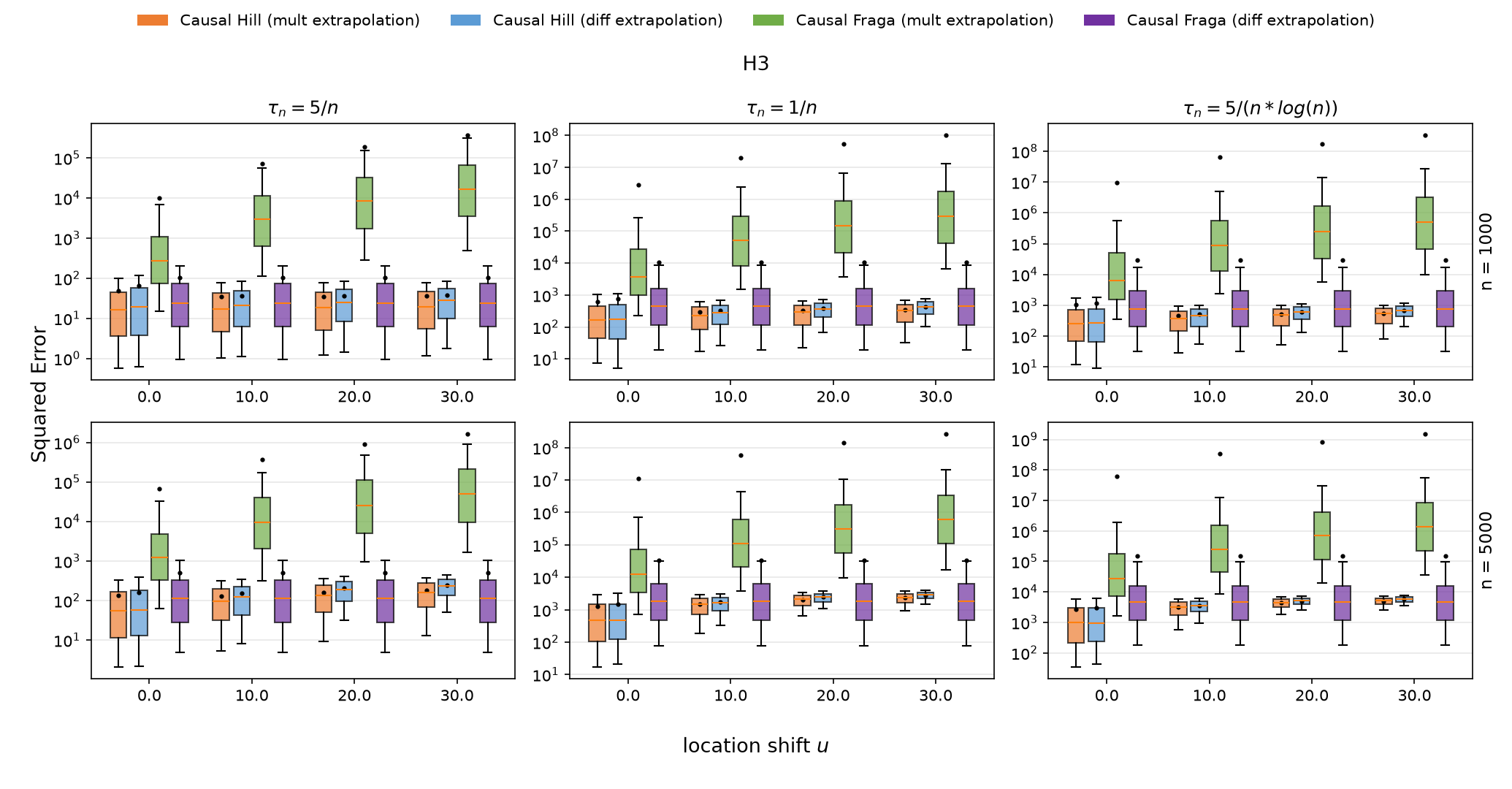}
  \caption{Squared errors of the quantile treatment effect
  estimations under the location shift $u$ for model $H_3$.  The boxplot
  whiskers correspond to the $0.1$ and $0.9$ quantiles, the orange line is
  the median, and the black dots are the means.}
  \label{fig:qte_shift_sqerr_H3}
\end{figure}

\subsection{Mean Squared Errors with respect to $k$}
\label{subsec:mse_k}

\begin{figure}[t]
  \centering
  \includegraphics[width=\linewidth]{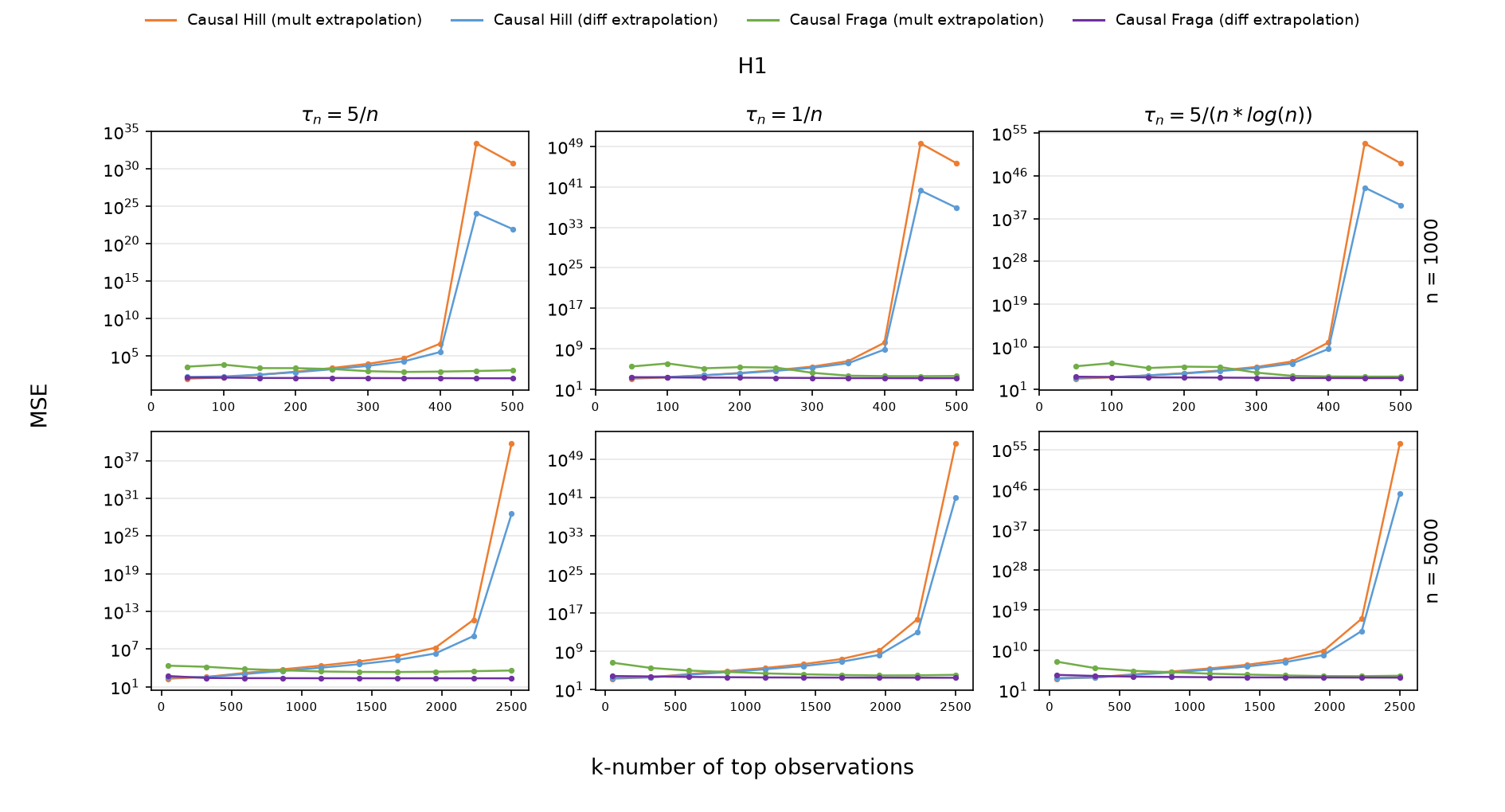}
  \caption{Mean squared errors of the quantile treatment effect
  estimations as a function of the threshold parameter $k$ for model
  $H_1$.}
  \label{fig:qte_k_mse_H1}
\end{figure}

\begin{figure}[t]
  \centering
  \includegraphics[width=\linewidth]{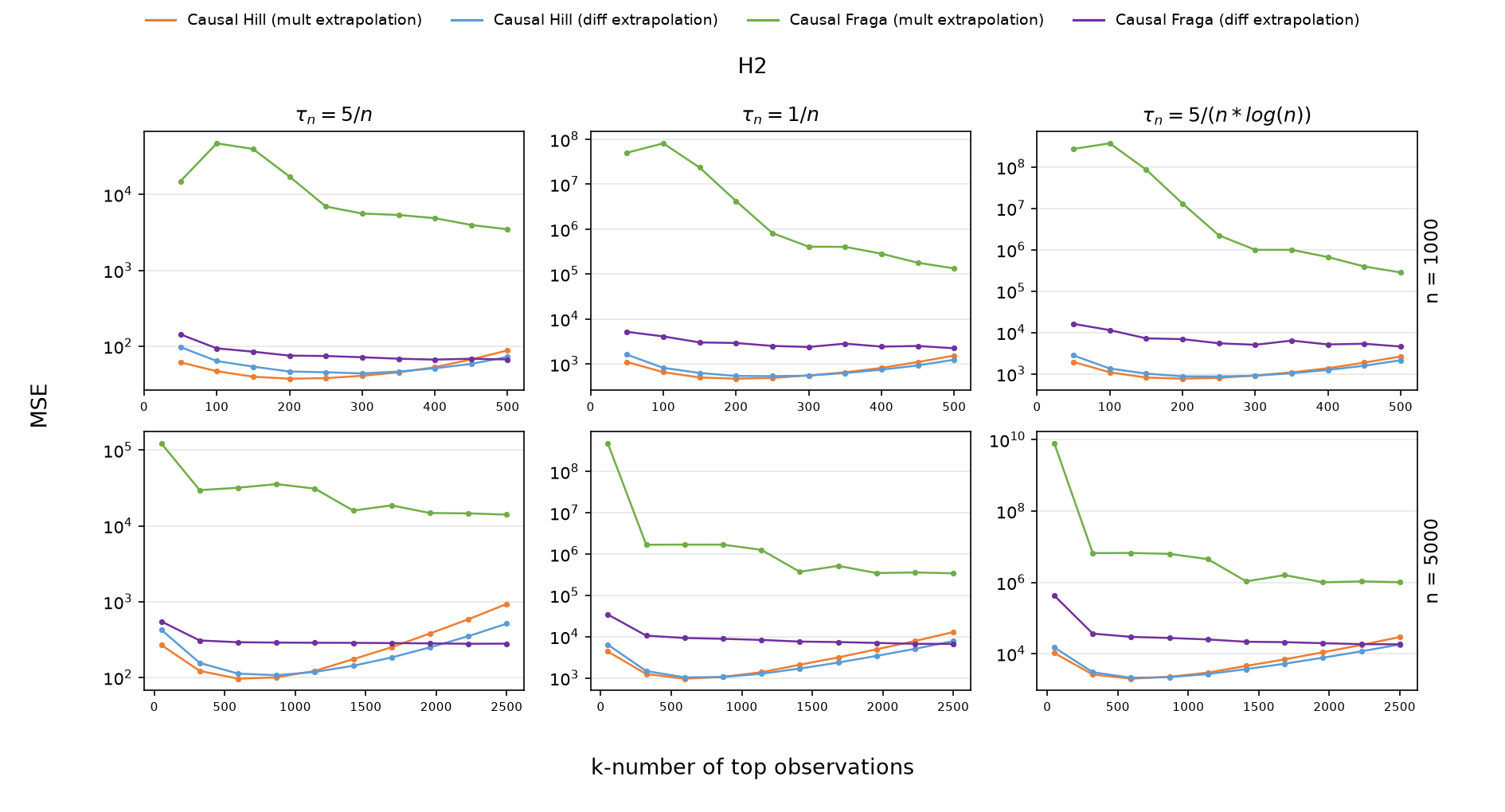}
  \caption{Mean squared errors of the quantile treatment effect
  estimations as a function of the threshold parameter $k$ for model
  $H_2$.}
  \label{fig:qte_k_mse_H2}
\end{figure}

\begin{figure}[t]
  \centering
  \includegraphics[width=\linewidth]{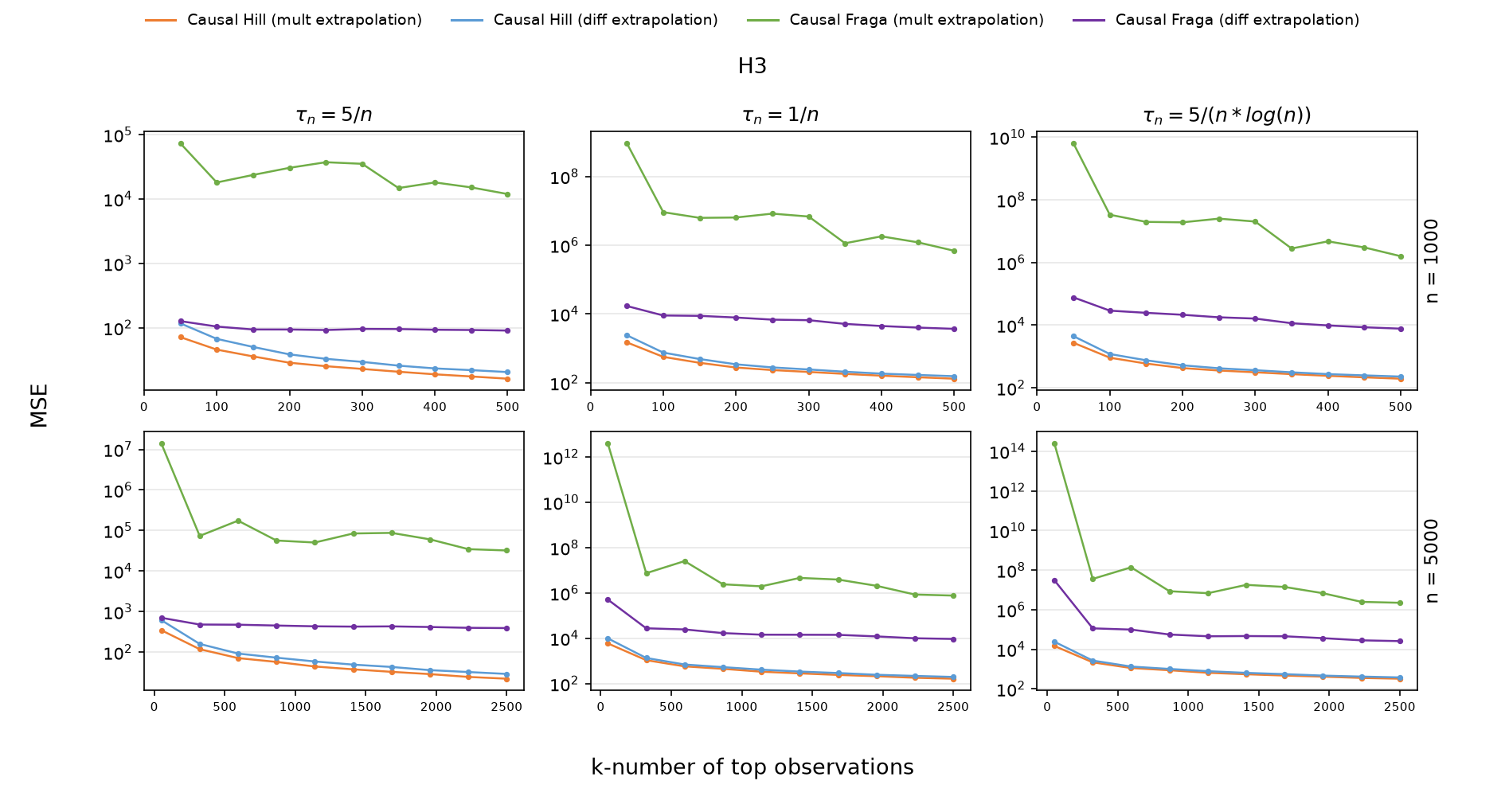}
  \caption{Mean squared errors of the quantile treatment effect
  estimations as a function of the threshold parameter $k$ for model
  $H_3$.}
  \label{fig:qte_k_mse_H3}
\end{figure}

Figures~\ref{fig:qte_k_mse_H1}--\ref{fig:qte_k_mse_H3} confirm the
patterns observed in
Figures~\ref{fig:qte_k_mean_H1}--\ref{fig:qte_k_mean_H3}.  The QTE
estimator that combines difference extrapolation with the causal Fraga
estimator attains a low and stable mean squared
error across the three models and the whole range of $k$.  The two extrapolation variants based
on the causal Hill estimator display large and erratic mean squared
errors under $H_1$ and $H_2$, in line with the instability of the Hill
extreme value index estimator in those settings (see the discussion in
Fraga Alves~\cite{fraga2001location}).  Under multiplicative
extrapolation, the high variance of the causal Fraga extreme value
index estimator, amplified by small values of $k$, inflates both the
variance and the bias component of the mean squared error, which
explains the large mean squared errors observed under $H_2$ and $H_3$.

\end{document}